%% file: arxiv.tex
\documentclass{article} %
\PassOptionsToPackage{numbers, compress}{natbib}
\usepackage{arxiv}
\usepackage{natbib}
\usepackage[dvipsnames]{xcolor}
\usepackage{graphicx}
\usepackage{caption}
\usepackage{subcaption}
\usepackage{booktabs}
\usepackage{amssymb}%
\usepackage{pifont}%

\input{math_commands.tex}

\usepackage{multirow}
\usepackage{nicematrix}
\usepackage{hyperref}
\usepackage{url}
\usepackage{soul}
\usepackage[textwidth=2.5cm]{todonotes}
\usepackage{amsmath, amsthm, amssymb}
\usepackage{wrapfig}
\usepackage{nicefrac}
\usepackage{enumitem}

\newcounter{daggerfootnote}
\newcommand*{\daggerfootnote}[1]{%
    \setcounter{daggerfootnote}{\value{footnote}}%
    \renewcommand*{\thefootnote}{\fnsymbol{footnote}}%
    \footnote[2]{#1}%
    \setcounter{footnote}{\value{daggerfootnote}}%
    \renewcommand*{\thefootnote}{\arabic{footnote}}%
    }
    
\title{ReLU Strikes Back: \\
Exploiting Activation Sparsity in Large Language Models}

\author{Iman Mirzadeh\daggerfootnote{} \And Keivan Alizadeh \And  Sachin Mehta \And Carlo C Del Mundo \AND Oncel Tuzel \And Golnoosh Samei \And Mohammad Rastegari \And Mehrdad Farajtabar\daggerfootnote{}}

\newcommand\figref[1]{Fig.~\ref{#1}}
\newcommand\secref[1]{Sec.~\ref{#1}}
\newcommand\tabref[1]{Tab.~\ref{#1}}
\newcommand{\relu}{ReLU}
\newcommand{\silu}{SiLU}
\newcommand{\gelu}{GELU}
\newcommand\negspace[1]{}

\date{}

\begin{document}

\maketitle
\def\thefootnote{$\dagger$}\footnotetext{Corresponding authors: \texttt{\{imirzadeh,farajtabar\}@apple.com } }\def\thefootnote{\arabic{footnote}}

\begin{abstract}
Large Language Models (LLMs) with billions of parameters have drastically transformed AI applications. However, their demanding computation during inference has raised significant challenges for deployment on resource-constrained devices. Despite recent trends favoring alternative activation functions such as GELU or SiLU, known for increased computation, this study strongly advocates for reinstating ReLU activation in LLMs. We demonstrate that using the ReLU activation function has a negligible impact on convergence and performance while significantly reducing computation and weight transfer. This reduction is particularly valuable during the memory-bound inference step, where efficiency is paramount. Exploring sparsity patterns in ReLU-based LLMs, we unveil the reutilization of activated neurons for generating new tokens and leveraging these insights, we propose practical strategies to substantially reduce LLM inference computation up to three times, using ReLU activations with minimal performance trade-offs.
\end{abstract}
\vspace{2mm}

\section{Introduction}
The widespread excitement surrounding Large Language Models (LLMs) has sparked significant interest in leveraging AI across diverse domains~\cite{brown2020language,chowdhery2022palm,bubeck2023sparks}. However, realizing the potential of LLMs is challenged by their significant computational and memory requirements during inference~\cite{pope2023efficiently,kim2023full,aminabadi2022deepspeed}. To enhance the inference efficiency\footnote{In this work, we use FLOPS as a proxy for inference efficiency. In Appendix~\ref{sec:appendix-discussion-efficiency}, we demonstrate that for LLMs with activation sparsity, FLOPS can serve as a good approximation of real-world efficiency due to the structure inherent in activation sparsity (e.g., skipping the entire row corresponding to zero activations).}, various techniques have been explored, including quantization~\cite{dettmers2022llm,liu2023llmqat}, speculative decoding~\cite{kim2023speculative}, pruning~\cite{ma2023llm,sun2023simple}, and weight sparsification~\cite{frantar2023sparsegpt,dong2023blockwise}. Among these techniques, achieving activation sparsity offers a compelling advantage by providing a favorable balance between accuracy and speedup, especially on modern hardware like GPUs~\cite{liu2023deja}.

Notably, employing the Rectified Linear Unit (ReLU) activation function~\cite{Fukushima1969VisualFE} in neural networks is recognized for inducing sparse activations and has been adopted in various prior works~\cite{he2016deep, kurtz2020inducing,li2022lazy,sheng2023flexgen}. To reaffirm this property, we employ the OPT model~\cite{OPTpaper}, utilizing \relu, and measure the sparsity of activations in the Feed Forward Network (FFN) between the fully connected layers. As illustrated in \figref{fig:act_compare_sparsity}, all layers exhibit sparsity exceeding $90\%$. On average, across all layers, this activation sparsity results in substantial weight transfer (I/O) savings between the GPU and CPU, impacting $95\%$ of the rows of the down projection layer's weights (\figref{fig:mat-vec}). This reduction directly translates to computation savings, as for these rows, the result of the matrix multiplication operation will be zero. Furthermore, unlike unstructured sparsity (e.g., weight pruning), this type of sparsity is more hardware-friendly due to zeroing more extensive and structured chunks, such as rows or columns~\cite{jaszczur2021sparse,liu2023deja}. For OPT models, this sparsity reduces the computation required for inference from 6.6G FLOPS (Floating Point Operations Per Second) to 4.5G FLOPS per token, resulting in a $32\%$ computation saving (\figref{fig:act_compare_fine_tune}).

However, a recent trend has emerged, favoring variations of ReLU that are smoother but more complex~\cite{hendrycks2016gaussian,ramachandran2017searching}. These alternatives have gained popularity due to their slightly faster convergence and improved final accuracy~\cite{shazeer2020glu}. For example, PaLM~\cite{chowdhery2022palm} and Llama models~\cite{Llamav1paper} adopt SiLU\footnote{To be more precise, the mentioned models use SwiGLU activation function, but in this work, we focus on the gating module that uses SiLU (Swish) function.}~\cite{hendrycks2016gaussian, elfwing2018sigmoid, ramachandran2017searching}, while MPT~\cite{MosaicMPT} and Falcon models~\cite{FalconPaper} use \gelu~\cite{hendrycks2016gaussian}. Nonetheless, as demonstrated in \figref{fig:act_compare_fine_tune}, when we finetune several pretrained LLMs with different activation functions, their performance does not change significantly (within a specific model), while \relu~models require much less computation.

In this paper, we re-evaluate using ReLU for LLMs. We are motivated by the pragmatic consideration that, in many real-world applications and computational platforms capable of supporting sparse vector-matrix multiplications, computational efficiency during \emph{inference} outweighs the one-time computational cost incurred during training.  We make the following contributions:

\begin{itemize}
\item We demonstrate that when trained from scratch, there is no significant difference in terms of performance between different activation functions. However, in terms of computational requirements during inference, ReLU activations prove significantly lighter~(\secref{sec:from-scratch}).

\item Considering that many modern LLMs (e.g., Llama and Falcon) have been trained with non-ReLU activations, and it is not cost-effective to train them from scratch, we investigate fine-tuning these models with ReLU activations. We show that the models quickly regain their original performance across various reasoning and reading comprehension tasks~(\secref{sec:relu_stage_1}). Moreover, we show that by leveraging the activation sparsity of ReLU layers and inserting additional ReLU layers after normalization layers, we can further reduce inference FLOPS by up to threefold~(\secref{sec:relu_stage_2}).

\item In addition to their computational benefits, we present two promising applications of activation sparsity that can inspire future work. Firstly, we demonstrate that LLMs with ReLU activations reuse a significant portion of already activated neurons during token generation, a phenomenon we term \emph{aggregated sparsity} (\secref{sec:aggregated-sparsity}). This reusability leads to an inference speedup for speculative decoding (\secref{sec:speculative-decoding}). Additionally, we show that studying the pre-activations of pretrained LLMs can guide the selection of unconventional activation functions (e.g., \emph{shifted ReLU}), achieving up to 90\% sparsity while maintaining performance similar to ReLU activation ~(\secref{sec:shifted-relu}).

\end{itemize}
Overall, we believe our work represents a significant step toward leveraging the potential of sparse activation functions for faster and more efficient inference in large language models.

\begin{figure}[t]
\centering
\begin{subfigure}{.36\textwidth}
  \centering
  \includegraphics[width=\textwidth]{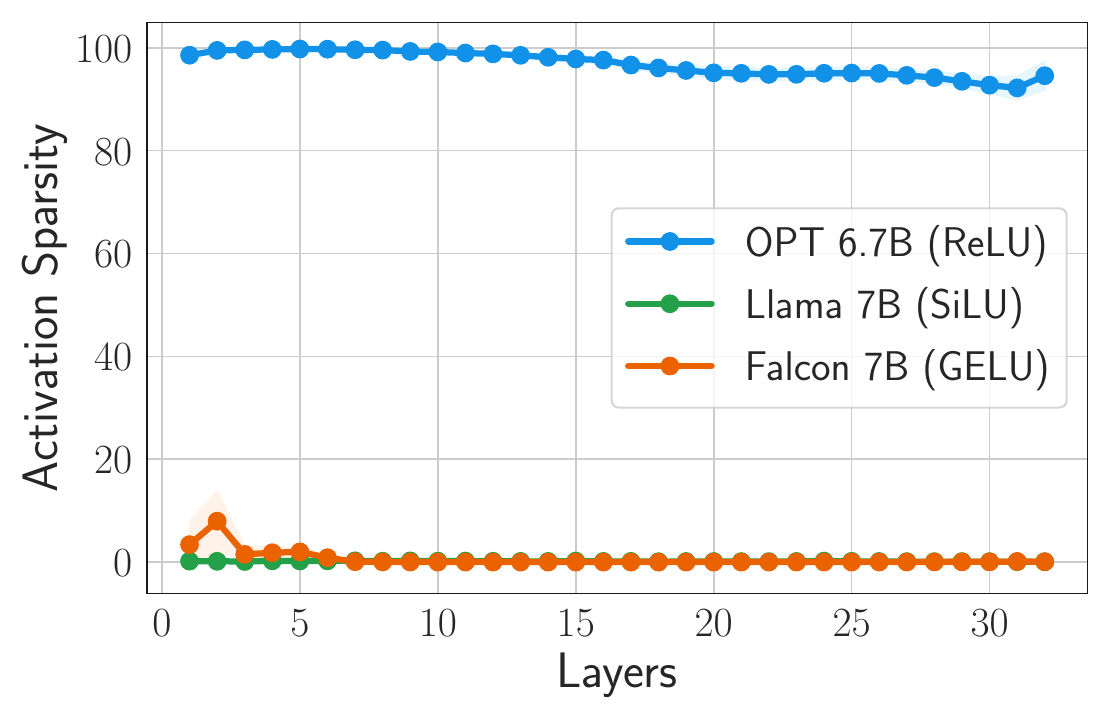}
\caption{Sparsity of different models}
  \label{fig:act_compare_sparsity}
\end{subfigure}\hfill
\begin{subfigure}{.26\textwidth}
  \centering
  \includegraphics[width=\textwidth]{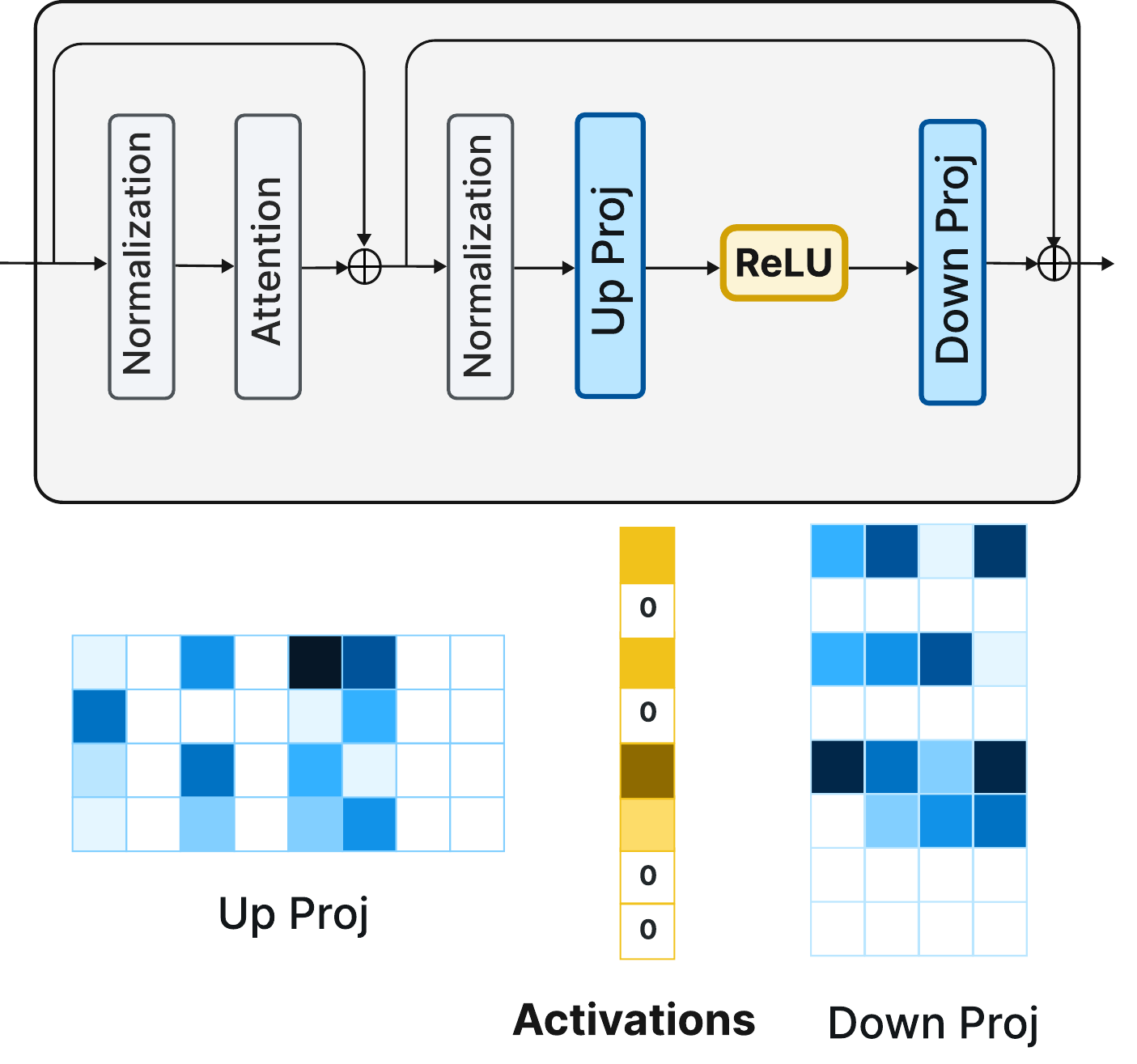}
    \caption{Sparsity for Efficiency}
  \label{fig:mat-vec}
\end{subfigure}\hfill
\begin{subfigure}{.31\textwidth}
  \centering
  \includegraphics[width=\textwidth]{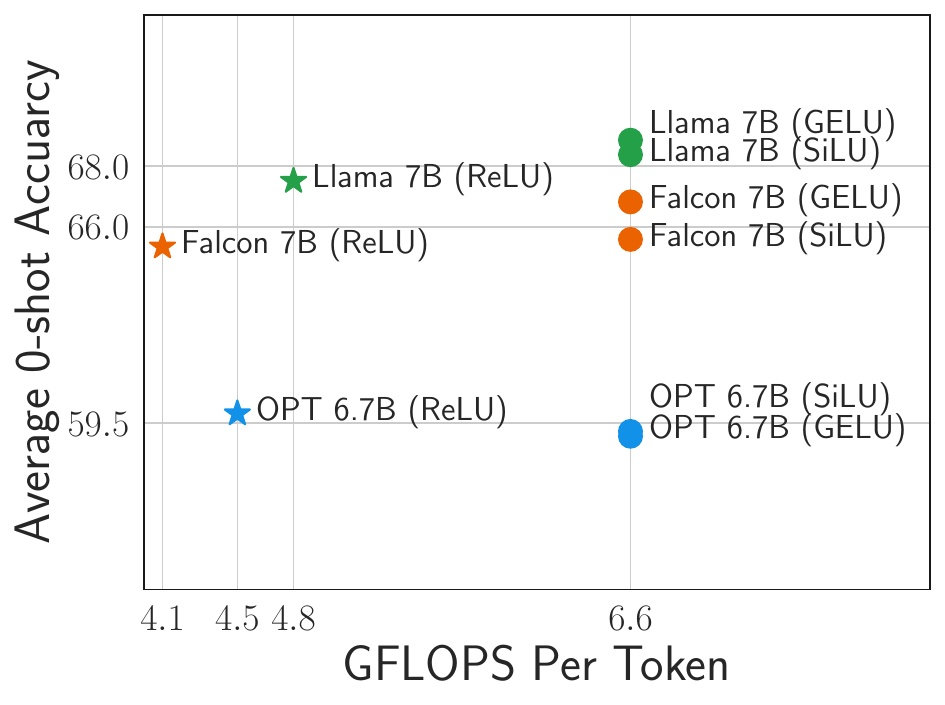}
\caption{Accuracy vs. Computation}
  \label{fig:act_compare_fine_tune}
\end{subfigure}
\caption{
\textbf{(a)} Activation Sparsity of different pretrained models: ReLU-based OPTs show significantly higher sparsity.
\textbf{(b)} Zeroed out entries after ReLU save compute in large semi-structured chunks (e.g., rows). 
\textbf{(c)} Comparison of inference efficiency and performance of the different models with different activation functions after fine-tuning: The choice of activation function does not significantly impact the accuracy, as any of \gelu, \silu, or ReLU can be used on all three models and achieve the same level of accuracy as the original activation function. However, using ReLU can provide an additional benefit of leading to activation sparsity and faster inference.} 
\label{fig:intro}
\end{figure}

\vspace{5mm}
\section{Related Works}

\textbf{Activation Functions in Transformers.} The original Transformer architecture~\cite{vaswani2023attention} was proposed with the ReLU activation function~\cite{Fukushima1969VisualFE}, following the popularity of ReLU at the time. Later, several studies aimed to improve the ReLU activation function by increasing its smoothness~\cite{hendrycks2016gaussian} and/or including parameterized gating mechanisms, such as GELU, SiLU, GLU, and SwiGLU~\cite{10.5555/3305381.3305478,ramachandran2017searching}. Earlier studies demonstrated the benefits of these alternatives to ReLU for transformers~\cite{shazeer2020glu,narang2021transformer}, but on a small scale (e.g., they trained models up to a couple of 100M parameters with at most 35B tokens, while in this work, we train 1B parameter models on more than 100B tokens). However, we believe the impact of activation functions on performance is marginal, following scaling laws~\cite{ScalingLawOpenAI,ScalingLawChinchilla}, which state that architectural changes do not significantly impact performance.

\textbf{Activation Sparsity.} Existing research shows increased sparsity reduces inference and training times~\cite{kurtz2020inducing, han2023retrospective, song2021training,zhang2022moefication,li2022large,liu2023deja}. For instance, \citet{jaszczur2021sparse} uses ReLU and added a controller to both promote and predict sparsity, while other works only use prediction modules to predict the activation masks~\cite{liu2023deja}. We note that the mentioned works assume the pretrained model has already been using a sparse ReLU activation, and hence, only training a separate module to predict sparsity could be enough. However, we note that most LLMs pretrained these days do not use ReLU, and we aim to bridge this gap. Moreover, these works focus only on a single transformer architecture while we focus on various architectures so our findings can be practical. Finally, we show that there is no need to train a separate prediction module that complicates the computation graph, and using efficient ReLU layers can be enough.

\textbf{Speculative Decoding and Sparsity.} Speculative decoding combats latency under memory constraints using a smaller model for token prediction and a larger model for verification \cite{leviathan2023fast, kim2023speculative}. Investigating its integration with sparsity, we find activation sparsity exhibits a temporal pattern, enhancing speculative decoding. We provide guidelines for parameter selection when incorporating sparsity.

 We defer other lines of related works that are orthogonal to our work, such as model compression techniques, sparse attention methods, and Mixture of Experts (MoE) to Appendix~\ref{sec:appendix-relatedworks}.

\section{Does the Activation Function Impact Performance?}
\label{sec:from-scratch}

This section first overviews our experimental setup, including models, data, and evaluations. Then, by training various models from scratch with different activation functions, we demonstrate that changing activation functions minimally impacts performance. However, the impact on inference efficiency is substantial.

\subsection{Experimental setup}

\textbf{Models.} We use open source pretrained models such as OPT~\cite{OPTpaper}, Llama (v1)~\cite{Llamav1paper}, and Falcon~\cite{FalconPaper} as they use different architectures and pretraining setup (e.g., attention/FFN structure/normalization, activation functions), allowing our study covers a wider range of models.

\textbf{Datasets.} We use the RefinedWeb dataset~\cite{RefinedWebDataset}, for our pretraining in~\secref{sec:from-scratch-comparison} and finetuning pretrained models in ~\secref{sec:relufication}. We chose RefinedWeb because it is a high-quality subset of Common Crawl, which is often used in the pretraining phase of LLMs, including Llama, Falcon, and OPT.  We also use the validation split of WikiText~\cite{Wikitext} for measuring the sparsity and recording preactivation distributions of various pretrained models. However, our conclusions hold for other datasets we have tested. 

\textbf{Training and Finetuning.} For finetuning the pretrained models, we follow the original pretraining recipe, except we use a fixed learning rate of 1.5e-5 for Llama 7B, Falcon 7B, and OPT 6.7B models. In addition, we use the AdamW optimizer~\cite{loshchilov2017decoupled} for our finetuning with ZeRO stage 1~\cite{ZeRO}, where we shard the optimizer states across different GPUs. For pretraining OPT 1.3B models from scratch in \secref{sec:from-scratch-comparison}, we follow the OPT training recipe.

\textbf{Evaluation.} For our \emph{performance} evaluation, we use the few-shot tasks from Language Model Evaluation Harness~\cite{eval-harness}. We select these tasks such that they can measure various abilities of the models (e.g., reading comprehension, reasoning, etc.), and we aim to be consistent with other works in the literature to make the comparison easier. Consistent with the other sections, we compare activation sparsity as a measure of \emph{efficiency}. 

Further details regarding the relationship between activation sparsity, FLOPS, and inference efficiency are discussed in Appendix~\ref{sec:appendix-discussion-efficiency}.

\subsection{Training from scratch: performance and sparsity}
\label{sec:from-scratch-comparison}
While the previous literature suggests that non-ReLU variants can improve the performance of transformers~\cite{shazeer2020glu,narang2021transformer}, we argue the impact is marginal at best. To support our claim, we train the OPT 1.3B model from scratch on a hundred billion tokens of the RefinedWeb datasets with different activation functions, including ReLU, \silu, and GELU. All these activation functions can be viewed as $\text{f(x)} = x \cdot \sigmoid(\beta x)$, where $\beta$ controls the gating part (smoothed cutoff threshold) of the activation function (see \figref{fig:scratch-comp-actfn-zoomed-out}). For $\beta=1$, we will have \silu ($x \cdot \sigmoid(x)$), and $\beta=1.7$ is a good approximation of \gelu. Finally, as $\beta \to \infty$, the activation function becomes closer to ReLU. To further explore the spectrum of ReLU to \silu we add another one with $\beta=8$.

\begin{figure}[t]
\negspace{-5mm}
\centering
\begin{subfigure}{.325\textwidth}
  \centering
  \includegraphics[width=\textwidth]{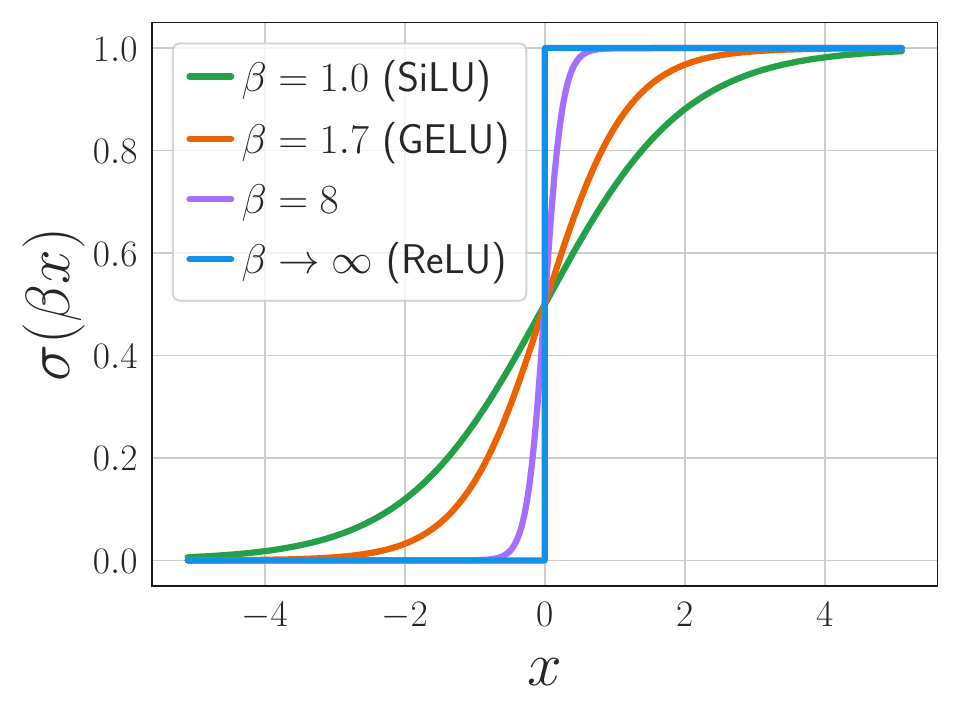}
  \caption{}
  \label{fig:scratch-comp-actfn-zoomed-out}
\end{subfigure}
\begin{subfigure}{.325\textwidth}
  \centering
  \includegraphics[width=\textwidth]{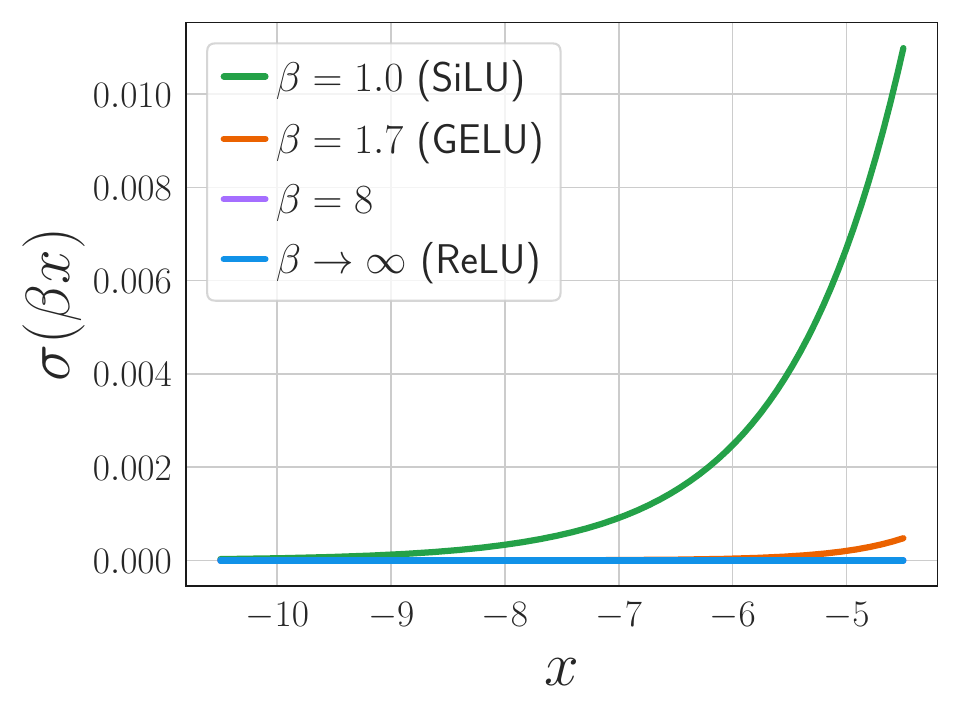}
    \caption{}
  \label{fig:scratch-comp-actfn-zoomed-in}
\end{subfigure}
\begin{subfigure}{.325\textwidth}
  \centering
  \includegraphics[width=\textwidth]{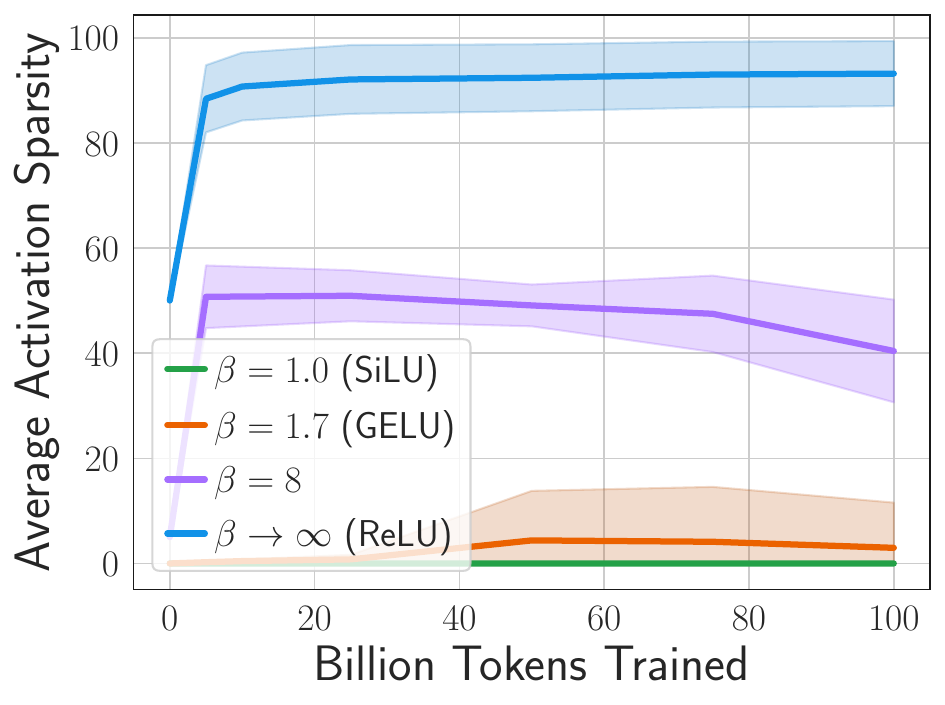}
    \caption{}
  \label{fig:scratch-comp-sparsity}
\end{subfigure}
\begin{subfigure}{.325\textwidth}
  \centering
  \includegraphics[width=\textwidth]{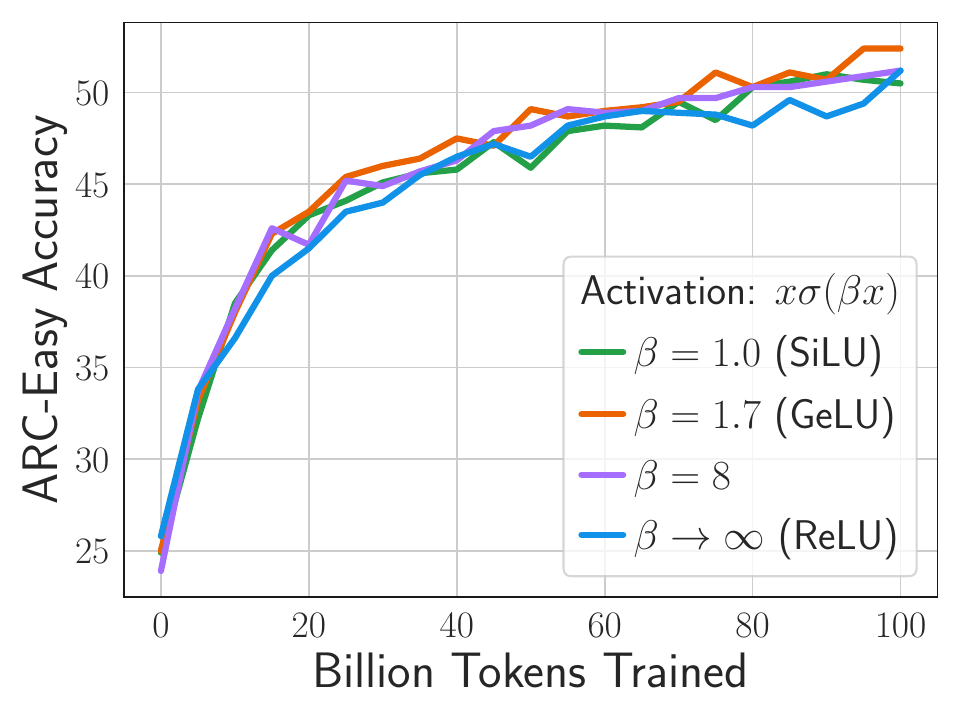}
\caption{}
  \label{fig:scratch-comp-arc}
\end{subfigure}
\begin{subfigure}{.325\textwidth}
  \centering
  \includegraphics[width=\textwidth]{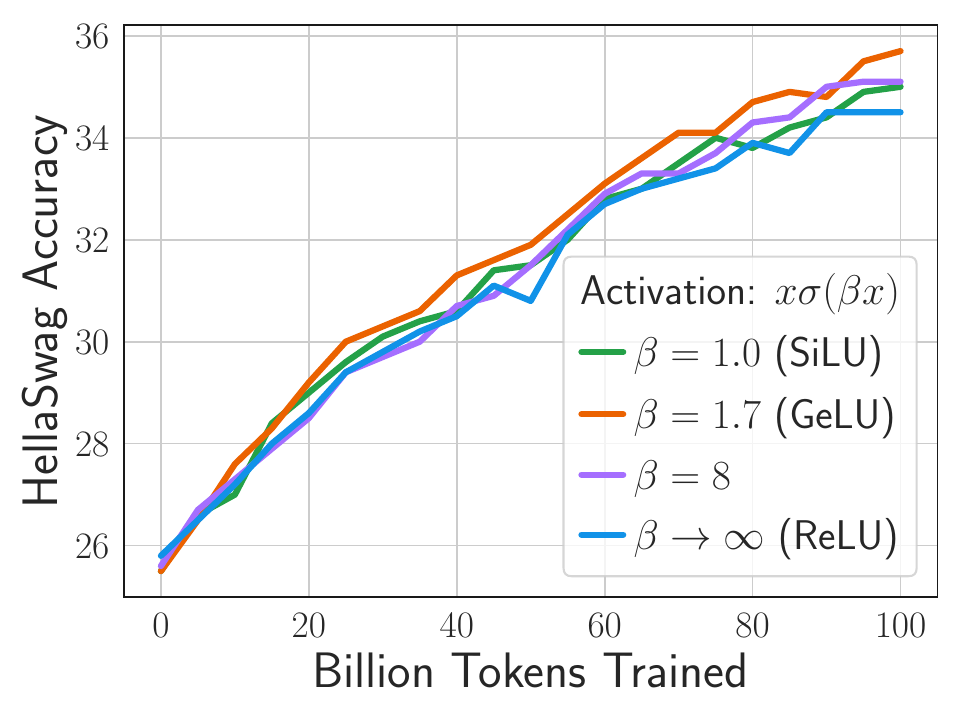}
\caption{}
  \label{fig:scratch-comp-hellaswag}
\end{subfigure}
\begin{subfigure}{.325\textwidth}
  \centering
  \includegraphics[width=\textwidth]{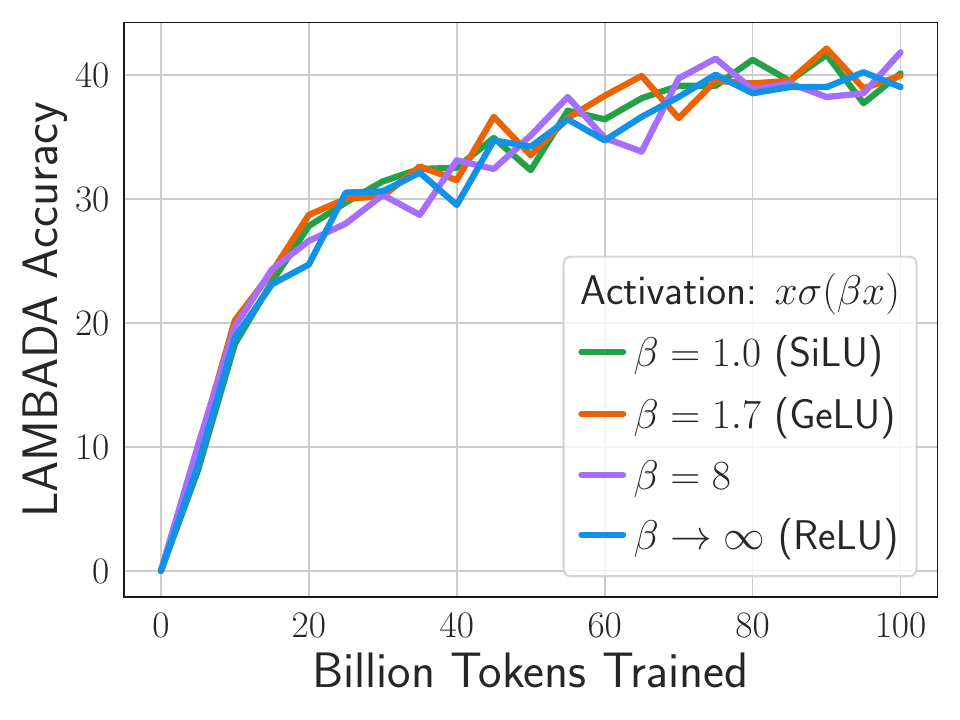}
\caption{}
  \label{fig:scratch-comp-lambada}
\end{subfigure}
\caption{ \textbf{(top)} (a) Shapes of different gating functions over [-5, 5]; (b) Continuation of (a) where SiLU is comparably larger compared to others; (c) Sparsity of the FFN with different activations: increasing $\beta$ increases sparsity. \textbf{(bottom)}
when trained from scratch, OPT 1.3 B models using different activation functions achieve similar performance.} 
\label{fig:scratch-comp-acts}
\end{figure}

\looseness=-1 As shown in the bottom row of \figref{fig:scratch-comp-acts}, the performance of the models is very similar when using different activation functions. This is consistent with the scaling laws literature (\cite{ScalingLawOpenAI,ScalingLawChinchilla}), which suggests that the performance of sufficiently large models trained on sufficiently large data depends heavily on compute and data, not architectural details.

While the performance levels of the different activations are similar, their activation sparsity levels differ. Here, we define sparsity as the average sparsity level across all layers for each model. As shown in \figref{fig:scratch-comp-sparsity}, as we transition from \silu to ReLU (increasing $\beta$), the sparsity also increases. This results from the different gating thresholds, as ReLU drops significantly more values compared to GELU and SiLU (see \figref{fig:scratch-comp-actfn-zoomed-in}). In Appendix \ref{sec:appendix-preact-dist-opt}, we illustrate the evolution of the pre-activation distribution throughout training.

\looseness=-1 Overall, the results support our initial claim: non-ReLU activations result in a negligible performance gain (if any) but a substantial loss in sparsity and efficiency. While, at times, the performance of GeLU or SiLU might be slightly higher, ReLU can match it with slightly longer training. We acknowledge that to compensate for the small gap in performance, we need to pay the one-time cost of longer training. However, in return, we get a significantly more sparsity.

\begin{figure}[t]
\centering
\includegraphics[width=0.88\textwidth]{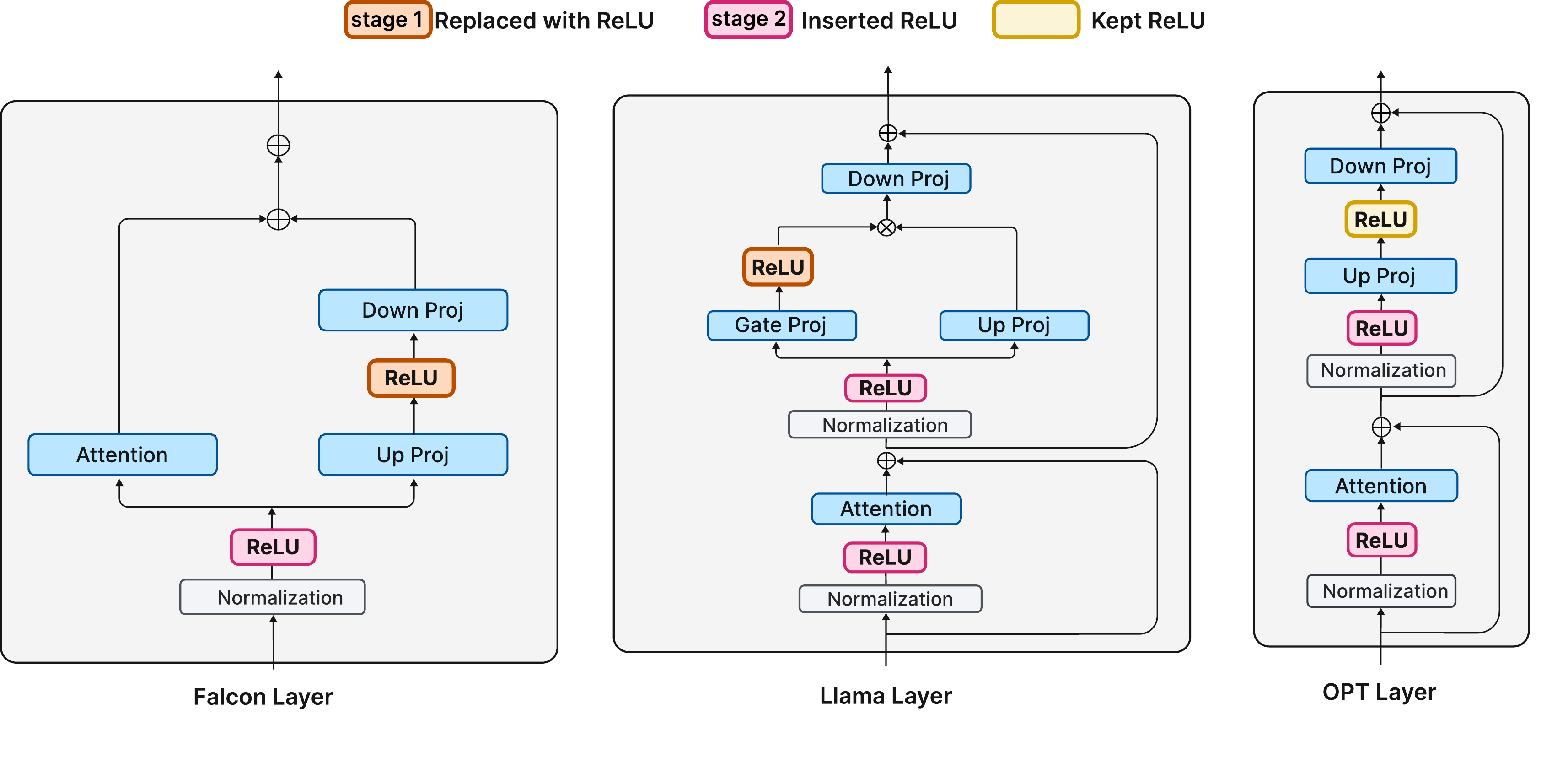}
\caption{Architectural surgeries for \emph{relufication}. In stage 1 we keep the existing ReLUs (in the case of OPT) or replace the activation function between up projection and down projections from GELU (Falcon) and SiLU (Llama) to ReLU. In stage 2, we insert new ReLUs after normalization layers. }
\label{fig:relufication-architecture-change}
\end{figure}

\section{Relufication}
\label{sec:relufication}
While in the previous section, we have seen that the performance does not depend on the activation function, we note that most of the available pretrained LLMs are trained with activation functions other than ReLU. Hence, to incorporate the computational benefits of ReLU activations at inference time, we perform various architectural surgeries and study the consequences of such changes.

We present our findings about incorporating ReLU activations into the pretrained LLMs, a process we refer to as \emph{relufication}. More specifically, we show that replacing the activation functions of pretrained LLMs with ReLU is possible, and the performance can be recovered very rapidly during finetuning. 
Moreover, we show that we can exploit the sparse ReLU activations, and by inserting additional ReLU layers after normalization layers, we can improve inference efficiency, as FLOPS indicates.
Finally, we show these modifications, which are easy to implement, lead to lighter models at inference time while maintaining comparable performance to the original pretrained models.

\subsection{Stage 1: replacing non-ReLU activations}
\label{sec:relu_stage_1}

The process of relufication for different pretrained architectures is shown in \figref{fig:relufication-architecture-change}. This process can be done in multiple stages, as we describe here.
The first and more intuitive stage replaces non-ReLU activations with ReLU in the FFN layer. For the Falcon and Llama models, this means replacing GELU and \silu, respectively. We note that since OPT models already use ReLU activations, we keep those unchanged.  After finetuning on 30 billion tokens of the RefinedWeb, \figref{fig:relufication-1x} shows that the modified models have significantly more sparsity in their activations.

\begin{figure}[t]
\centering
\begin{subfigure}{.245\textwidth}
  \centering
  \includegraphics[width=\textwidth]{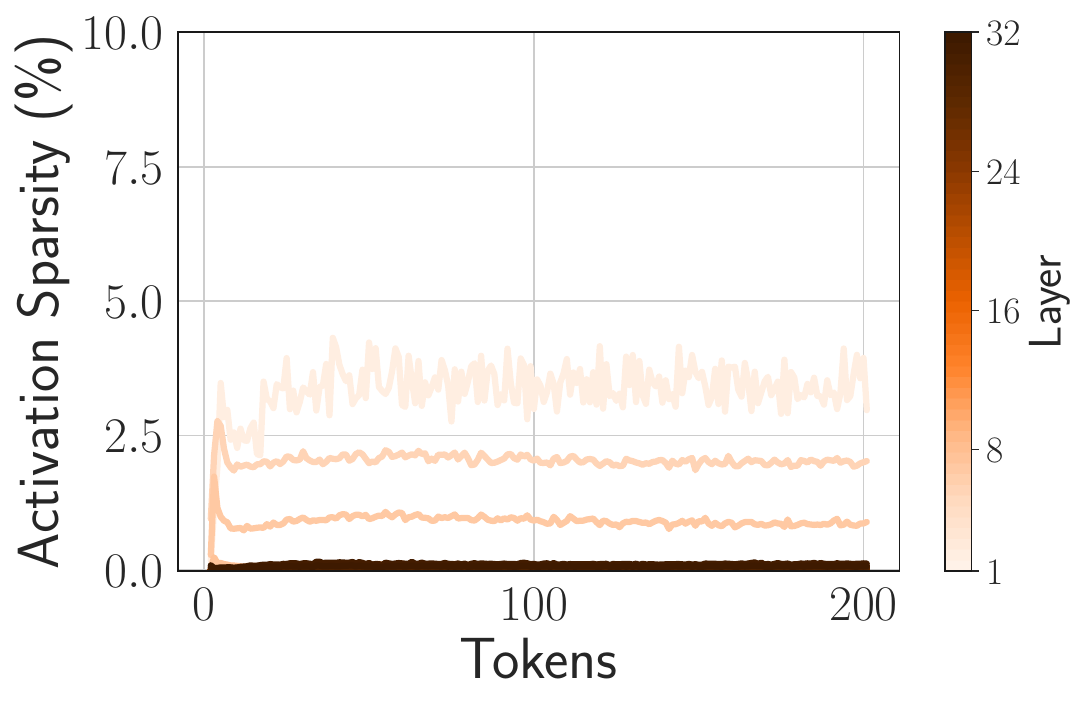}
\caption{Falcon 7B (GELU)}
\label{fig:isolated-falcon}
\end{subfigure}
\begin{subfigure}{.245\textwidth}
  \centering
  \includegraphics[width=\textwidth]{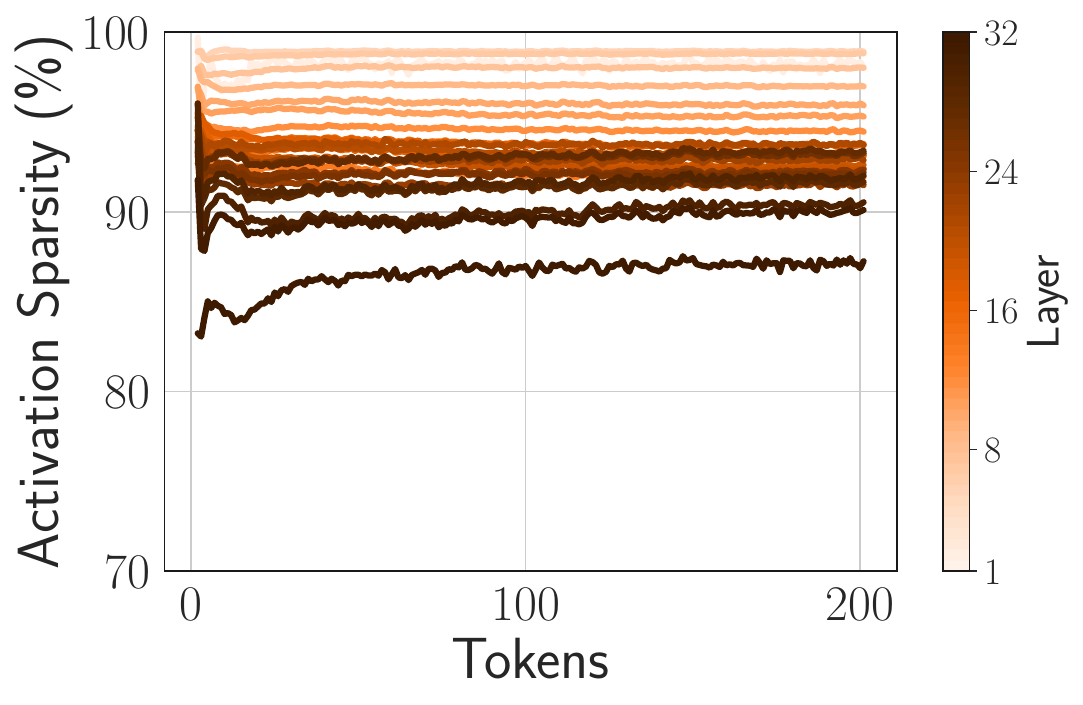}
\caption{Falcon 7B (ReLU)}
\label{fig:isolated-falcon-1x}
\end{subfigure}\hfill
\begin{subfigure}{.245\textwidth}
  \centering
  \includegraphics[width=\textwidth]{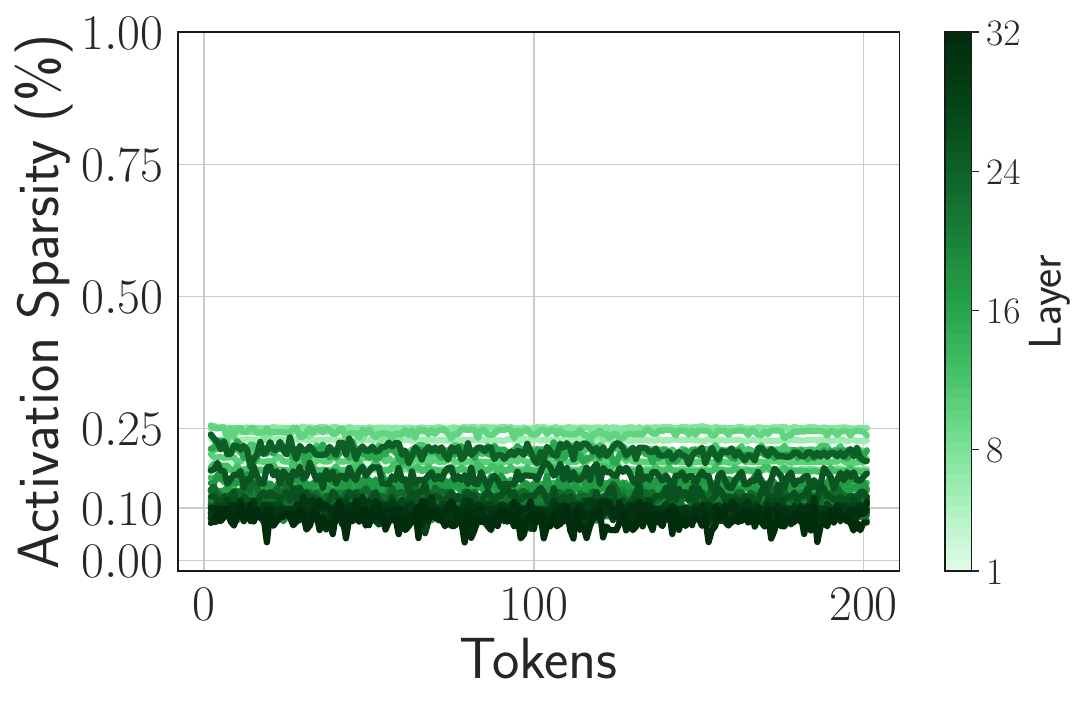}
  \caption{Llama 7B (SiLU)}
  \label{fig:isolated-llama}
\end{subfigure}
\begin{subfigure}{.245\textwidth}
  \centering
  \includegraphics[width=\textwidth]{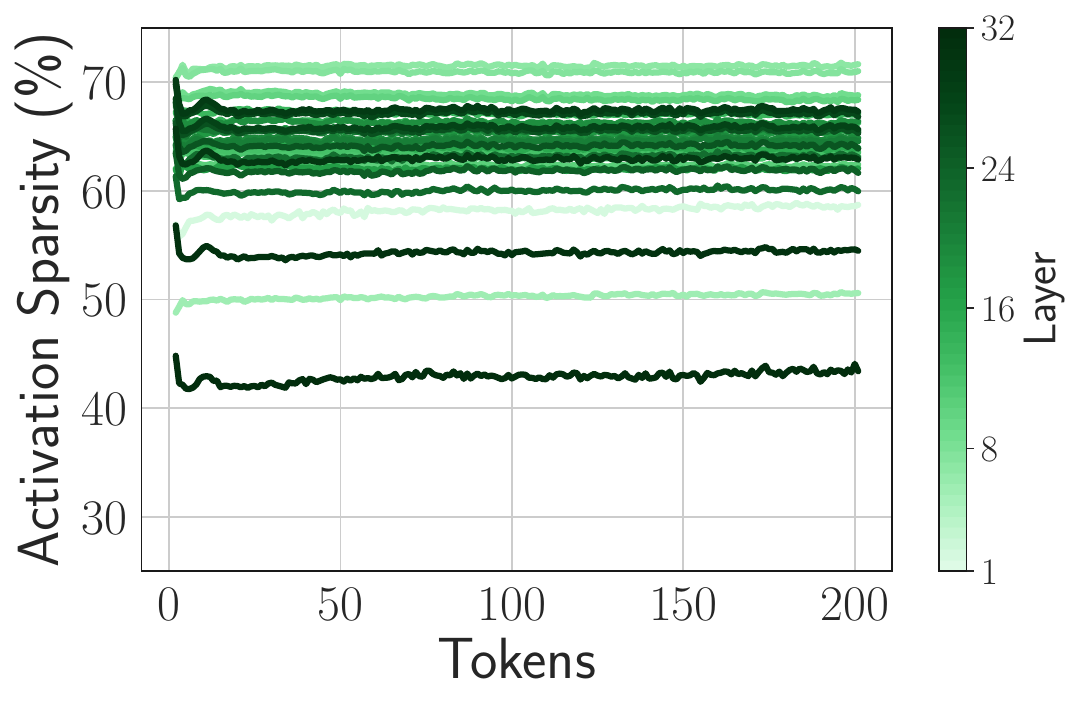}
  \caption{Llama 7B (ReLU)}
  \label{fig:isolated-llama-1x}
\end{subfigure}
\caption{ Activation sparsity of Falcon and Llama models improves significantly after \emph{relufication}.}
\label{fig:relufication-1x}
\vspace{-2mm}
\end{figure}

\begin{figure}[t]
\centering
\begin{subfigure}{.245\textwidth}
  \centering
  \includegraphics[width=\textwidth]{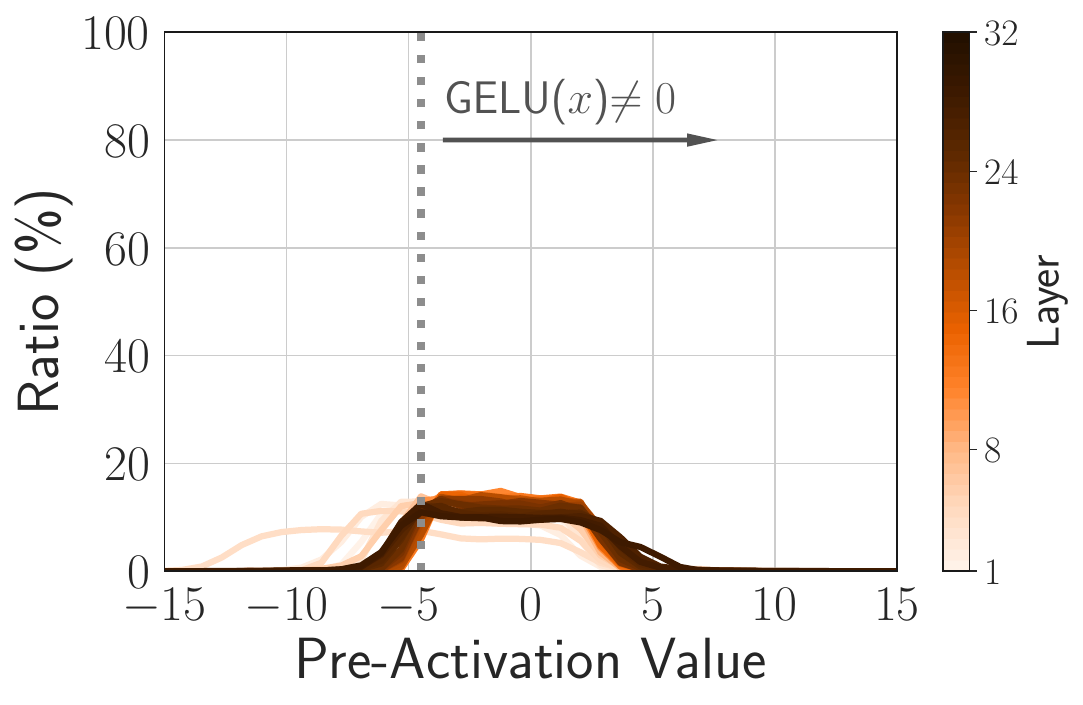}
  \caption{Falcon 7B (GELU)}
  \label{fig:preact-falcon}
\end{subfigure}\hfill
\begin{subfigure}{.245\textwidth}
  \centering
  \includegraphics[width=\textwidth]{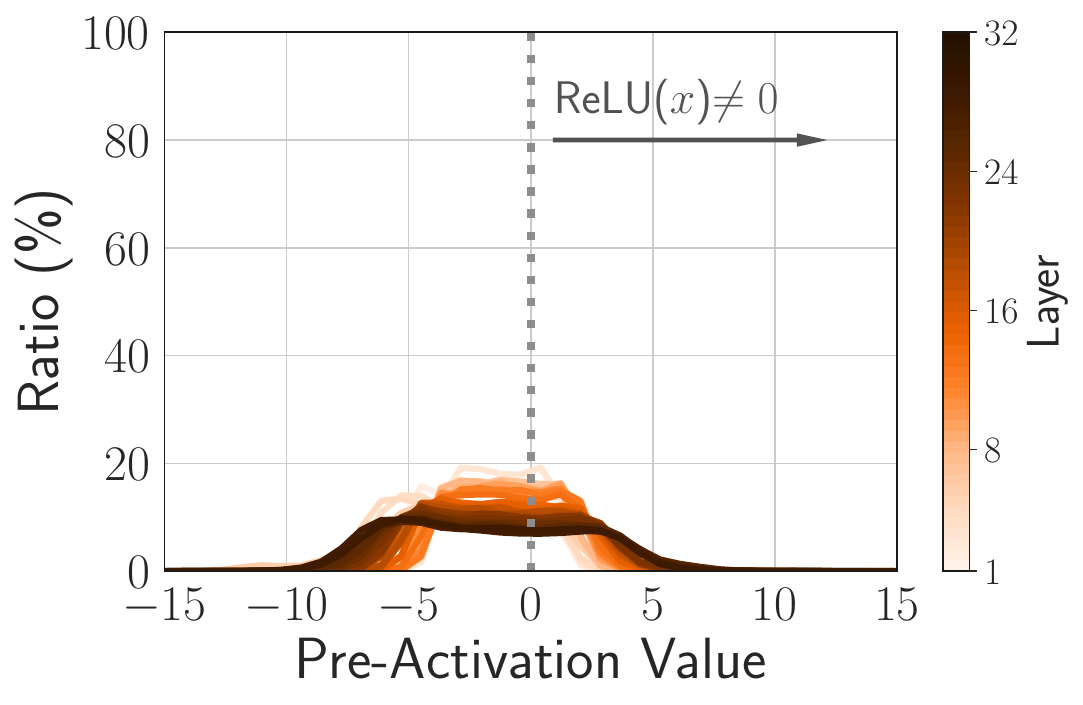}
  \caption{Falcon 7B (ReLU)}
  \label{fig:preact-falcon-1x}
\end{subfigure}\hfill
\begin{subfigure}{.245\textwidth}
  \centering
  \includegraphics[width=\textwidth]{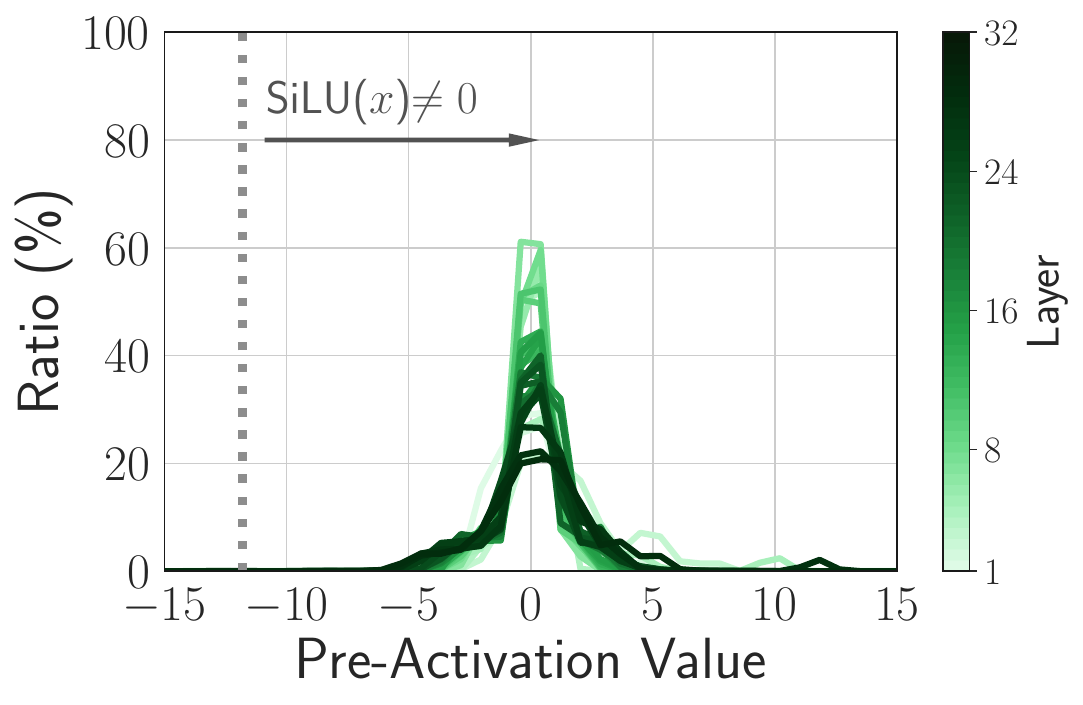}
  \caption{Llama 7B (SiLU)}
  \label{fig:preact-llama}
\end{subfigure}
\begin{subfigure}{.245\textwidth}
  \centering
  \includegraphics[width=\textwidth]{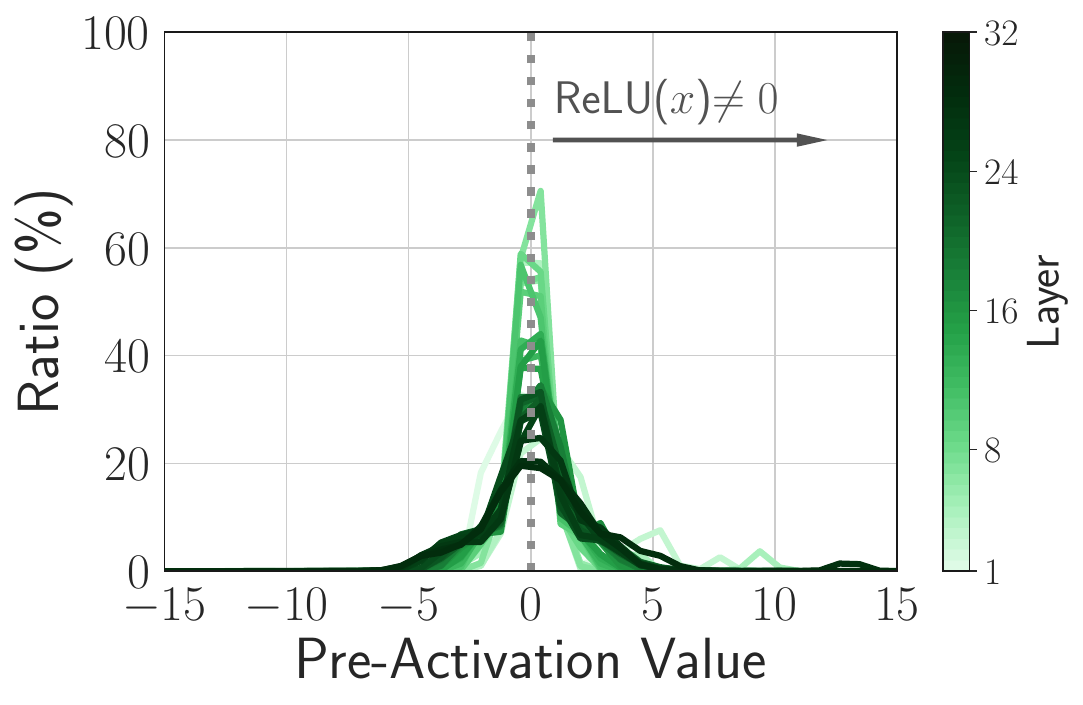}
  \caption{Llama 7B (ReLU)}
  \label{fig:preact-llama-1x}
\end{subfigure}
\caption{The preactivation distribution of pretrained models for Falcon and Llama does not change significantly during the short finetuning stage of relufication. The dashed line shows the cutoff point before which the output is almost zero.}
\label{fig:relufication_preact_distribuiton}
\negspace{-6mm}
\end{figure}

 \begin{wrapfigure}{R}{0.4\textwidth}
  \begin{center}
    \includegraphics[width=0.4\textwidth]{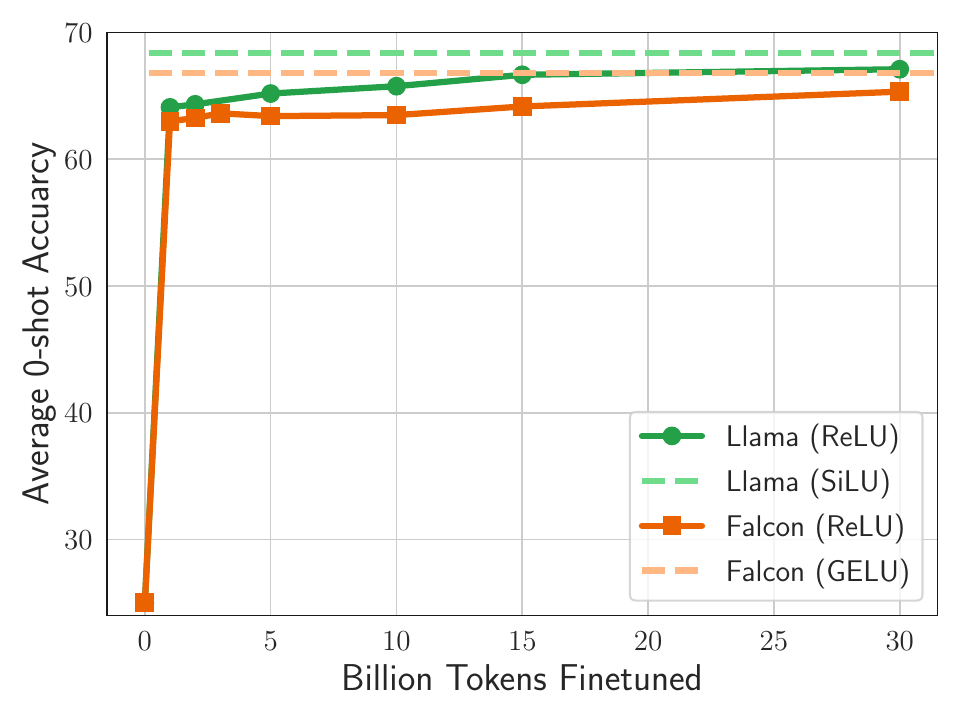}
  \end{center}
  \caption{Evolution of zero-shot accuracy during finetuning: The model quickly recovers most of its lost performance due to the architecture surgery.}
  \label{fig:relufication-evolution-acc}
\end{wrapfigure}
 In addition to the drastic improvement in activation sparsity, we can make several notable observations. First, while the shape of preactivation depends on the pretraining dynamics and architecture, in \figref{fig:relufication_preact_distribuiton}, we show that it does not change significantly during the relatively short finetuning stage. As a result, we can predict the activation sparsity before finetuning, knowing it will not change significantly. Later in \secref{sec:shifted-relu} we build on this observation and propose shifting the preactivation values before applying ReLU and further increasing the activation sparsity. The stability of the preactivation distribution may suggest that the behavior of the network does not change while creating sparse representations. Indeed, we show that after replacing the activation function with ReLU, finetuned models quickly recover their performance in~\figref{fig:relufication-evolution-acc}. We believe optimizing this process even further (e.g., using better finetuning data) is an exciting follow-up direction.

\subsection{Stage 2: Pushing for more sparsity}
\label{sec:relu_stage_2}
In the previous stage, we replaced non-ReLU activations to gain more sparsity. This leads to the input of \emph{down projection} layer being sparse, roughly $30\%$ of the total computation. However, there are other matrix-vector multiplications in the decoder layer of transformers besides the down projection. For instance, before the \emph{up projection} and \emph{gate projections} of FFN layer, and \emph{QKV projections} in the attention layer (see \figref{fig:relufication-architecture-change}). Together, the mentioned matrix-vector multiplications consume about $55\%$ of the total computation.

To this end, we utilize the fact that in modern transformer layers, the input to both the attention and FFN layers come from a normalization layer, e.g., LayerNorm~\cite{ba2016layer} or RMSNorm~\cite{zhang2019root}. These layers can be viewed as a specific form of MLP, where, instead of applying arbitrary learnable parameters, they learn to scale inputs. Therefore, we apply ReLU to obtain sparse activations after normalization layers which we call the \emph{second stage} of relufication in \figref{fig:relufication-architecture-change}.

\begin{table}[t]
\centering
\caption{Comparing zero-shot performance across several tasks: After \emph{relufication}, the activation sparsity of models increases significantly, hence increased efficiency measured by FLOPS. Within each group, the performance levels are comparable.}
\label{tab:relufy-results-table}
\resizebox{\textwidth}{!}{%
\begin{tabular}{@{}l|ccc|c|cccccccccc@{}}
\toprule
\multirow{2}{*}{\textbf{Model (stage)}} & \multicolumn{3}{c|}{\textbf{Input Sparsity (\%)}}  & \multirow{2}{*}{\textbf{\begin{tabular}[c]{@{}c@{}}FLOPS\\ (G)\end{tabular}}} & \multicolumn{10}{c}{\textbf{Zero-Shot Accuracy (\%)}}                                                                                                                                                  \\ \cmidrule(lr){2-4} \cmidrule(l){6-15} 
                                & \textbf{QKV} & \textbf{DownProj} & \textbf{UpProj} &                                                                               & \multicolumn{1}{c|}{\textbf{Avg}} & \textbf{Arc-E} & \textbf{Arc-C} & \textbf{Hellaswag} & \textbf{BoolQ} & \textbf{PIQA} & \textbf{LAMBADA} & \textbf{TriviaQA} & \textbf{WinoGrande} & \textbf{SciQ} \\ \midrule
OPT 1.3B                        & 0            & 96                & 0               & 1.3                                                                           & \multicolumn{1}{c|}{50.7}         & 57.3           & 22.9           & 41.3               & 57.0           & 71.8          & 56.0             & 6.1               & 58.9                & 84.6          \\
OPT 2.7B (s2)                   & 50           & 96                & 35              & 1.1                                                                           & \multicolumn{1}{c|}{53.1}         & 60.3           & 26.8           & 44.9               & 55.4           & 73.9          & 57.6             & 12.4              & 59.6                & 86.7          \\
OPT 2.7B                        & 0            & 96                & 0               & 1.8                                                                           & \multicolumn{1}{c|}{54.5}         & 63.3           & 29.2           & 45.8               & 57.6           & 74.2          & 61.4             & 12.3              & 60.8                & 85.9          \\
OPT 6.7B (s2)                   & 50           & 97                & 40              & 2.8                                                                           & \multicolumn{1}{c|}{58.6}         & 66.5           & 32.2           & 49.1               & 63.0           & 76.4          & 63.3             & 23.8              & 63.1                & 90.3          \\
OPT 6.7B                        & 0            & 97                & 0               & 4.5                                                                           & \multicolumn{1}{c|}{59.8}         & 68.0           & 32.4           & 50.2               & 68.4           & 75.8          & 67.2             & 20.9              & 65.3                & 90.2          \\
   \midrule
Falcon 7B (s2)                  & 56           & 95                & 56              & 2.2                                                                           & \multicolumn{1}{c|}{64.8}         & 73.6           & 38.6           & 55.3               & 68.4           & 78.9          & 67.6             & 40.4              & 67.1                & 93.4          \\
Falcon 7B (s1)                  & 0            & 94                & 0               & 4.1                                                                           & \multicolumn{1}{c|}{65.2}         & 72.2           & 39.1           & 55.4               & 70.6           & 78.4          & 69.2             & 40.5              & 67.5                & 93.1          \\
Falcon 7B                       & 0            & 1                 & 0               & 6.6                                                                           & \multicolumn{1}{c|}{66.8}         & 74.6           & 40.2           & 57.7               & 73.5           & 79.4          & 74.5             & 40.4              & 67.2                & 94.0          \\
   \midrule
Llama 7B (s2)                   & 51           & 65                & 67              & 2.9                                                                           & \multicolumn{1}{c|}{66.4}         & 73.8           & 39.6           & 54.8               & 69.9           & 77.9          & 70.7             & 48.5              & 68.6                & 93.8          \\
Llama 7B (s1)                   & 0            & 62                & 0               & 4.8                                                                           & \multicolumn{1}{c|}{67.1}         & 75.2           & 40.1           & 55.2               & 73.4           & 77.7          & 71.5             & 49.6              & 67.1                & 94.2          \\
Llama 7B                        & 0            & 0                 & 0               & 6.6                                                                           & \multicolumn{1}{c|}{68.4}         & 75.5           & 42.1           & 69.9               & 74.8           & 78.7          & 73.1             & 49.9              & 69.8                & 95.4          \\ \bottomrule
\end{tabular}%
}
\end{table}
\begin{table}[t]
\centering
\caption{MMLU five-shot accuracy. Models finetuned with different activation functions have similar performance.\small{* Denotes we replace the SiLU function in Llama's SwiGLU activation function with ReLU.} }
\label{tab:mmlu-5shot}
\resizebox{0.8\textwidth}{!}{%
\begin{tabular}{@{}cccc|cccc@{}}
\toprule
Model     & Activation & FLOPS(\%) & Avg & Humanities & STEM & Social Sciences & Other \\ \midrule
Falcon 7B & SiLU       & 100       & 26.4    & 24.8       & 27.4 & 27.2            & 26.2  \\
Falcon 7B & GELU       & 100       & 27.7    & 28.1       & 26.0   & 28.0              & 29.4  \\
Falcon 7B & ReLU       & 62        & 27.9    & 26.0         & 26.5 & 31.8            & 27.9  \\ \midrule
Llama 7B  & SiLU*       & 100       & 35.1    & 37.9       & 30.2 & 37              & 37.1  \\
Llama 7B  & GELU       & 100       & 35.9    & 38.4       & 29.4 & 37.6            & 39.5  \\
Llama 7B  & ReLU       & 72        & 34.7    & 34.8       & 31.2 & 36.3            & 37.8  \\ \bottomrule
\end{tabular}%
}
\end{table}

\tabref{tab:relufy-results-table} shows that different stages of the relufication process do not significantly reduce zero-shot accuracy while using significantly less compute. The sparsity is broken down into three categories: up, down, and QKV projections. Notably, the input to QKV is less sparse than FFN projections, which opens an interesting avenue for future research.
We note that the small gap in performance between the original vs. relufied models may be partially due to the finetuning process and not necessarily the activation function. Our finetuning is applied only for 30B and 50B tokens for stages 1 and 2, respectively. Putting into prospect and comparing it with 1T tokens of Llama, for example, this is equivalent to 3-5\% of the original training duration. As discussed in \secref{sec:from-scratch-comparison}, according to the scaling properties of LLMs, the gap will be further bridged by additional finetuning steps. 

We also assess the in-context learning ability of the relufied models with the Massive Multitask Language Understanding (MMLU) \cite{hendrycks2021measuring} benchmark in \tabref{tab:mmlu-5shot}.
Our results show that when we augment the original LLMs with different activations and finetune, the few-shot performance does not change significantly either. Moreover, ~\secref{appendix:scaling-opt} in the appendix shows that a larger but relufied model performs better than an original smaller model of the same FLOPS.
Overall, the results affirm that the proposed relufication procedure can decrease the inference FLOPS at various stages and rates while maintaining on-par performance on various tasks.

\section{Applications}

In this section, we discuss promising directions motivated by our investigation in \secref{sec:relufication}. First, we introduce \emph{aggregated sparsity}, showing that ReLU networks reuse previously activated neurons when generating tokens. Hence, we can leverage this to increase the generation speed. Next, we relate aggregated sparsity with speculative decoding to further improve speculative decoding's inference time.
Finally, we briefly discuss a promising direction of using the \emph{shifted ReLU} activation function to improve the sparsity further.

\subsection{Aggregated Sparsity: reusing previously activated neurons}
\label{sec:aggregated-sparsity}
A consequence of using only a small subset of neurons for each token is that if these neurons are shared to some degree, the model still does not use all of the neurons until many tokens are processed. We refer to this as \emph{aggregated sparsity}, which we defined as the ratio of neurons that have not been used up to processing the first $t$ token. Note that this metric will always be non-increasing. Intuitively, it measures the unused capacity of feed-forward neurons for processing a specific prompt.

\looseness=-1 Here in \figref{fig:reuse-acts} we show that for the OPT-6.7B model, on average, about 50\% of all the neurons will be unused across the first 150 tokens of prompts coming from the WikiText dataset. Our empirical results hold for other ReLU models and other datasets. Additionally, in \figref{fig:reuse-vs-random}, we show that this pattern is far from random activation of neurons during the token generation with a rate equal to the average rate of activation usage per token. Let $s_i$ be the activation sparsity of layer $i$ averaged over all tokens. Then, the probability of an activation not used in generating the first $t$ tokens in uniformly random selection is $s_i^t$ . \figref{fig:reuse-vs-random} shows this quantity for two layers $i=8, 24$ for the first $256$ tokens in dashed line. It also shows the real (observed) number of activations being used in the solid line. The fact that the random aggregated sparsity (referred to as random sparsity) is lower than the observed aggregated sparsity (we refer to it as aggregated sparsity) shows a clear pattern of reusing activations.

We can benefit from the overlapping activations by utilizing previously loaded weights from the down projection layer for upcoming tokens. To test this, we initiate with reading 128 tokens. For the subsequent 128 tokens, we intermittently avoid loading new weights for every $\gamma$ token. Using $\gamma=16$ as an example, tokens 129-145 are generated conventionally. However, for tokens 146-161, we retain the existing weight without introducing any new weight. This pattern continues, with every next set of $\gamma$ tokens alternating between conventional generation and weight reuse. In \figref{fig:reuse-perplexity}, we observe only a slight increase in perplexity when using this approximation to address the memory and I/O-intensive nature of LLM inference. This figure contrasts the perplexity obtained from reused activations and random selections. The reuse strategy aligns well with the baseline, whereas random selection notably increases perplexity, highlighting the effectiveness of reusing the already loaded activations for subsequent tokens.

\begin{figure}[t]
\centering
\begin{subfigure}{.285\textwidth}
  \centering
  \includegraphics[width=\textwidth]{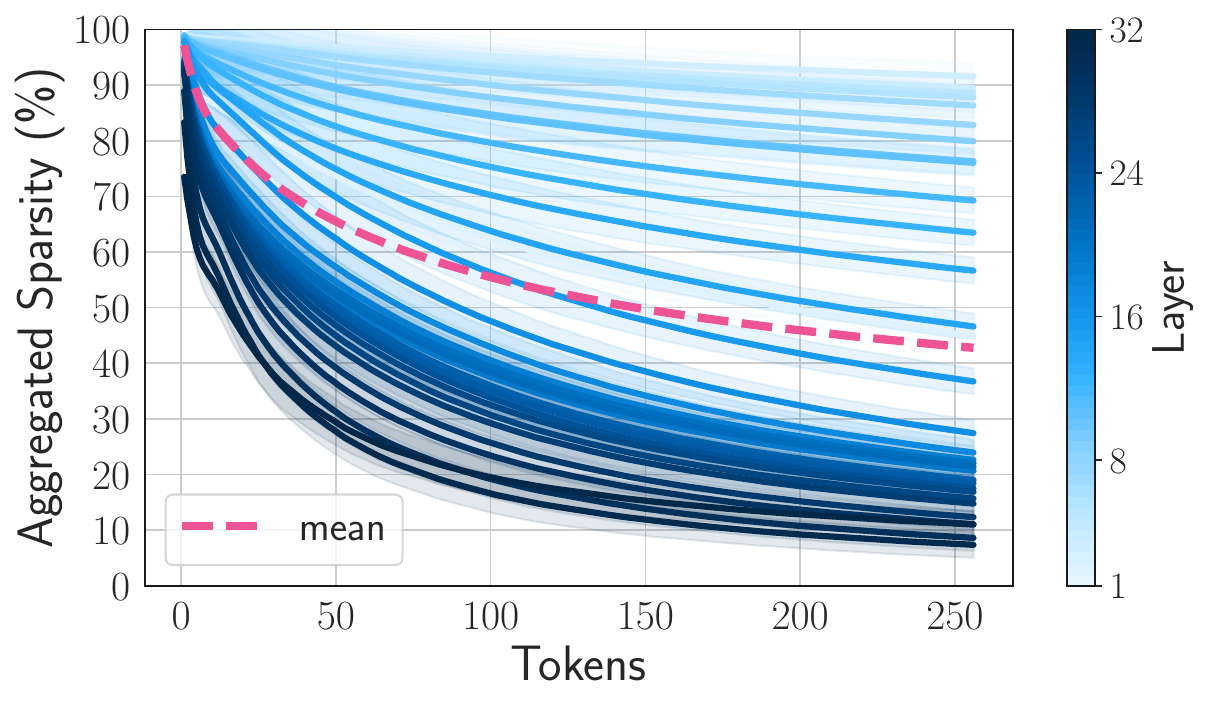}
  \caption{}
  \label{fig:reuse-acts}
\end{subfigure}
\begin{subfigure}{.225\textwidth}
  \centering
  \includegraphics[width=\textwidth]{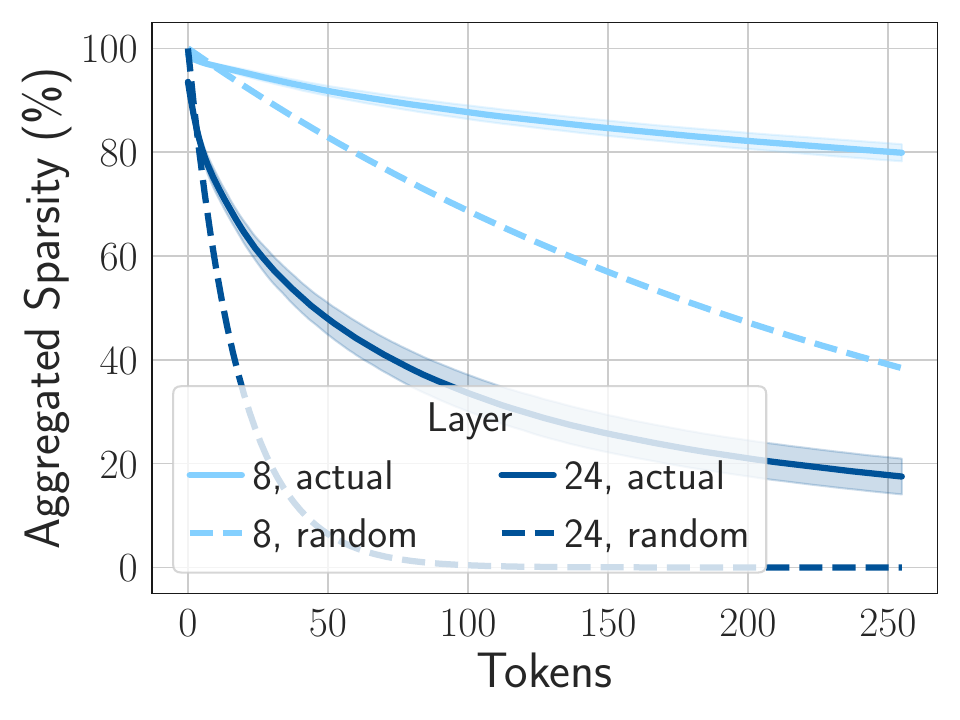}
   \caption{}
  \label{fig:reuse-vs-random}
\end{subfigure}
\begin{subfigure}{.225\textwidth}
  \centering
  \includegraphics[width=\textwidth]{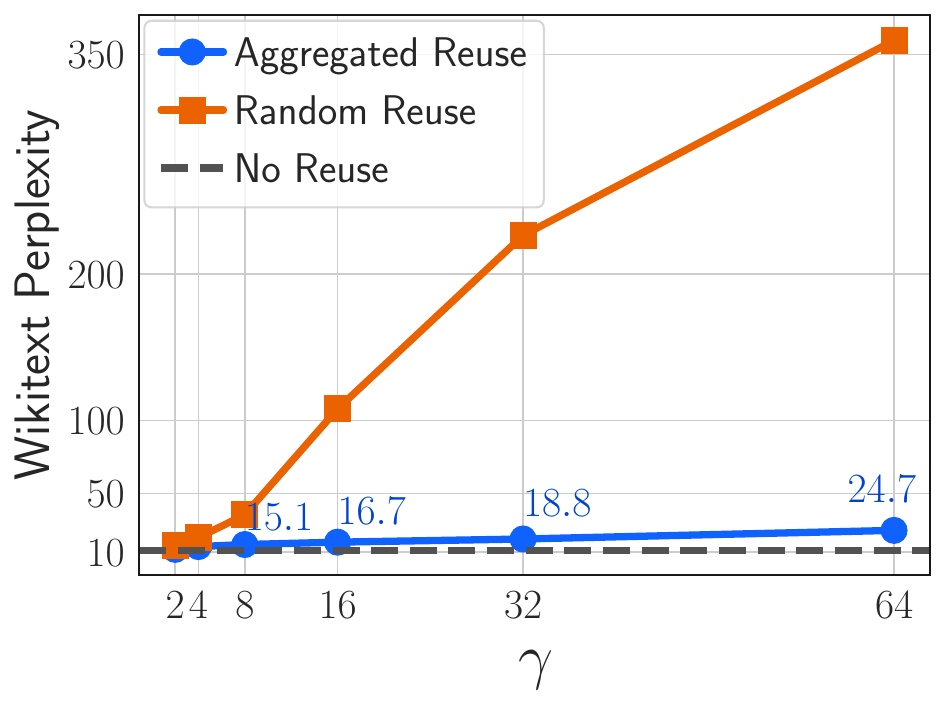}
   \caption{}
   \label{fig:reuse-perplexity}
\end{subfigure}
\begin{subfigure}{.23\textwidth}
  \centering
  \includegraphics[width=\textwidth]{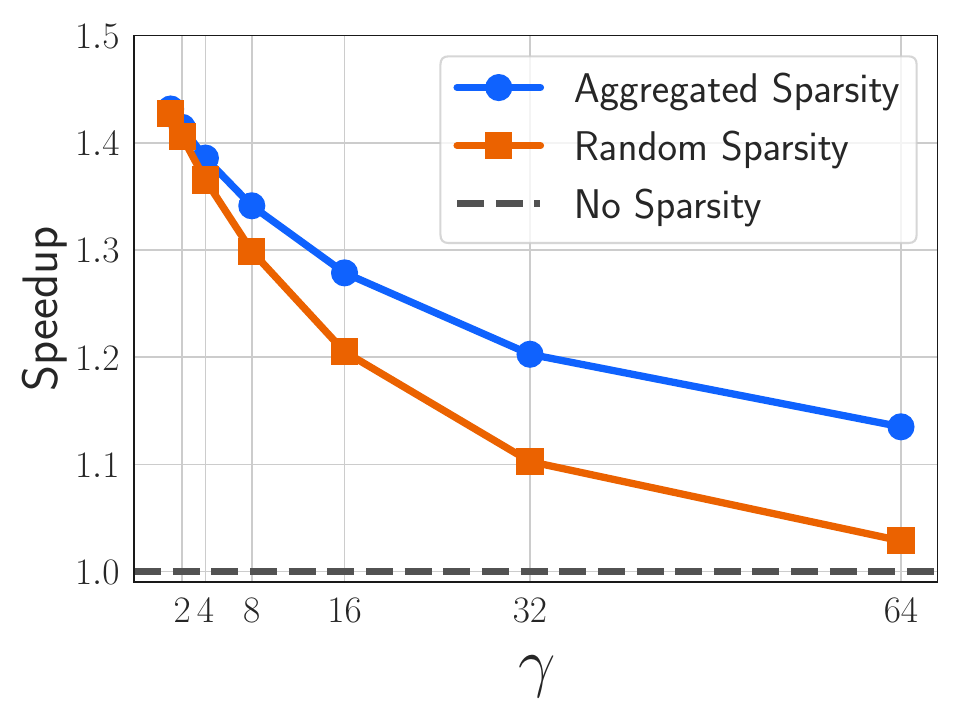}
   \caption{}
   \label{fig:speedup-over-speculative}
\end{subfigure}
\caption{\textbf{(a)} Aggregated sparsity of different layers and their mean. \textbf{(b)} Aggregated sparsity during token generation and comparison with a random sparsity. \textbf{(c)} Perplexity, based on the number of tokens for which loaded weights from previous tokens are reused. The dashed line represents no reuse, the solid blue line shows the case with activation reuse according to aggregated sparsity, and the orange line depicts the perplexity when activations are reused according to a random sparsity. \textbf{(d)} The inference speedup of speculative decoding with aggregated sparsity and with random sparsity. Speedup equal to 1.0 is the standard version of speculative decoding.} 
\label{fig:activation-reuse}
\end{figure}

\subsection{Activation sparsity and speculative decoding}
\label{sec:speculative-decoding}
As highlighted in \secref{sec:aggregated-sparsity}, activation reuse happens for multiple consecutive tokens. When multiple consecutive tokens are processed together, we can save the I/O (i.e., transferring weights to GPU/CPU as discussed in Appendix~\ref{sec:appendix-discussion-efficiency}) associated with activations that are not used in any of them. If the reuse was not happening, and the sparsity of all tokens was purely random, the aggregated sparsity would shrink exponentially and quickly diminish.
Speculative decoding~\cite{leviathan2023fast} is a related technique that uses a smaller model $M_q$ to propose $\gamma$ tokens and a larger model $M_p$ to verify those tokens and select matching ones. It improves the runtime of the model by avoiding running $M_p$ sequentially.

To improve  speculative decoding, aggregated sparsity can trim down the portion of the model that needs to be run. Instead of running the full model, only the non-sparse parts need to be evaluated, which will reduce I/O and compute latency. Suppose the average aggregated sparsity of $M_p$ for $\gamma$ tokens is $\bar{s}_{\text{agg}}(\gamma)$, and cost of running $M_q$ over $M_p$ is $c$. Then the expected latency speedup when going from standard speculative decoding to sparse speculative decoding is $\frac{c \gamma+1}{c \gamma+(1-\bar{s}_{\text{agg}}(\gamma))}$.

\figref{fig:speedup-over-speculative} compares sparse speculative decoding to the standard version for OPT 6.7B model. As a case study, for $\gamma=16$, the sparse version has a 1.27x speedup over the standard speculative decoding. If the aggregated sparsity was random over different tokens, the speedup would have been only 1.20x. Note that even random sparsity will lead to speedup over standard speculative decoding. This further shows the value of relufication. However, the speedup due to random sparsity would diminish quickly in comparison to aggregated sparsity as we go for larger $\gamma$. For example, for $\gamma=64$ the speedup is almost negligible, while the speedup for the aggregated sparsity is around 1.14x. Further discussion and details are postponed to Appendix~\ref{sec:appendix-spec-decode}, where we compare sparse speculative decoding, standard speculative decoding, and autoregressive decoding and discuss optimal $\gamma$ in the case of sparse speculative decoding.

\subsection{The shifted ReLU activation}
\label{sec:shifted-relu}

Our work in this section is motivated by the observation from \secref{sec:relufication}, where, comparing \figref{fig:isolated-llama-1x} with \figref{fig:isolated-falcon-1x} revealed that the relufied Llama has much less sparsity (65\%) than the relufied Falcon model (95\%). In addition, we build on two of our previous findings. First, the preactivation distribution of the relufied Llama~(\figref{fig:preact-llama}) includes a considerable mass after the cutoff value at zero. Second, the shape of the preactivation distribution does not change before and after the relufication process (\figref{fig:preact-llama} and \figref{fig:preact-llama-1x}).

Therefore, we may be able to shift the preactivation distribution to the left to put more volume before the cutoff at 0. To this end, for preactivation input $x$, rather than applying $\text{ReLU}(x)$, we use $\text{ReLU}(x-b)$ where $b \in \mathbb{R}$ is a constant scalar. We propose to set the value $b$ based on the preactivation distribution. For instance, based on the distribution in \figref{fig:preact-llama-1x}, setting $b = 1$ and hence using $\text{ReLU}(x-1)$ as our activation function will result in dropping $95\%$ of the preactivations and make it significantly sparser. Another benefit of this approach is simplicity, as this does not require changing the loss function or the training regime.

Figure~\ref{fig:shifted-acc} shows that the shifted ReLU activation function has on-par accuracy with the ReLU activation function. Moreover, similar to our observation in \secref{sec:relufication}, the shifted ReLU activation quickly recovers the lost performance due to the drastic change of activation function, while it also maintains a very high-level activation sparsity during the finetuning stage. The gap between shifted ReLU and ReLU is wider in the early stages of training, and it narrows down when more tokens are seen.

A deeper investigation of ReLU-variants that can promote sparsity without sacrificing performance is an appealing future direction. Moreover, it will be interesting to study the impact of the shifted ReLU for stage-2 of our relufication process where the sparsity level is usually not very high.

\begin{figure}[t]
\centering
\begin{subfigure}{.35\textwidth}
  \centering
  \includegraphics[width=\textwidth]{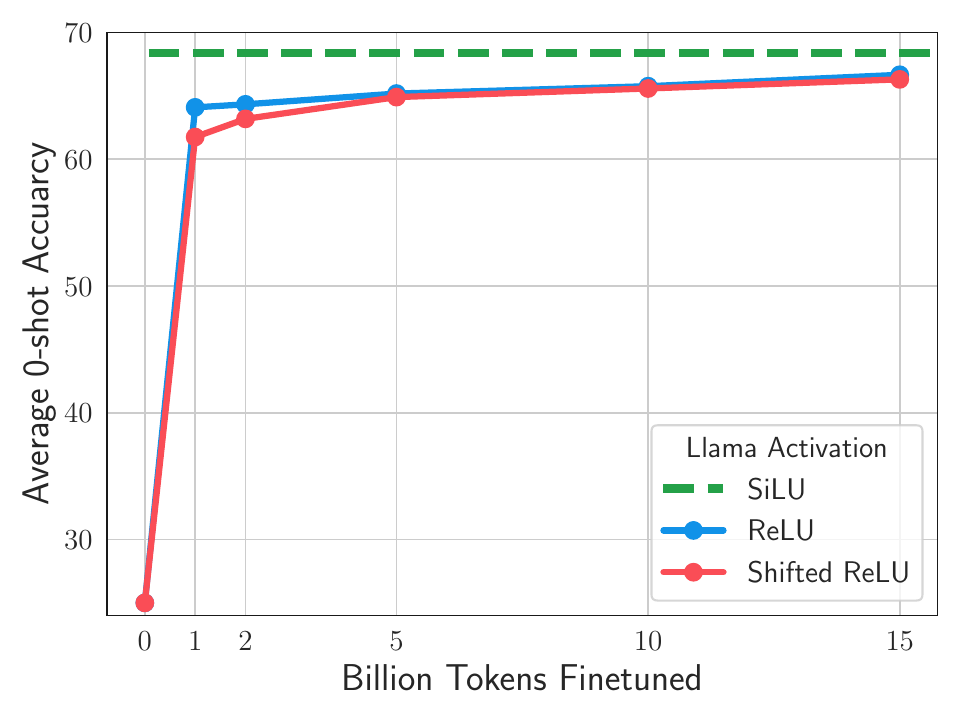}
   \caption{}
  \label{fig:shifted-acc}
\end{subfigure}\hspace{1.5cm}
\begin{subfigure}{.35\textwidth}
  \centering
  \includegraphics[width=\textwidth]{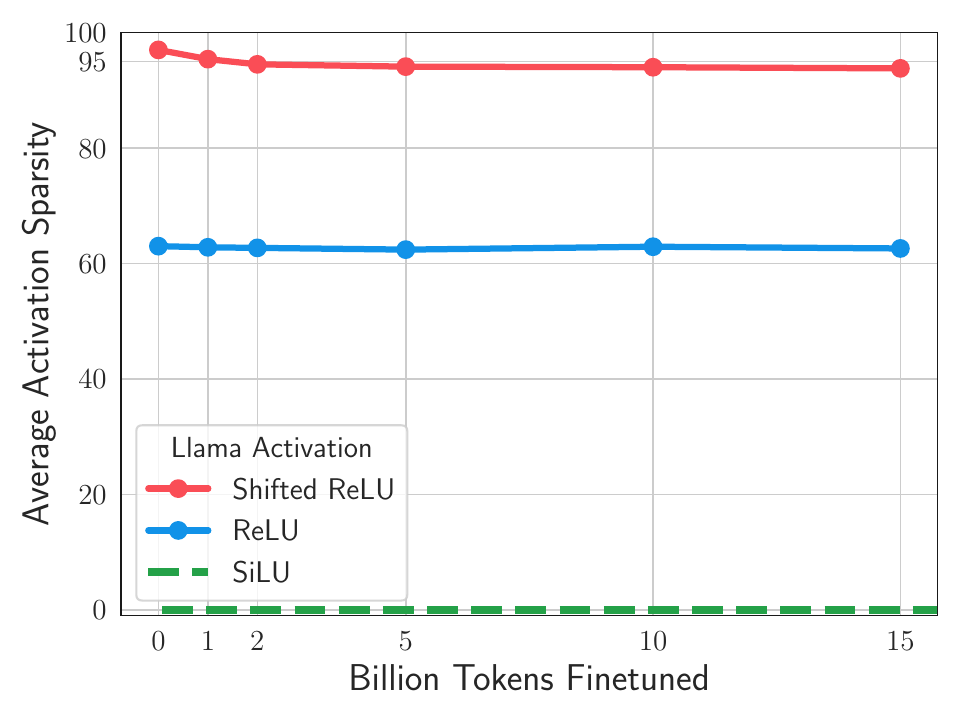}
  \caption{}
  \label{fig:shifted-sparsity}
\end{subfigure}
\caption{The effect of shifted ReLU on Llama model. \textbf{(a)} The performance is almost the same as the original ReLU. \textbf{(b)} Shifted ReLU (i.e., $\text{ReLU}(x-1)$) is much sparser than the original ReLU.}
\label{fig:shifted-relu-figures}
\end{figure}

\section{Conclusion}
In this study, we conducted a large-scale investigation of the activation functions, and we have shown that the choice of activation functions during pretraining and finetuning does not have a significant impact on performance while using ReLU can provide an additional benefit of leading to activation sparsity and more efficient inference. To bridge the gap between existing pre-trained models and our work, we have \emph{relufied} several models to incorporate the ReLU activation function into the architecture of these already pre-trained models. We have shown that across several zero-shot and few-shot tasks, the ReLU-based LLMs perform similarly to their non-ReLU models at a significantly reduced computation. In addition, after observing sparsity patterns in ReLU LLMs, we explored a few promising directions to improve the token generation speed through \emph{aggregated sparsity} and achieve greater efficiency using ReLU-based activation functions like \emph{shifted ReLU}.

We believe our work is among the few studies that investigate changes in the architectural components of LLMs on a large scale. We hope our findings motivate the community to further investigate the advantages of well-structured activation sparsity, ultimately enhancing the efficiency of these models.

\section*{Acknowledgments}
The authors would like to thank Fartash Faghri, Minsik Cho, Thomas Merth, and Mohammad Samragh for their invaluable discussions and feedback on this project.

\clearpage
\newpage

\bibliography{refs}
\bibliographystyle{plainnat}
\newpage
\clearpage
\appendix
\section*{Appendix}
\label{sec:appendix}

\section{Extended Related Works}
\label{sec:appendix-relatedworks}

\textbf{Activation Functions.}
ReLU, introduced by \cite{Fukushima1969VisualFE}, remains a predominant activation function for deep neural networks and was notably utilized in the original transformers work \cite{vaswani2023attention}. SwiGLU \cite{shazeer2020glu} has been shown to enhance performance when replacing ReLU in feedforward layers and is a feature in models like Llama \cite{Llamav1paper}. \citet{narang2021transformer} conducted an extensive comparison of various activation functions, such as GeLU, SiLU \cite{hendrycks2016gaussian, elfwing2018sigmoid}, ELU \cite{clevert2016fast}, SeLU \cite{klambauer2017selfnormalizing}, and GLU variants \cite{10.5555/3305381.3305478}, identifying certain advantages over ReLU. Our paper and results differ from theirs by training billion scale models and data as opposed to their smaller scale ones.
Furthermore, our results indicate that extended training can diminish the performance gap between ReLU and these other functions, also leading to savings in computational costs.

\textbf{Activation Sparsity.}
A body of prior research \cite{kurtz2020inducing, han2023retrospective, song2021training} has demonstrated that increasing sparsity can lead to reductions in both inference and training times. Dejavu \cite{liu2023deja} and \cite{li2022large} observed pronounced sparsity in activations when using the ReLU function in feedforward layers. These studies propose that predicting this sparsity can further boost inference speeds. Similarly, \cite{jaszczur2021sparse} employed ReLU activations and introduced a controller to actively promote sparsity. Notably, these studies predominantly focus on networks employing ReLU activations, leaving out those with alternative activation functions. In contrast, our approach modifies networks by substituting other activation functions with ReLU. We then fine-tune these networks to achieve activation sparsity in the MLP layer post-ReLUfication. We further illustrate that inserting ReLU prior to the QKV and Feedforward layers can substantially reduce FLOPS, albeit at a minor cost to accuracy. Unlike the aforementioned studies, we do not utilize a sparsity predictor to minimize FLOPS.

\textbf{ReLU in Attention Mechanisms.}
\looseness=-1 Beyond the activation function in the MLPs of large language models, a softmax activation is often employed within the attention module. Prior studies have indicated that it's feasible to replace this softmax with ReLU without compromising accuracy \cite{wortsman2023replacing, shen2023study, hron2020infinite}. This avenue of research is distinct from our approach of Relufication, which specifically focuses on activations preceding weight multiplications.

\textbf{Model compression for efficient inference}
Quantization, pruning and distillation are the main three techniques for compressing neural networks \cite{zhu2023survey}. Quantization has been used to reduce model size and faster inference \cite{dettmers2023qlora, liu2023llmqat, park2023lutgemm, dettmers2022llm, lin2023awq, lee2023owq, dettmers2023spqr, kim2023squeezellm, chee2023quip, xiao2023smoothquant}. 
The quantized model occupies less space reducing the memory latency \cite{frantar2023gptq, kim2023memoryefficient}. Reluification is orthogonal to quantization and reduces the amount of memory required to be loaded and can further decrease the memory latency. Distillation \cite{hsieh2023distilling, hinton2015distilling, gu2023knowledge, mirzadeh2020improved, agarwal2023gkd} is another technique to train smaller models. This is orthogonal to using ReLU activations as any activation can be used in distillation methods.
Sparsifying or pruning weights of neural networks \cite{frantar2023sparsegpt, jaiswal2023emergence, zhang2023pruning, sun2023simple, santacroce2023matters, ma2023llm} can reduce computation and inference time. Weight sparsity is usually unstructured and hard to implement for hardware, but the sparsity induced by ReLU can easily be implemented as a matrix multiplication of non zero rows. Weight sparse models can be combined with our relufication for further decrease in compute.  

\textbf{Mixture of Experts.}
 Mixture of Experts (MoE) LLMs usually subdivide the feed-forward layer into multiple experts. A router is then employed to selectively and sparsely activate these experts \cite{shazeer2017outrageously, fedus2022switch, nllbteam2022language}. Similar to our work, MoE is a form of activation sparsity but in a group form and can be seen as a subset of sparse activation. Subsequent studies have further refined the inference and training methodologies for MoE models \cite{puigcerver2023sparse,hwang2023pre,yi2023edgemoe, du2022glam, kong2023serving, rajbhandari2022deepspeedmoe, zoph2022st, Chen2022TaskSpecificEP, Hazimeh2021DSelectkDS}. MoE can be also combined with Relufication, having sparsity inside FFN of each expert.

Another line of work is MoEfication of networks that have sparse activations by subdividing neurons \cite{zhang2022moefication}. Relufication can also help MoEfication be applicable to a wider range of networks by increasing sparsity of FFNs. For a more in depth review of mixture of expert models we refer the reader to \cite{Fedus2022ARO}.

\textbf{Speculative Decoding and Sparsity.} Speculative decoding is a method that aims to improve model latency when faced with memory bandwidth constraints \cite{leviathan2023fast, kim2023speculative}. It involves using a smaller model to predict the next tokens, with a larger model subsequently verifying multiple tokens in a single operation. In this work, we examine the direct effects of incorporating sparsity into speculative decoding. We show that adding sparsity can lead to performance improvements in speculative decoding. Additionally, we provide guidelines on selecting parameters for speculative decoding when sparsity is introduced.

\section{Discussion on Activation Sparsity and Inference Efficiency}
\label{sec:appendix-discussion-efficiency}
The primary motivation behind our work is to enhance \emph{efficiency}, and we believe it is essential to provide a precise definition of this term. Throughout the main text, we predominantly use FLOPS as our efficiency metric. In this section, we argue why, in the presence of \emph{activation sparsity}, FLOPS can serve as a suitable proxy for measuring various efficiency metrics.

Firstly, it is important to be reminded of the two major factors contributing to the efficiency of Large Language Models (LLMs): (1) the total amount of computation and (2) input/output (IO) transfer—i.e., transferring parameters from RAM to CPU/GPU for calculations. Notably, for today's large models, factor (2) acts as the major bottleneck during the inference phase. We refer the reader to the detailed analysis by \citet{liu2023deja}. Ultimately, for a specific target device and assuming an efficient implementation, we believe that the most practical measure of efficiency is \emph{latency} (e.g., the average time to generate a token). However, each device possesses its unique properties, necessitating a more ubiquitous proxy metric. 

\begin{figure}[h]
\centering
\begin{subfigure}{.58\textwidth}
  \centering
  \includegraphics[width=\textwidth]{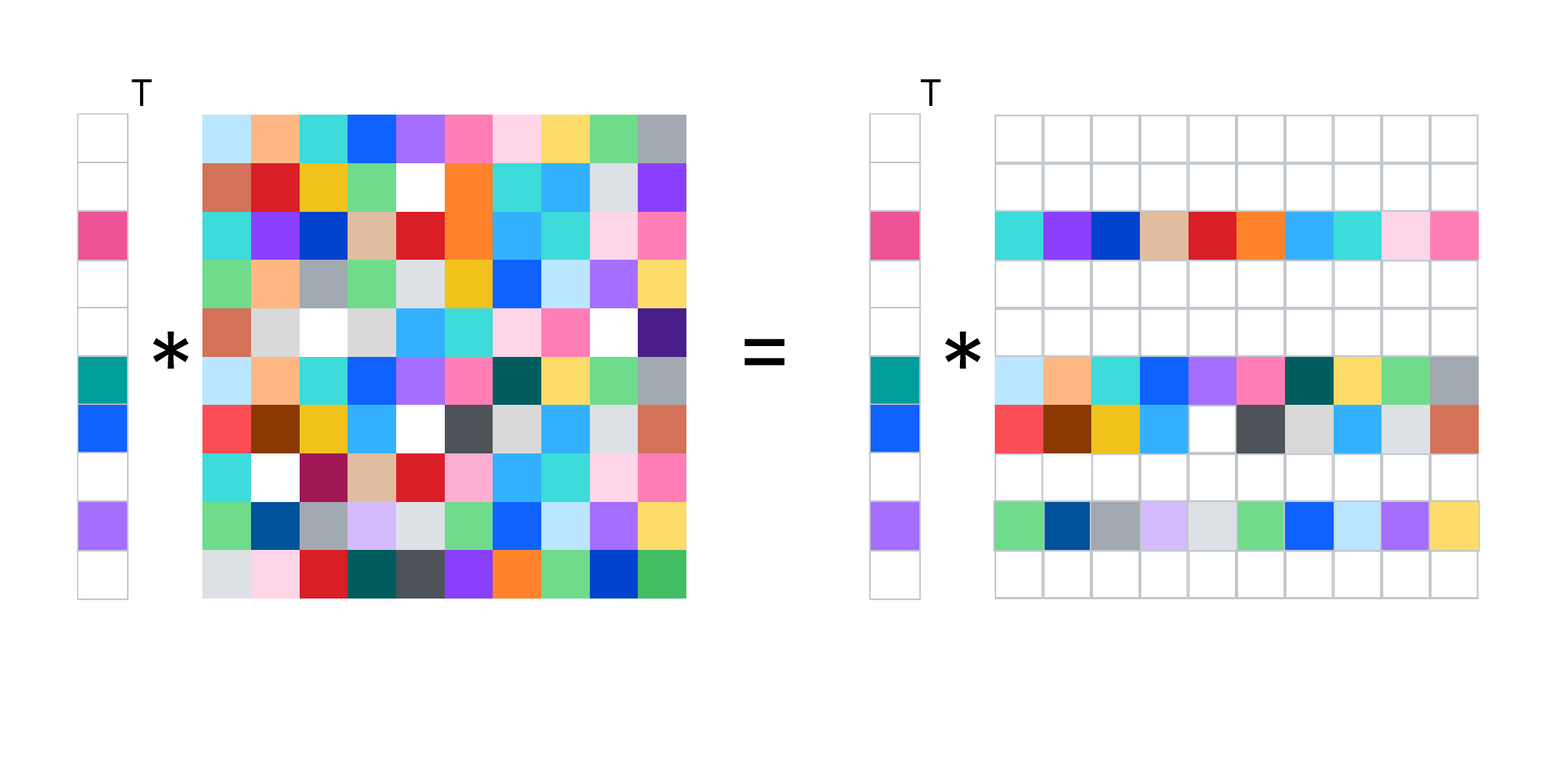}
  \caption{sparse vector, dense matrix multiplication: by skipping rows, we reduce both the weight transfer (i.e., loading these rows for computation) and computation (i.e., the result will be zero).}
  \label{fig:appendix-matvec-illustration}
\end{subfigure}\hfill
\begin{subfigure}{.4\textwidth}
  \centering
  \includegraphics[width=\textwidth]{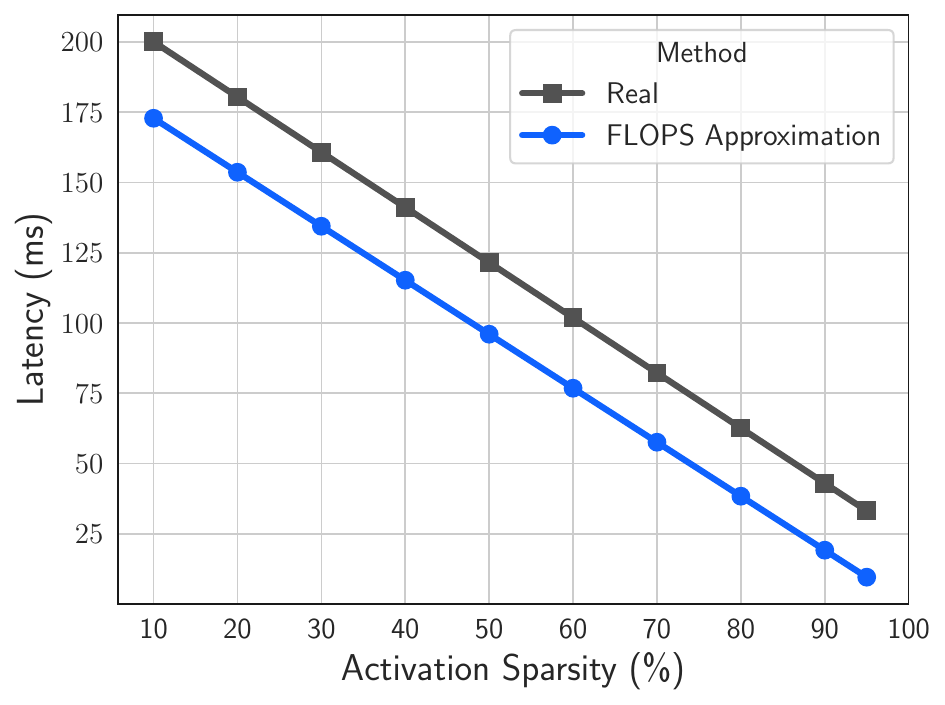}
   \caption{Comparing FLOPS versus real latency for OPT model (FFN).} 
  \label{fig:FLOPS-approx}
\end{subfigure}
\caption{For LLMs that have sparse activations, FLOPS is a good approximation of the real latency.}
\label{fig:appendix-efficiency}
\end{figure}

We argue that the way we calculated \emph{FLOPS} in our paper and is greatly influenced by \emph{activation sparsity} can reasonably approximate \emph{efficiency}. Here are our reasons:
\begin{itemize}[leftmargin=*]
    \item \textbf{Reduced Computation}: As shown in \figref{fig:appendix-matvec-illustration}, with activation sparsity, we have a sparse vector-dense matrix multiplication at inference, while this will be a sparse-matrix-dense matrix multiplication during training. It is important to note that this is a semi-structured sparsity (unlike magnitude weight pruning), and we can leverage the fact that we are transferring weights in large chunks (i.e., \emph{rows}). Modern accelerators already support sparse operations\footnote{For example, both \href{https://docs.nvidia.com/cuda/cusparse/index.html}{cuSPARSE} on NVIDIA CUDA® and \href{https://developer.apple.com/documentation/accelerate/sparse_solvers/sparse_matrix_and_dense_matrix_multiplication}{Accelerate} on Apple devices.} and we can build on these existing tools. 
    \item \textbf{Reduced IO Transfer}: During inference, weights need to be transferred to the device cache for computation (e.g., from RAM to CPU cache or GPU VRAM to GPU cache). This step constitutes the main bottleneck during token generation. For instance, approximately $99.3\%$ of the total latency is attributed to IO, as indicated by \citet{liu2023deja}. However, by storing matrices in a row-major order, we can \emph{skip loading} unnecessary rows as the output will be zero.
\end{itemize}

Overall, as depicted in Figure \ref{fig:FLOPS-approx} based on the calculations by~\citet{liu2023deja}, we demonstrate that for the OPT model on an NVIDIA A100 node,  counting FLOPS provides a reasonable approximation to and is highly correlated with the time needs to generate tokens, especially, for LLMs with activation sparsity.

\section{Activation Sparsity and Speculative Decoding}
\label{sec:appendix-spec-decode}

Speculative decoding~\cite{leviathan2023fast} is a technique that uses a smaller model $M_q$ to propose $\gamma$ tokens and a larger model $M_p$ to verify those tokens and selects matching ones. This technique improves the runtime of the model by avoiding running $M_p$ sequentially per all tokens. To further improve the speculative decoding inference, we can leverage sparsity as follows.

\textbf{Latency model. }
 We assume a simple conceptual model for latency in speculative decoding. Following Deja Vu \cite{liu2023deja} latency can be broken down into compute and I/O latency. The compute latency is negligible to I/O. Also, notice that the Speculative decoding is meant for the constraints that memory bandwidth is the bottleneck. Therefore we only compare I/O latency between sparse and non-sparse models. If the average aggregated sparsity of $M_p$ for $\gamma$ tokens is $\bar{s}_{\text{agg}}(\gamma)$, and runtime of $M_p$ is $T$, we approximate the latency of running $M_p$ for $\gamma$ consecutive tokens with $(1-\bar{s}_{\text{agg}}(\gamma))T$. As discussed in the previous section, this is a good approximation of real latency.

\textbf{Theoretical latency improvement.} Assume the smaller model $M_q$ operates $c$ times faster than the cumbersome model $M_p$. As per the primary text, token acceptances are assumed to follow an independent and identically distributed (i.i.d.) behavior. Denote \( \alpha \) as the expected probability of a token generated by \( M_q \)  being approved by \( M_p \).   The following theorems hold:

\newtheorem{theorem}{Theorem}

\begin{theorem}
\label{thm:speedup-theorem1}
The expected improvement factor in latency for speculative decoding with sparsity, over standard speculative decoding is $\frac{c \gamma+1}{c \gamma+(1-\bar{s}_{\text{agg}}(\gamma))}$.
\end{theorem}
\begin{proof}
The amount of time required to run sparsified model is quantified as \(Tc\gamma + (1-\bar{s}_{\text{agg}}(\gamma))T\). It is the time of running a smaller model plus a larger model's non-sparse portion.  Run time of speculative decoding without sparsity is \(Tc\gamma + T\). The number of generated tokens in both is the same. Therefore the relative speedup of using sparsity is given by: 
$\frac{Tc \gamma+T}{Tc\gamma + (1-\bar{s}_{\text{agg}}(\gamma))T}$.
\end{proof}

\begin{theorem}
\label{thm:speedup-theorem2}
The expected improvement factor in latency, when combining sparsity with speculative decoding against normal (autoregressive) decoding using only \( M_p \), is $\frac{ 1 - \alpha^{\gamma+1} }{(c \gamma + (1 - \bar{s}_{\text{agg}}(\gamma)))(1- \alpha)}$.
\end{theorem}
\begin{proof}
Similar to theorem above \(Tc\gamma + (1-\bar{s}_{\text{agg}}(\gamma))T\) gives the time required for sparse speculative decoding.  According to the original paper, the standard speculative decoding yields an average of \(\frac{1-\alpha^{\gamma+1}}{1-\alpha}\) tokens generated per each run~\cite{leviathan2023fast}. Thus, the anticipated run time when generating tokens with sparsity during speculative decoding becomes $
\frac{(c \gamma + (1 - \bar{s}_{\text{agg}}(\gamma)))(1- \alpha)}{ 1 - \alpha^{\gamma+1} } T $. 
Given the runtime for producing a single token via an autoregressive approach is \(T\), the inverse of this fraction gives the desired results.
\end{proof}

\textbf{Optimal $\gamma$.} The optimal $\gamma$ for speculative decoding can be found by optimizing the speedup factor equation in Theorem \ref{thm:speedup-theorem2}. When sparsity is not present, the equation can be solved numerically, but for reluified networks, the aggregated sparsity for different $\gamma$'s will affect the final results. We have found optimal $\gamma$s based on $\bar{s}_{\text{agg}}(\gamma)$ for OPT 6.7B and presented the results in figure \ref{fig:optimal-gamma}. The chosen $\gamma$ for sparse speculative decoding is smaller than standard speculative decoding since higher $\gamma$ will result in less sparsity. The gap in $\gamma$ is always less than 20\%. Also, in figure \ref{fig:speedup-over-autoregressive}, it can be seen for the specific case of $\alpha=0.8, c = 0.02$, the sparse speculative decoding has the highest speed-up factor over autoregressive at $\gamma=10$s vs standard version's optimal point which happens for $\gamma=12$. Sparse speculative decoding at $\gamma=12$ is better than standard speculative decoding at $\gamma=12$, while sparse speculative decoding at $\gamma=10$ beats both. Another observation from \ref{fig:speedup-over-autoregressive} is for the case of purely random sparsity, the benefit of sparse speculative decoding would diminish over standard speculative decoding in higher $\gamma$s. In contrast, the benefits of aggregated sparsity would last for larger values of $\gamma$.

\begin{figure}[h]
\centering
\begin{subfigure}{.38\textwidth}
  \centering
  \includegraphics[width=\textwidth]{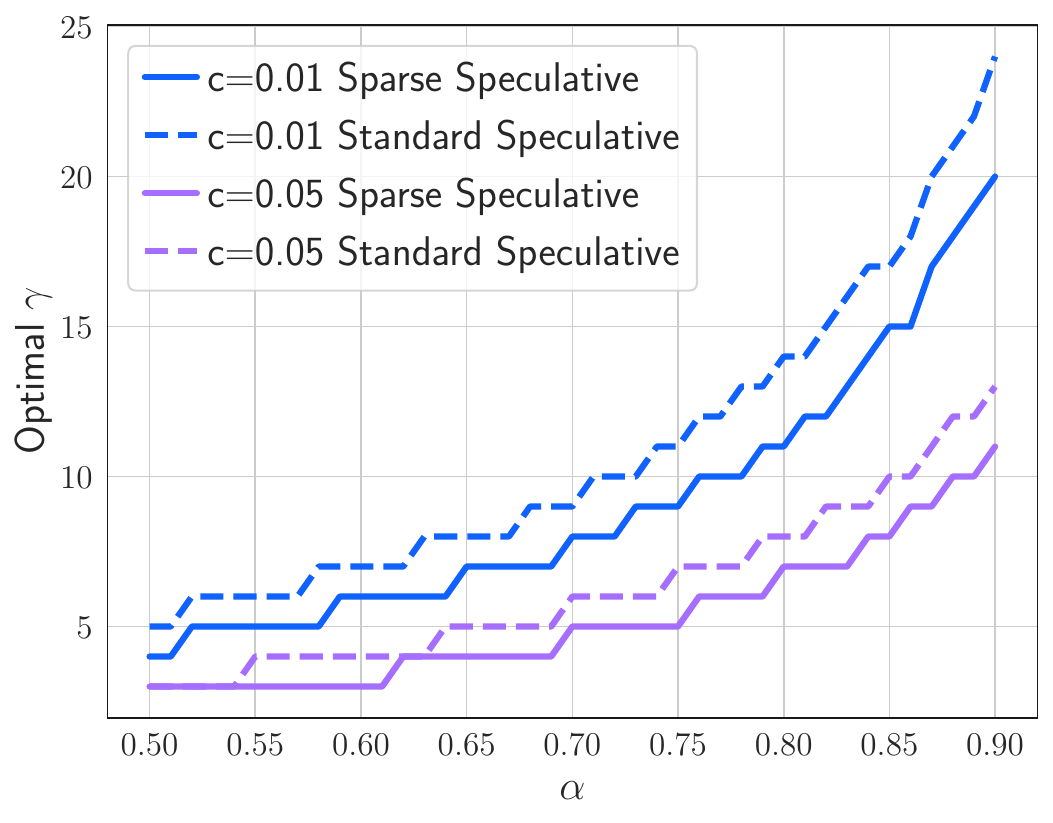}
  \caption{}
  \label{fig:optimal-gamma}
\end{subfigure}\hspace{1.5cm}
\begin{subfigure}{.41\textwidth}
  \centering
  \includegraphics[width=\textwidth]{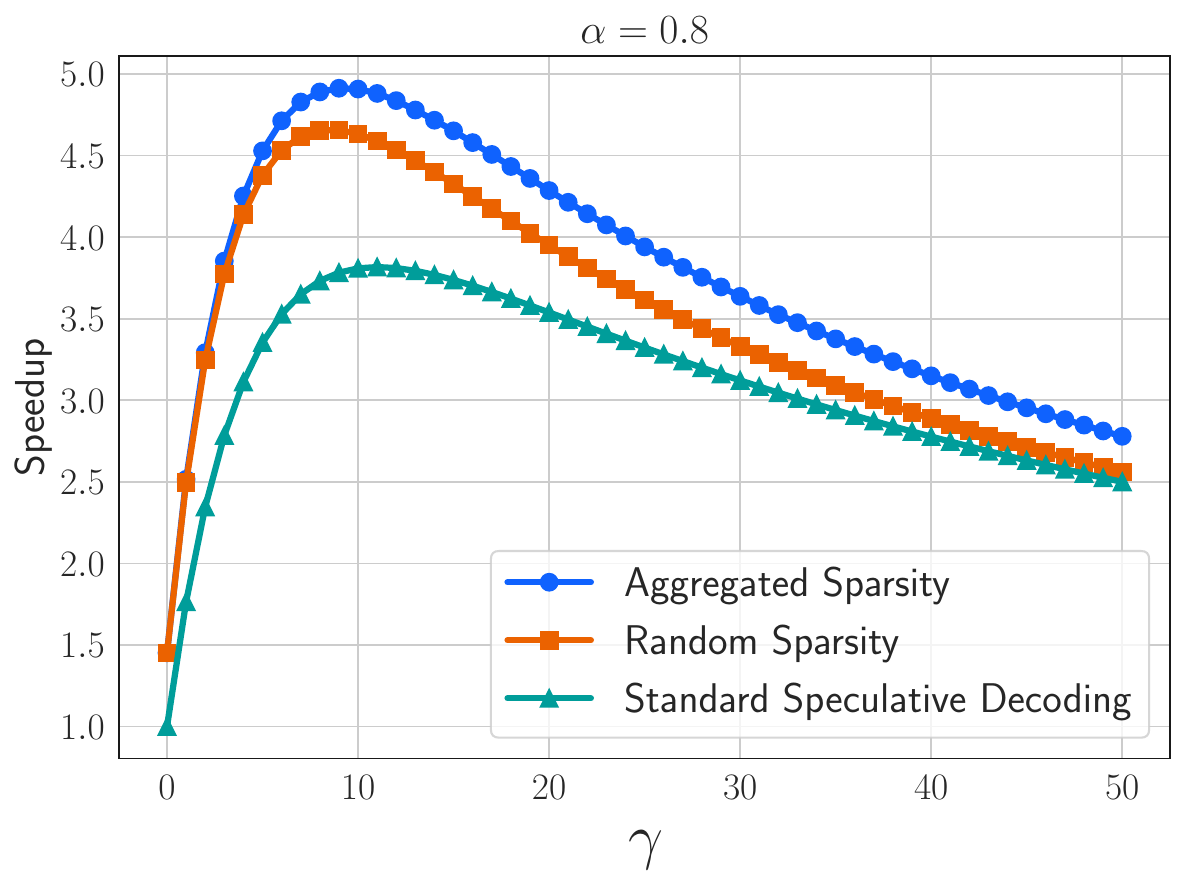}
   \caption{}
  \label{fig:speedup-over-autoregressive}
\end{subfigure}

\caption{ (a) optimal $\gamma$ for sparse speculative decoding (b) speed up of sparse speculative decoding and standard speculative decoding over autoregressive decoding when $\alpha=0.8$ and $c=0.02$}
\label{fig:speculative-decoding}
\end{figure}

\section{Pre-activation Distribution of OPT models trained from scratch}
\label{sec:appendix-preact-dist-opt}
The distribution of pre-activation inputs is suspected to be the main factor determining the amount of sparsity. As we saw in \secref{sec:relu_stage_1}, the pre-activation distribution for Llama and Falcon differs a lot. One may wonder that if we control for the training data and optimization algorithm, would the shapes of the distribution differ? To this end, we train OPT 1.3B models from scratch using our four variants of the activation function and depicted the pre-activation distribution along the training in \figref{fig:appendix-preact-dist-opt}. They start from the exact figure but gradually diverge. From SiLU to ReLU (increasing $\beta$), the pre-activation distribution becomes more concentrated around 0 and would be almost uni-modal. Further investigations on the dynamics of pre- and post-activations and their relation to efficiency and accuracy are left as an exciting future direction for research.
\begin{figure}[h]
\centering
\begin{subfigure}{.21\textwidth}
  \centering
  \includegraphics[width=\textwidth]{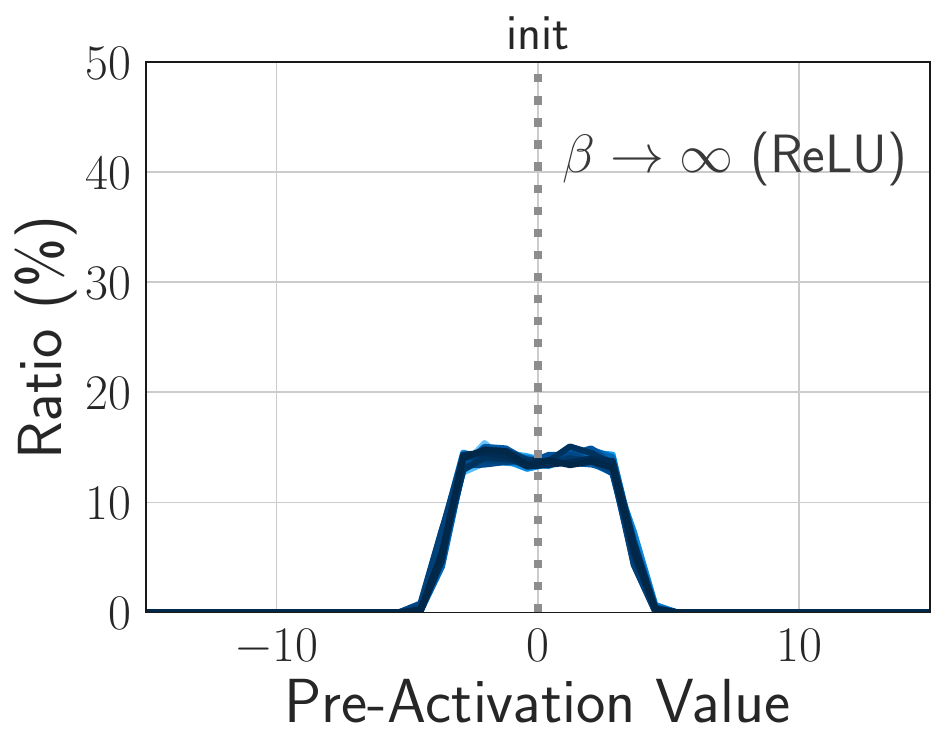}
\end{subfigure}\hfill
\begin{subfigure}{.21\textwidth}
  \centering
  \includegraphics[width=\textwidth]{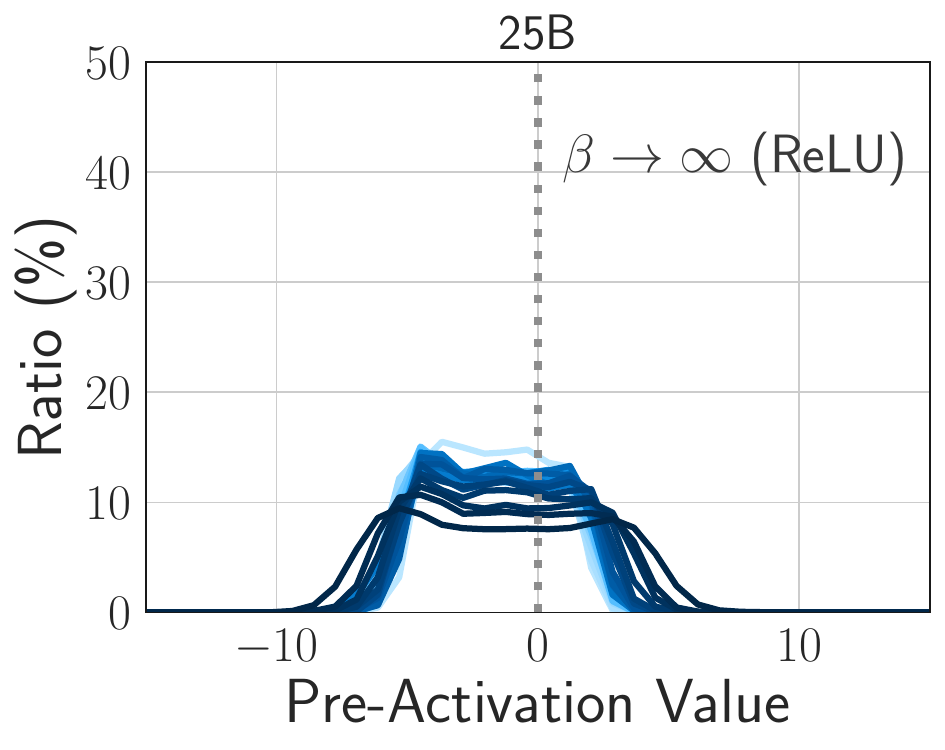}
\end{subfigure}\hfill 
\begin{subfigure}{.21\textwidth}
  \centering
  \includegraphics[width=\textwidth]{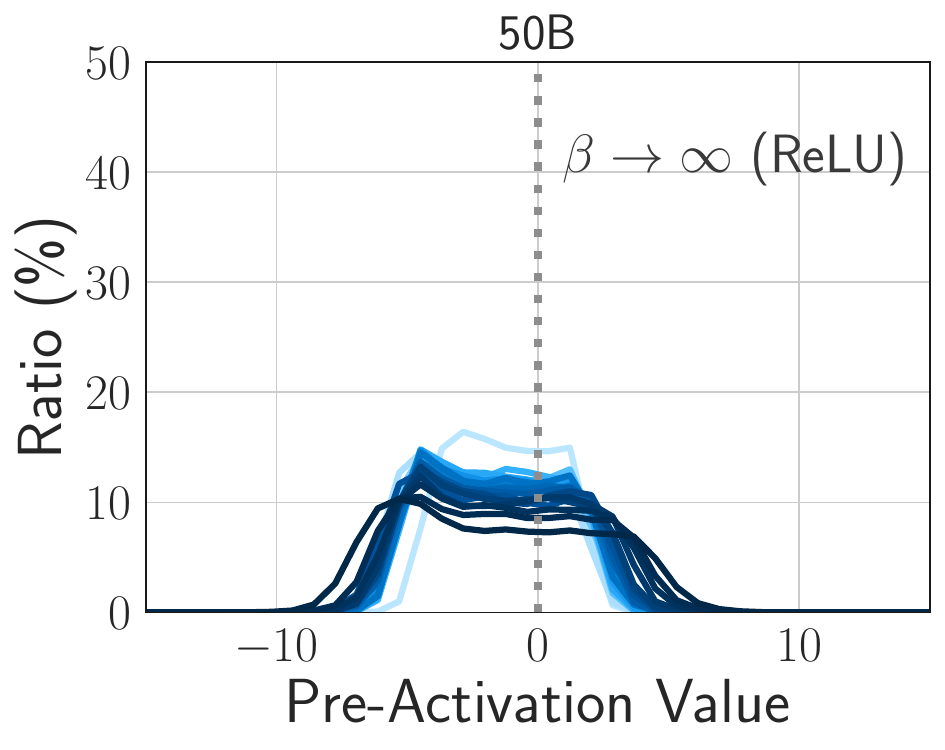}
\end{subfigure}\hfill
\begin{subfigure}{.241\textwidth}
  \centering
  \includegraphics[width=\textwidth]{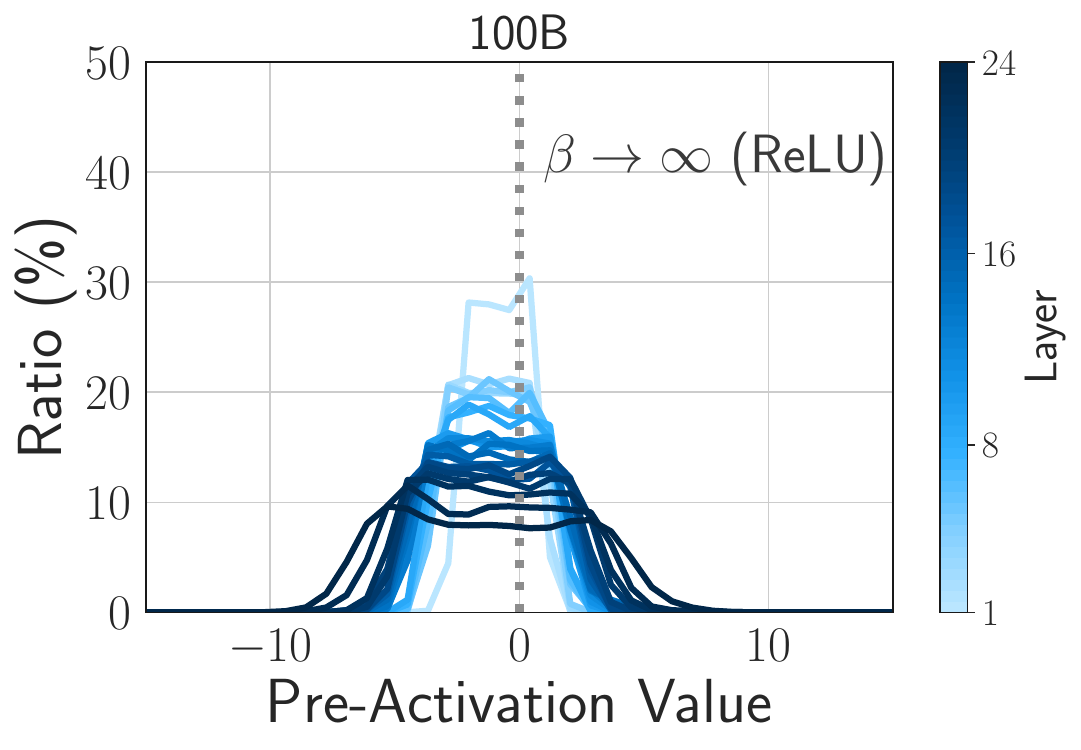}
\end{subfigure}\hfill 
\begin{subfigure}{.21\textwidth}
  \centering
  \includegraphics[width=\textwidth]{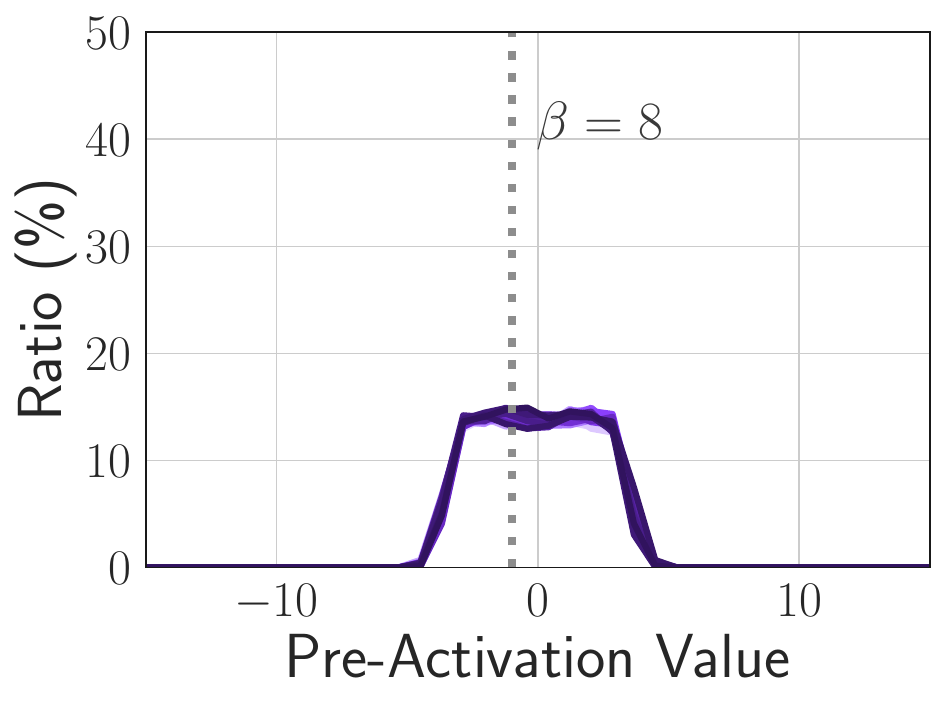}
\end{subfigure}\hfill
\begin{subfigure}{.21\textwidth}
  \centering
  \includegraphics[width=\textwidth]{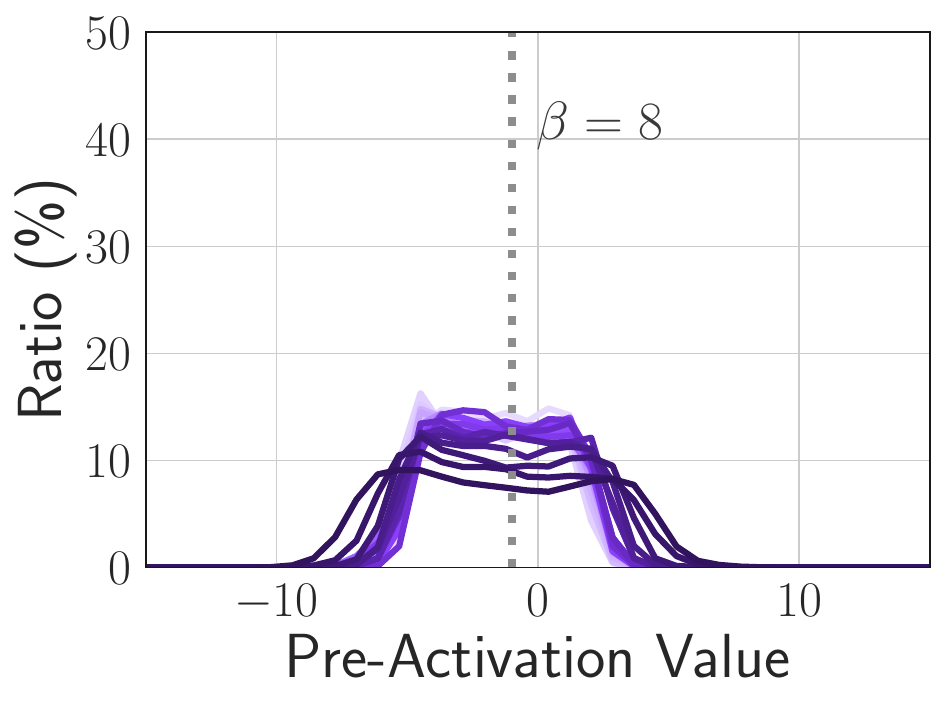}
\end{subfigure}\hfill 
\begin{subfigure}{.21\textwidth}
  \centering
  \includegraphics[width=\textwidth]{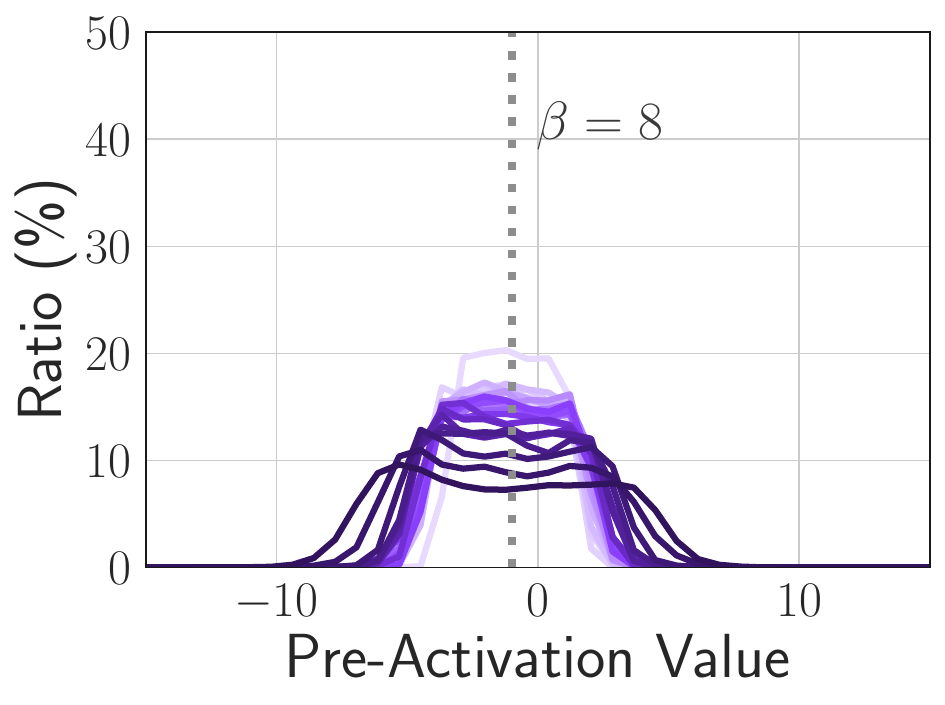}
\end{subfigure}\hfill
\begin{subfigure}{.241\textwidth}
  \centering
  \includegraphics[width=\textwidth]{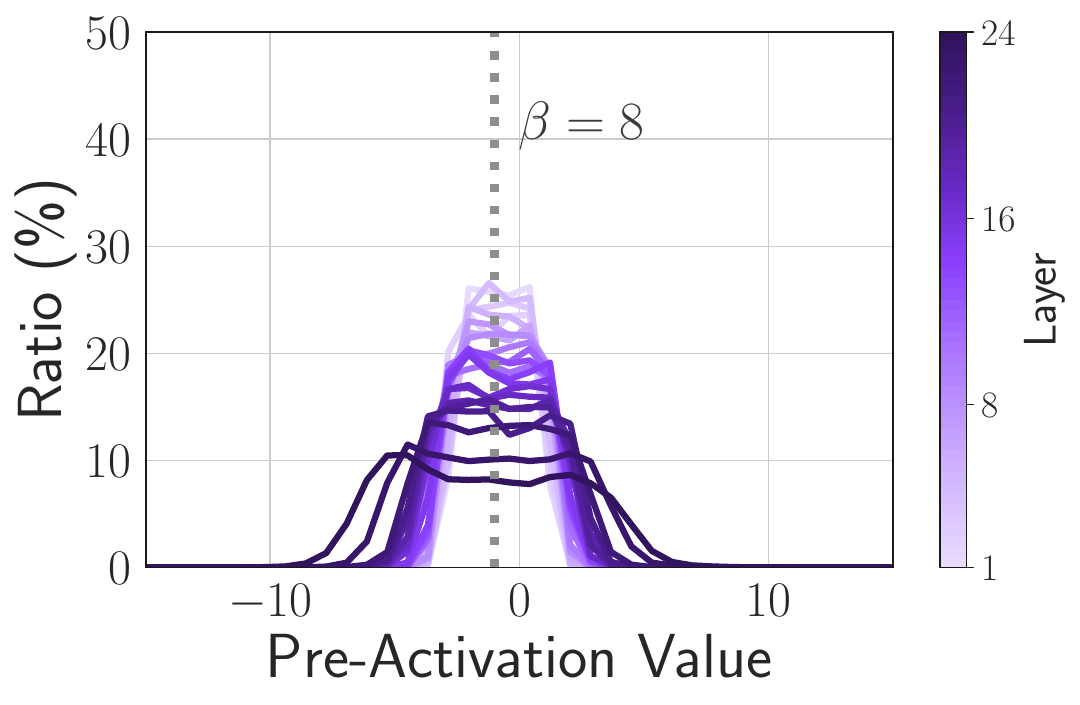}
\end{subfigure}\hfill 
\begin{subfigure}{.21\textwidth}
  \centering
  \includegraphics[width=\textwidth]{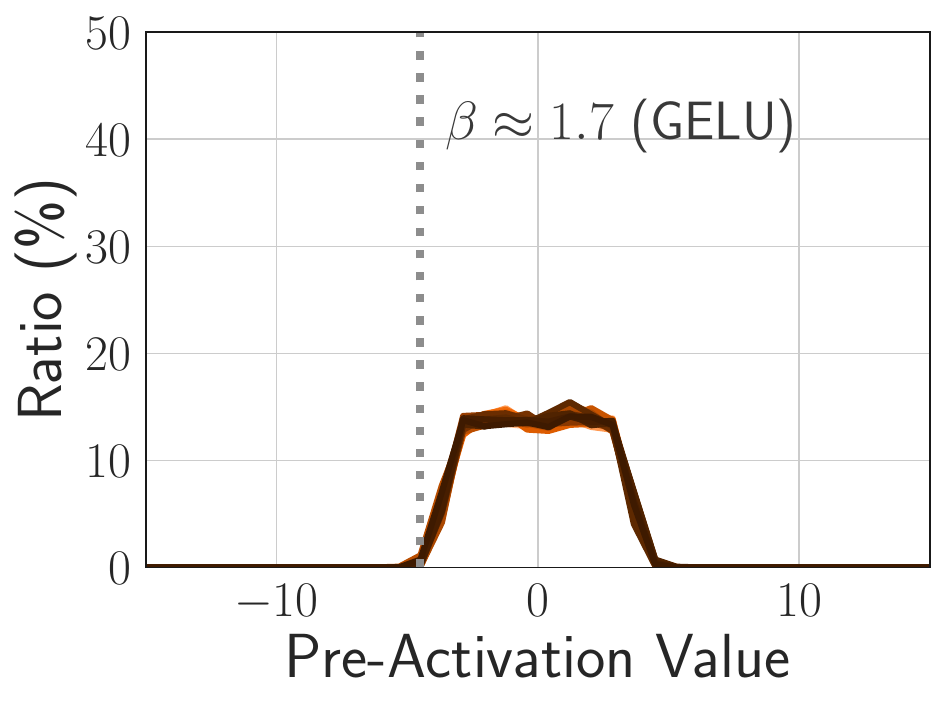}
\end{subfigure}\hfill
\begin{subfigure}{.21\textwidth}
  \centering
  \includegraphics[width=\textwidth]{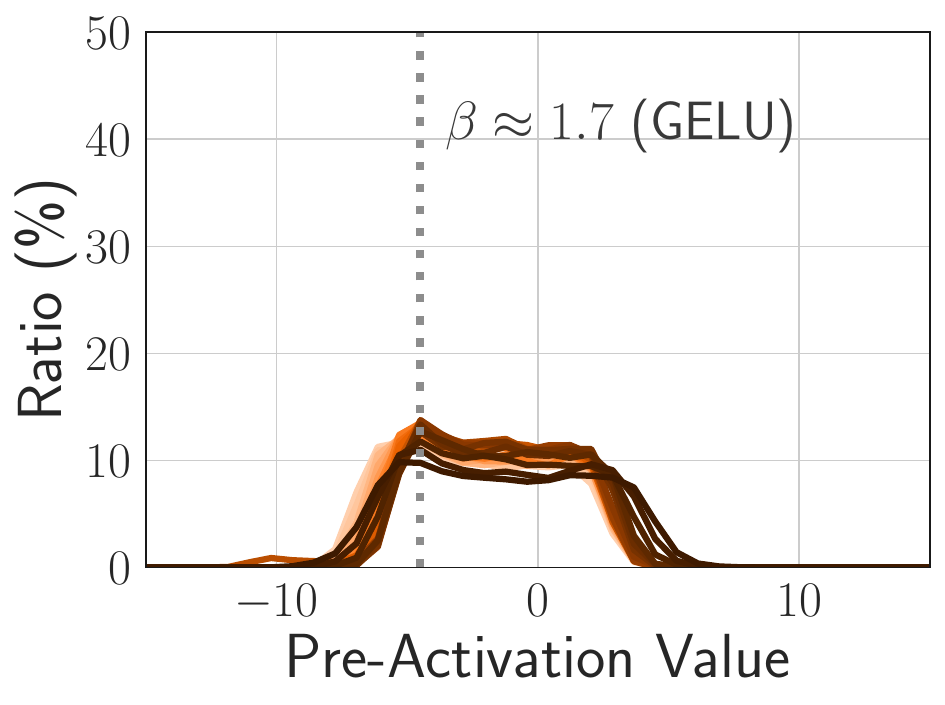}
\end{subfigure}\hfill 
\begin{subfigure}{.21\textwidth}
  \centering
  \includegraphics[width=\textwidth]{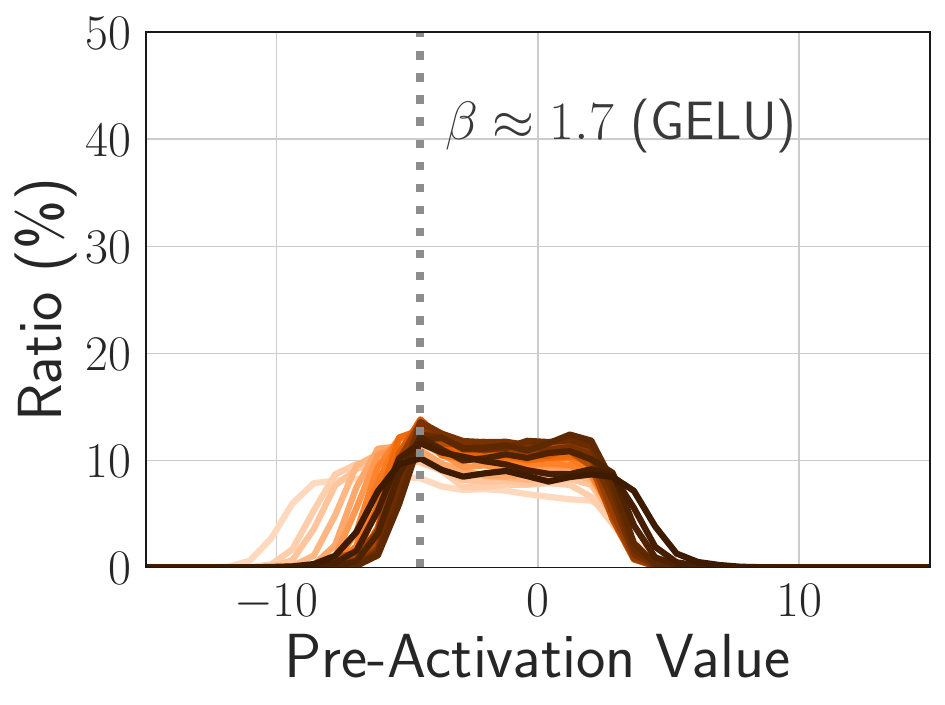}
\end{subfigure}\hfill
\begin{subfigure}{.241\textwidth}
  \centering
  \includegraphics[width=\textwidth]{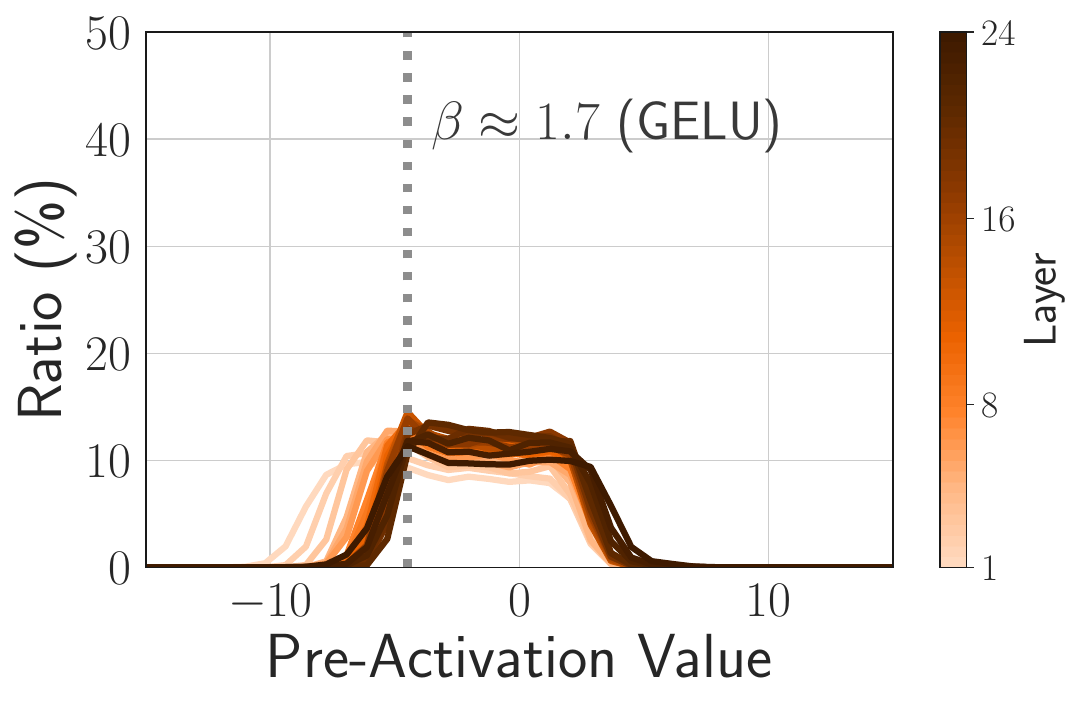}
\end{subfigure}\hfill 
\begin{subfigure}{.21\textwidth}
  \centering
  \includegraphics[width=\textwidth]{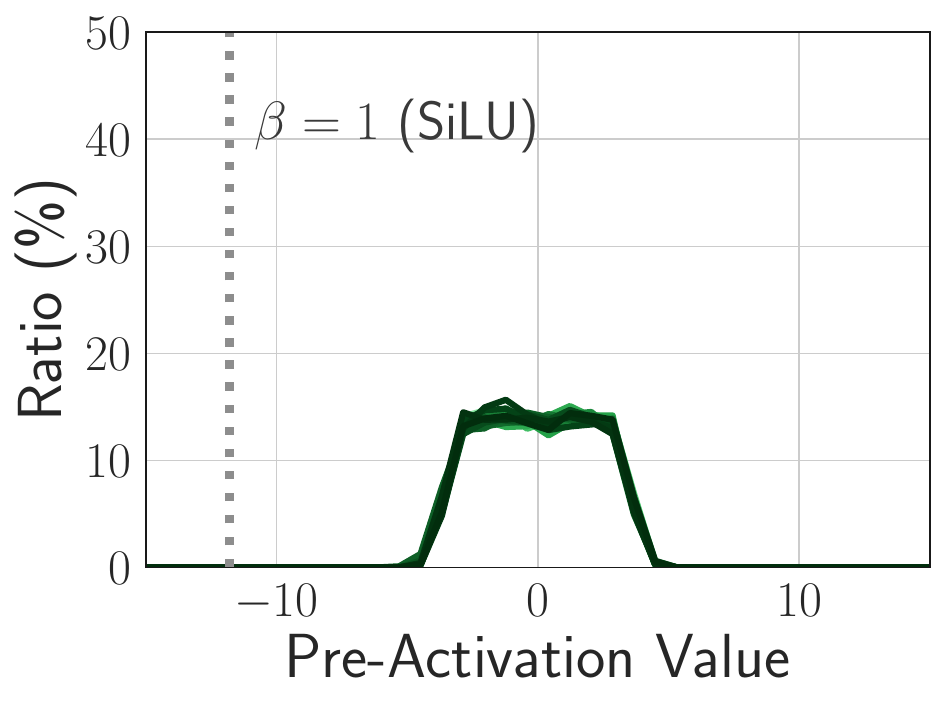}
\end{subfigure}\hfill
\begin{subfigure}{.21\textwidth}
  \centering
  \includegraphics[width=\textwidth]{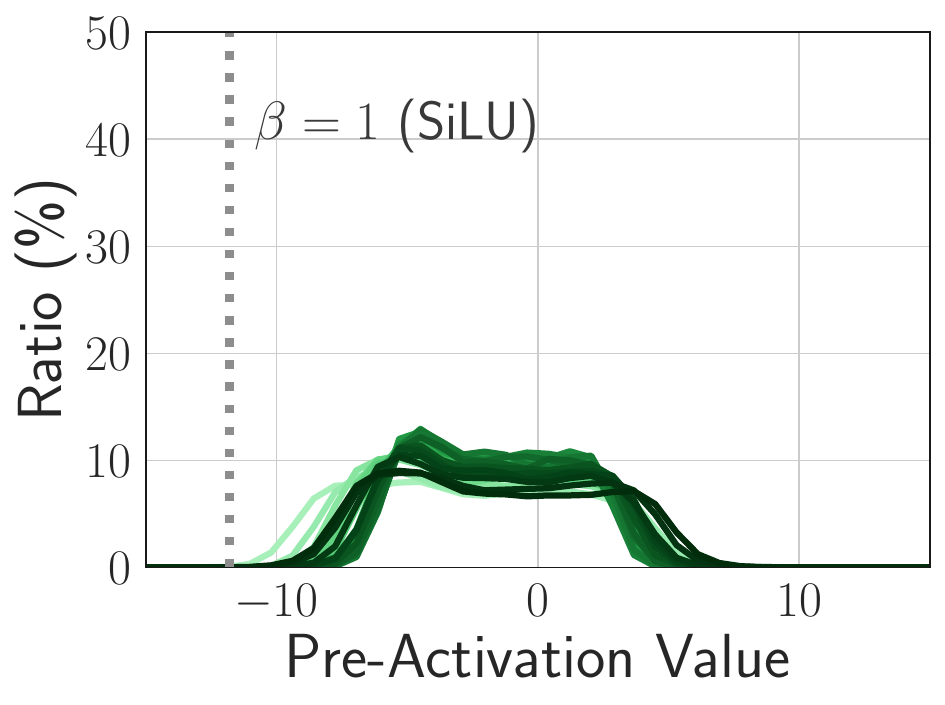}
\end{subfigure}\hfill 
\begin{subfigure}{.21\textwidth}
  \centering
  \includegraphics[width=\textwidth]{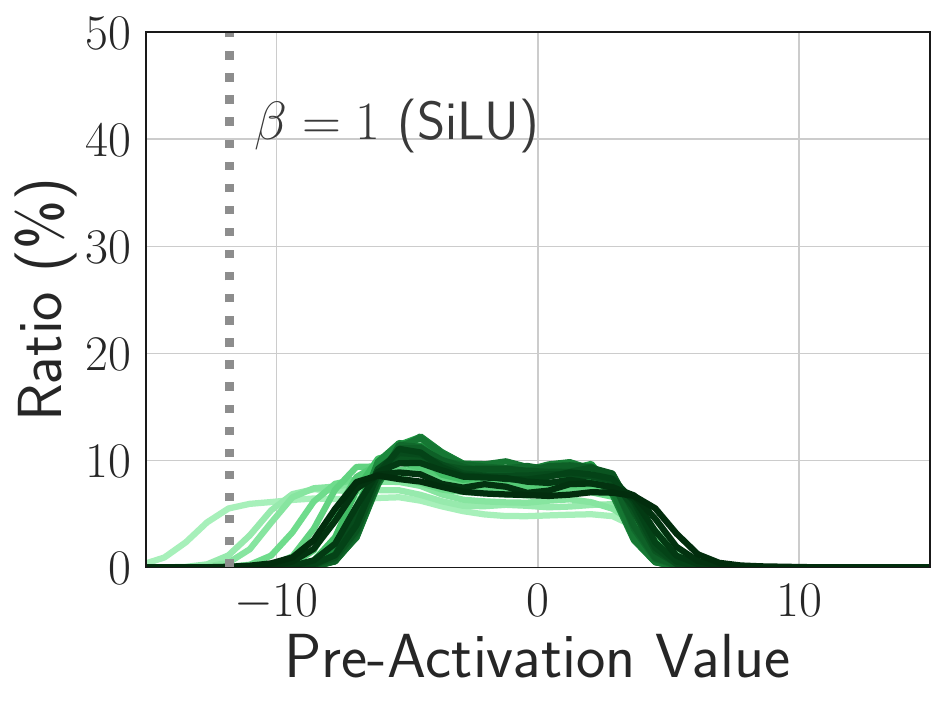}
\end{subfigure}\hfill
\begin{subfigure}{.241\textwidth}
  \centering
  \includegraphics[width=\textwidth]{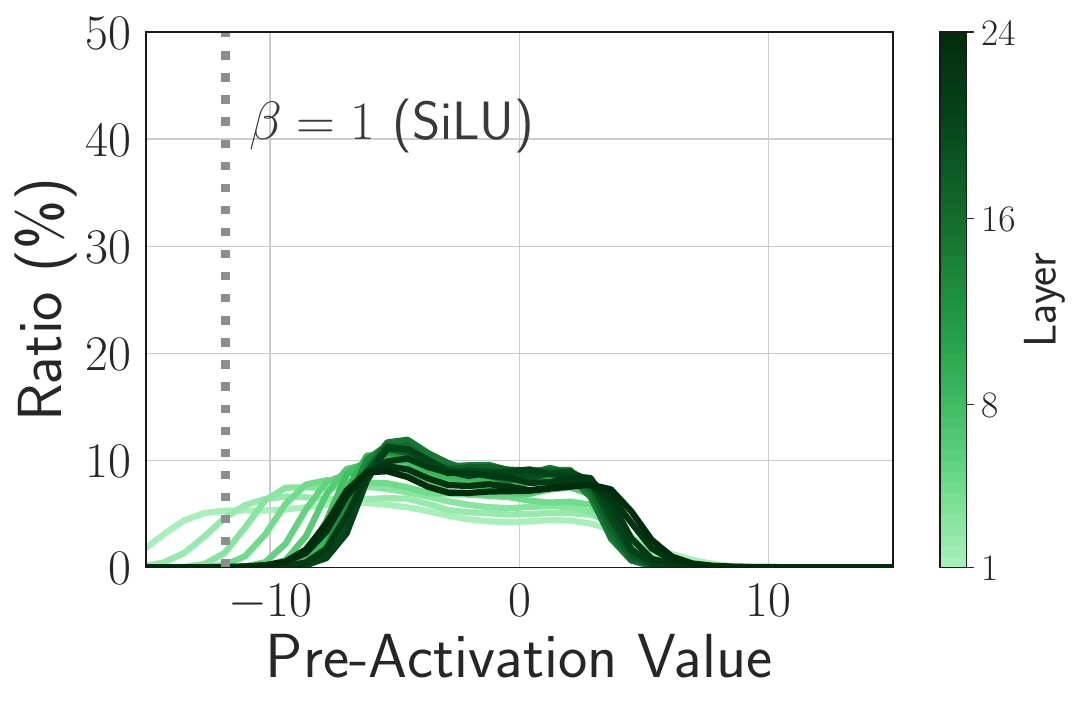}
\end{subfigure}\hfill 
\caption{Pre-activation distributions of various OPT 1.3B models with all four types of activations trained from scratch at various number of seen tokens during training.}
\label{fig:appendix-preact-dist-opt}
\end{figure}

\section{Is a sparsified large model better than dense smaller ones?}
\label{appendix:scaling-opt}
When it comes to deploying efficient models, one may naturally use an original smaller-size (dense) model. The argument would be the performance of the relufied larger model might be already equal to or less than the smaller dense model. To study the above question, we plotted the performance vs. efficiency of the original and the relufied OPT models in~\figref{fig:relufication-results-scaling}. Taking the relufied OPT 6.7B model as an example, it operates at 2.8 GFLPOPs per token. Interpolating the blue line (that can be seen as a scaling plot of the OPT model), a dense model with equivalent FLOPS falls more than 2\%  short in zero-shot performance.

Similarly, compared to the relufied OPT 2.7B model, the equivalent (in FLOPS) dense model performs almost 2\% lower. Indeed, the fact that the relufied models lie well above the scaling graph of the original OPT models, shows the effectiveness of relufication processes as a method to get better but more efficient models. 
As a side benefit, it makes the efficiency spectrum of the available LLMs more continuous. For example, consider a combination of hardware and use case that only allows deploying LLMs with lower than 3 GFLOPS during inference. Going with standard pretrained models, the only available option is OPT 2.7B with almost 1 GFLOPS, as the 6.7B does not satisfy the hardware constraint. In this situation, our relufied model not only falls in the limited inference budget but is also very close to the next largest available model in terms of accraucy. An exciting and timely direction for future research is finding methods, that, given an LLM (or a family of LLMs), are able to produce the best performing model matching the specified inference computation budget.

\begin{figure}[h]
\centering
\includegraphics[width= 0.45\textwidth]{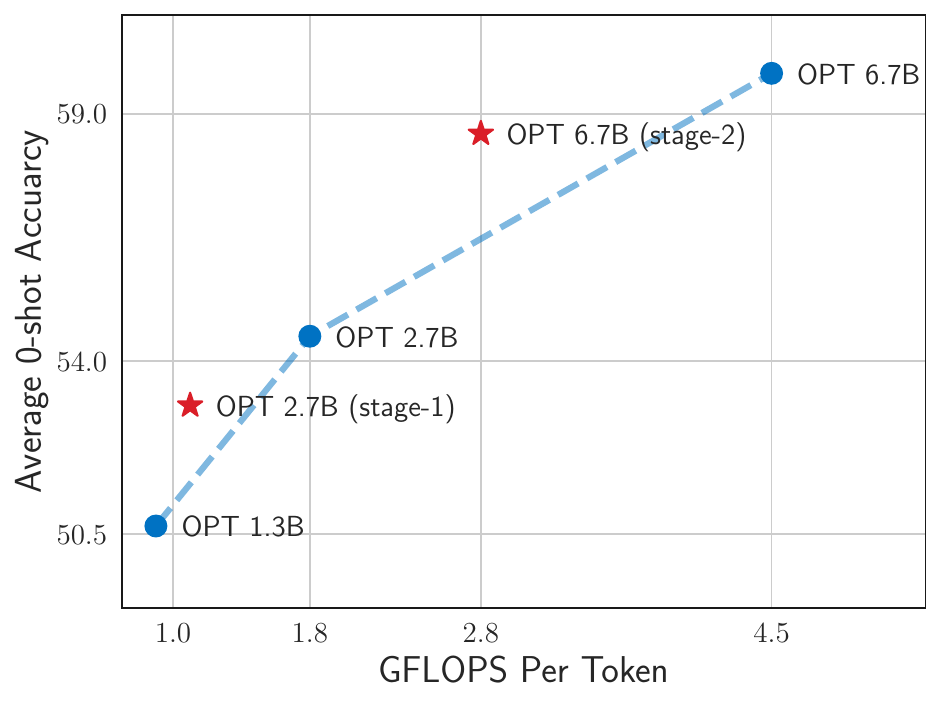}
\caption{Performance of sparse large models vs. dense smaller models: The relufied large models (red stars) are above the scaling curve of original dense models (blue circles and dashed line).}
\label{fig:relufication-results-scaling}
\end{figure}

\end{document}

%% file: math_commands.tex
\usepackage{amsmath,amsfonts,bm}

\def\1{\bm{1}}

\DeclareMathAlphabet{\mathsfit}{\encodingdefault}{\sfdefault}{m}{sl}
\SetMathAlphabet{\mathsfit}{bold}{\encodingdefault}{\sfdefault}{bx}{n}

\newcommand{\sigmoid}{\sigma}

%% file: arxiv.bbl
\begin{thebibliography}{83}
\providecommand{\natexlab}[1]{#1}
\providecommand{\url}[1]{\texttt{#1}}
\expandafter\ifx\csname urlstyle\endcsname\relax
  \providecommand{\doi}[1]{doi: #1}\else
  \providecommand{\doi}{doi: \begingroup \urlstyle{rm}\Url}\fi

\bibitem[Agarwal et~al.(2023)Agarwal, Vieillard, Stanczyk, Ramos, Geist, and
  Bachem]{agarwal2023gkd}
Rishabh Agarwal, Nino Vieillard, Piotr Stanczyk, Sabela Ramos, Matthieu Geist,
  and Olivier Bachem.
\newblock Gkd: Generalized knowledge distillation for auto-regressive sequence
  models.
\newblock \emph{CoRR}, 2023.

\bibitem[Almazrouei et~al.(2023)Almazrouei, Alobeidli, Alshamsi, Cappelli,
  Cojocaru, Alhammadi, Daniele, Heslow, Launay, Malartic, Noune, Pannier, and
  Penedo]{FalconPaper}
Ebtesam Almazrouei, Hamza Alobeidli, Abdulaziz Alshamsi, Alessandro Cappelli,
  Ruxandra Cojocaru, Maitha Alhammadi, Mazzotta Daniele, Daniel Heslow, Julien
  Launay, Quentin Malartic, Badreddine Noune, Baptiste Pannier, and Guilherme
  Penedo.
\newblock The falcon series of language models: Towards open frontier models.
\newblock 2023.

\bibitem[Aminabadi et~al.(2022)Aminabadi, Rajbhandari, Awan, Li, Li, Zheng,
  Ruwase, Smith, Zhang, Rasley, et~al.]{aminabadi2022deepspeed}
Reza~Yazdani Aminabadi, Samyam Rajbhandari, Ammar~Ahmad Awan, Cheng Li, Du~Li,
  Elton Zheng, Olatunji Ruwase, Shaden Smith, Minjia Zhang, Jeff Rasley, et~al.
\newblock Deepspeed-inference: enabling efficient inference of transformer
  models at unprecedented scale.
\newblock In \emph{SC22: International Conference for High Performance
  Computing, Networking, Storage and Analysis}, pages 1--15. IEEE, 2022.

\bibitem[Ba et~al.(2016)Ba, Kiros, and Hinton]{ba2016layer}
Jimmy~Lei Ba, Jamie~Ryan Kiros, and Geoffrey~E. Hinton.
\newblock Layer normalization, 2016.

\bibitem[Brown et~al.(2020)Brown, Mann, Ryder, Subbiah, Kaplan, Dhariwal,
  Neelakantan, Shyam, Sastry, Askell, et~al.]{brown2020language}
Tom Brown, Benjamin Mann, Nick Ryder, Melanie Subbiah, Jared~D Kaplan, Prafulla
  Dhariwal, Arvind Neelakantan, Pranav Shyam, Girish Sastry, Amanda Askell,
  et~al.
\newblock Language models are few-shot learners.
\newblock \emph{Advances in neural information processing systems},
  33:\penalty0 1877--1901, 2020.

\bibitem[Bubeck et~al.(2023)Bubeck, Chandrasekaran, Eldan, Gehrke, Horvitz,
  Kamar, Lee, Lee, Li, Lundberg, et~al.]{bubeck2023sparks}
S{\'e}bastien Bubeck, Varun Chandrasekaran, Ronen Eldan, Johannes Gehrke, Eric
  Horvitz, Ece Kamar, Peter Lee, Yin~Tat Lee, Yuanzhi Li, Scott Lundberg,
  et~al.
\newblock Sparks of artificial general intelligence: Early experiments with
  gpt-4.
\newblock \emph{arXiv preprint arXiv:2303.12712}, 2023.

\bibitem[Chee et~al.(2023)Chee, Cai, Kuleshov, and Sa]{chee2023quip}
Jerry Chee, Yaohui Cai, Volodymyr Kuleshov, and Christopher~De Sa.
\newblock Quip: 2-bit quantization of large language models with guarantees.
\newblock \emph{CoRR}, abs/2307.13304, 2023.
\newblock \doi{10.48550/arXiv.2307.13304}.

\bibitem[Chen et~al.(2022)Chen, Huang, Xie, Jiao, Jiang, Zhou, Li, and
  Wei]{Chen2022TaskSpecificEP}
Tianyu Chen, Shaohan Huang, Yuan Xie, Binxing Jiao, Daxin Jiang, Haoyi Zhou,
  Jianxin Li, and Furu Wei.
\newblock Task-specific expert pruning for sparse mixture-of-experts.
\newblock \emph{ArXiv}, abs/2206.00277, 2022.

\bibitem[Chowdhery et~al.(2022)Chowdhery, Narang, Devlin, Bosma, Mishra,
  Roberts, Barham, Chung, Sutton, Gehrmann, et~al.]{chowdhery2022palm}
Aakanksha Chowdhery, Sharan Narang, Jacob Devlin, Maarten Bosma, Gaurav Mishra,
  Adam Roberts, Paul Barham, Hyung~Won Chung, Charles Sutton, Sebastian
  Gehrmann, et~al.
\newblock Palm: Scaling language modeling with pathways.
\newblock \emph{arXiv preprint arXiv:2204.02311}, 2022.

\bibitem[Clevert et~al.(2016)Clevert, Unterthiner, and
  Hochreiter]{clevert2016fast}
Djork{-}Arn{\'{e}} Clevert, Thomas Unterthiner, and Sepp Hochreiter.
\newblock Fast and accurate deep network learning by exponential linear units
  (elus).
\newblock In Yoshua Bengio and Yann LeCun, editors, \emph{4th International
  Conference on Learning Representations, {ICLR} 2016, San Juan, Puerto Rico,
  May 2-4, 2016, Conference Track Proceedings}, 2016.

\bibitem[Dauphin et~al.(2017)Dauphin, Fan, Auli, and
  Grangier]{10.5555/3305381.3305478}
Yann~N. Dauphin, Angela Fan, Michael Auli, and David Grangier.
\newblock Language modeling with gated convolutional networks.
\newblock In \emph{Proceedings of the 34th International Conference on Machine
  Learning - Volume 70}, ICML'17, page 933–941. JMLR.org, 2017.

\bibitem[Dettmers et~al.(2022)Dettmers, Lewis, Belkada, and
  Zettlemoyer]{dettmers2022llm}
Tim Dettmers, Mike Lewis, Younes Belkada, and Luke Zettlemoyer.
\newblock Llm. int8 (): 8-bit matrix multiplication for transformers at scale.
\newblock \emph{arXiv preprint arXiv:2208.07339}, 2022.

\bibitem[Dettmers et~al.(2023{\natexlab{a}})Dettmers, Pagnoni, Holtzman, and
  Zettlemoyer]{dettmers2023qlora}
Tim Dettmers, Artidoro Pagnoni, Ari Holtzman, and Luke Zettlemoyer.
\newblock Qlora: Efficient finetuning of quantized llms.
\newblock \emph{CoRR}, 2023{\natexlab{a}}.

\bibitem[Dettmers et~al.(2023{\natexlab{b}})Dettmers, Svirschevski, Egiazarian,
  Kuznedelev, Frantar, Ashkboos, Borzunov, Hoefler, and
  Alistarh]{dettmers2023spqr}
Tim Dettmers, Ruslan Svirschevski, Vage Egiazarian, Denis Kuznedelev, Elias
  Frantar, Saleh Ashkboos, Alexander Borzunov, Torsten Hoefler, and Dan
  Alistarh.
\newblock Spqr: {A} sparse-quantized representation for near-lossless {LLM}
  weight compression.
\newblock \emph{CoRR}, abs/2306.03078, 2023{\natexlab{b}}.
\newblock \doi{10.48550/arXiv.2306.03078}.

\bibitem[Dong and Chen(2023)]{dong2023blockwise}
Gaochen Dong and Wei Chen.
\newblock Blockwise compression of transformer-based models without retraining.
\newblock \emph{arXiv preprint arXiv:2304.01483}, 2023.

\bibitem[Du et~al.(2022)Du, Huang, Dai, Tong, Lepikhin, Xu, Krikun, Zhou, Yu,
  Firat, Zoph, Fedus, Bosma, Zhou, Wang, Wang, Webster, Pellat, Robinson,
  Meier{-}Hellstern, Duke, Dixon, Zhang, Le, Wu, Chen, and Cui]{du2022glam}
Nan Du, Yanping Huang, Andrew~M. Dai, Simon Tong, Dmitry Lepikhin, Yuanzhong
  Xu, Maxim Krikun, Yanqi Zhou, Adams~Wei Yu, Orhan Firat, Barret Zoph, Liam
  Fedus, Maarten~P. Bosma, Zongwei Zhou, Tao Wang, Yu~Emma Wang, Kellie
  Webster, Marie Pellat, Kevin Robinson, Kathleen~S. Meier{-}Hellstern, Toju
  Duke, Lucas Dixon, Kun Zhang, Quoc~V. Le, Yonghui Wu, Zhifeng Chen, and
  Claire Cui.
\newblock Glam: Efficient scaling of language models with mixture-of-experts.
\newblock In Kamalika Chaudhuri, Stefanie Jegelka, Le~Song, Csaba
  Szepesv{\'{a}}ri, Gang Niu, and Sivan Sabato, editors, \emph{International
  Conference on Machine Learning, {ICML} 2022, 17-23 July 2022, Baltimore,
  Maryland, {USA}}, volume 162 of \emph{Proceedings of Machine Learning
  Research}, pages 5547--5569. {PMLR}, 2022.

\bibitem[Elfwing et~al.(2018)Elfwing, Uchibe, and Doya]{elfwing2018sigmoid}
Stefan Elfwing, Eiji Uchibe, and Kenji Doya.
\newblock Sigmoid-weighted linear units for neural network function
  approximation in reinforcement learning.
\newblock \emph{Neural networks}, 107:\penalty0 3--11, 2018.

\bibitem[Fedus et~al.(2022{\natexlab{a}})Fedus, Dean, and Zoph]{Fedus2022ARO}
William Fedus, Jeff Dean, and Barret Zoph.
\newblock A review of sparse expert models in deep learning.
\newblock \emph{ArXiv}, abs/2209.01667, 2022{\natexlab{a}}.

\bibitem[Fedus et~al.(2022{\natexlab{b}})Fedus, Zoph, and
  Shazeer]{fedus2022switch}
William Fedus, Barret Zoph, and Noam Shazeer.
\newblock Switch transformers: Scaling to trillion parameter models with simple
  and efficient sparsity.
\newblock \emph{J. Mach. Learn. Res.}, 23:\penalty0 120:1--120:39,
  2022{\natexlab{b}}.

\bibitem[Frantar and Alistarh(2023)]{frantar2023sparsegpt}
Elias Frantar and Dan Alistarh.
\newblock Sparsegpt: Massive language models can be accurately pruned in
  one-shot.
\newblock In \emph{International Conference on Machine Learning, {ICML} 2023,
  23-29 July 2023, Honolulu, Hawaii, {USA}}, volume 202 of \emph{Proceedings of
  Machine Learning Research}, pages 10323--10337. {PMLR}, 2023.
\newblock URL \url{https://proceedings.mlr.press/v202/frantar23a.html}.

\bibitem[Frantar et~al.(2022)Frantar, Ashkboos, Hoefler, and
  Alistarh]{frantar2023gptq}
Elias Frantar, Saleh Ashkboos, Torsten Hoefler, and Dan Alistarh.
\newblock {GPTQ:} accurate post-training quantization for generative
  pre-trained transformers.
\newblock \emph{CoRR}, abs/2210.17323, 2022.
\newblock \doi{10.48550/arXiv.2210.17323}.

\bibitem[Fukushima(1969)]{Fukushima1969VisualFE}
Kunihiko Fukushima.
\newblock Visual feature extraction by a multilayered network of analog
  threshold elements.
\newblock \emph{IEEE Trans. Syst. Sci. Cybern.}, 5:\penalty0 322--333, 1969.

\bibitem[Gao et~al.(2021)Gao, Tow, Biderman, Black, DiPofi, Foster, Golding,
  Hsu, McDonell, Muennighoff, Phang, Reynolds, Tang, Thite, Wang, Wang, and
  Zou]{eval-harness}
Leo Gao, Jonathan Tow, Stella Biderman, Sid Black, Anthony DiPofi, Charles
  Foster, Laurence Golding, Jeffrey Hsu, Kyle McDonell, Niklas Muennighoff,
  Jason Phang, Laria Reynolds, Eric Tang, Anish Thite, Ben Wang, Kevin Wang,
  and Andy Zou.
\newblock A framework for few-shot language model evaluation, September 2021.
\newblock URL \url{https://doi.org/10.5281/zenodo.5371628}.

\bibitem[Gu et~al.(2023)Gu, Dong, Wei, and Huang]{gu2023knowledge}
Yuxian Gu, Li~Dong, Furu Wei, and Minlie Huang.
\newblock Knowledge distillation of large language models.
\newblock \emph{CoRR}, 2023.

\bibitem[Han et~al.(2023)Han, Liu, Mao, Pu, Pedram, Horowitz, and
  Dally]{han2023retrospective}
Song Han, Xingyu Liu, Huizi Mao, Jing Pu, Ardavan Pedram, Mark~A. Horowitz, and
  William~J. Dally.
\newblock Retrospective: Eie: Efficient inference engine on sparse and
  compressed neural network, 2023.

\bibitem[Hazimeh et~al.(2021)Hazimeh, Zhao, Chowdhery, Sathiamoorthy, Chen,
  Mazumder, Hong, and Chi]{Hazimeh2021DSelectkDS}
Hussein Hazimeh, Zhe Zhao, Aakanksha Chowdhery, Maheswaran Sathiamoorthy, Yihua
  Chen, Rahul Mazumder, Lichan Hong, and Ed~H. Chi.
\newblock Dselect-k: Differentiable selection in the mixture of experts with
  applications to multi-task learning.
\newblock In \emph{Neural Information Processing Systems}, 2021.

\bibitem[He et~al.(2016)He, Zhang, Ren, and Sun]{he2016deep}
Kaiming He, Xiangyu Zhang, Shaoqing Ren, and Jian Sun.
\newblock Deep residual learning for image recognition.
\newblock In \emph{Proceedings of the IEEE conference on computer vision and
  pattern recognition}, pages 770--778, 2016.

\bibitem[Hendrycks and Gimpel(2016)]{hendrycks2016gaussian}
Dan Hendrycks and Kevin Gimpel.
\newblock Gaussian error linear units (gelus).
\newblock \emph{arXiv preprint arXiv:1606.08415}, 2016.

\bibitem[Hendrycks et~al.(2021)Hendrycks, Burns, Basart, Zou, Mazeika, Song,
  and Steinhardt]{hendrycks2021measuring}
Dan Hendrycks, Collin Burns, Steven Basart, Andy Zou, Mantas Mazeika, Dawn
  Song, and Jacob Steinhardt.
\newblock Measuring massive multitask language understanding.
\newblock In \emph{9th International Conference on Learning Representations,
  {ICLR} 2021, Virtual Event, Austria, May 3-7, 2021}. OpenReview.net, 2021.
\newblock URL \url{https://openreview.net/forum?id=d7KBjmI3GmQ}.

\bibitem[Hinton et~al.(2015)Hinton, Vinyals, and Dean]{hinton2015distilling}
Geoffrey Hinton, Oriol Vinyals, and Jeff Dean.
\newblock Distilling the knowledge in a neural network.
\newblock \emph{CoRR}, 2015.

\bibitem[Hoffmann et~al.(2022)Hoffmann, Borgeaud, Mensch, Buchatskaya, Cai,
  Rutherford, de~Las~Casas, Hendricks, Welbl, Clark, Hennigan, Noland,
  Millican, van~den Driessche, Damoc, Guy, Osindero, Simonyan, Elsen, Rae,
  Vinyals, and Sifre]{ScalingLawChinchilla}
Jordan Hoffmann, Sebastian Borgeaud, Arthur Mensch, Elena Buchatskaya, Trevor
  Cai, Eliza Rutherford, Diego de~Las~Casas, Lisa~Anne Hendricks, Johannes
  Welbl, Aidan Clark, Tom Hennigan, Eric Noland, Katie Millican, George van~den
  Driessche, Bogdan Damoc, Aurelia Guy, Simon Osindero, Karen Simonyan, Erich
  Elsen, Jack~W. Rae, Oriol Vinyals, and Laurent Sifre.
\newblock Training compute-optimal large language models.
\newblock \emph{CoRR}, abs/2203.15556, 2022.
\newblock \doi{10.48550/arXiv.2203.15556}.

\bibitem[Hron et~al.(2020)Hron, Bahri, Sohl{-}Dickstein, and
  Novak]{hron2020infinite}
Jiri Hron, Yasaman Bahri, Jascha Sohl{-}Dickstein, and Roman Novak.
\newblock Infinite attention: {NNGP} and {NTK} for deep attention networks.
\newblock In \emph{Proceedings of the 37th International Conference on Machine
  Learning, {ICML} 2020, 13-18 July 2020, Virtual Event}, volume 119 of
  \emph{Proceedings of Machine Learning Research}, pages 4376--4386. {PMLR},
  2020.

\bibitem[Hsieh et~al.(2023)Hsieh, Li, Yeh, Nakhost, Fujii, Ratner, Krishna,
  Lee, and Pfister]{hsieh2023distilling}
Cheng{-}Yu Hsieh, Chun{-}Liang Li, Chih{-}Kuan Yeh, Hootan Nakhost, Yasuhisa
  Fujii, Alex Ratner, Ranjay Krishna, Chen{-}Yu Lee, and Tomas Pfister.
\newblock Distilling step-by-step! outperforming larger language models with
  less training data and smaller model sizes.
\newblock In \emph{Findings of the Association for Computational Linguistics:
  {ACL} 2023, Toronto, Canada, July 9-14, 2023}, pages 8003--8017. Association
  for Computational Linguistics, 2023.
\newblock URL \url{https://doi.org/10.18653/v1/2023.findings-acl.507}.

\bibitem[Hwang et~al.(2023)Hwang, Wei, Cao, Hwang, Tang, Cao, Yang, and
  Rhu]{hwang2023pre}
Ranggi Hwang, Jianyu Wei, Shijie Cao, Changho Hwang, Xiaohu Tang, Ting Cao, Mao
  Yang, and Minsoo Rhu.
\newblock Pre-gated moe: An algorithm-system co-design for fast and scalable
  mixture-of-expert inference.
\newblock \emph{arXiv preprint arXiv:2308.12066}, 2023.

\bibitem[Jaiswal et~al.(2023)Jaiswal, Liu, Chen, and
  Wang]{jaiswal2023emergence}
Ajay Jaiswal, Shiwei Liu, Tianlong Chen, and Zhangyang Wang.
\newblock The emergence of essential sparsity in large pre-trained models: The
  weights that matter.
\newblock \emph{CoRR}, 2023.

\bibitem[Jaszczur et~al.(2021)Jaszczur, Chowdhery, Mohiuddin, Kaiser, Gajewski,
  Michalewski, and Kanerva]{jaszczur2021sparse}
Sebastian Jaszczur, Aakanksha Chowdhery, Afroz Mohiuddin, Lukasz Kaiser,
  Wojciech Gajewski, Henryk Michalewski, and Jonni Kanerva.
\newblock Sparse is enough in scaling transformers.
\newblock In \emph{Advances in Neural Information Processing Systems}, 2021.
\newblock URL \url{https://openreview.net/forum?id=-b5OSCydOMe}.

\bibitem[Kaplan et~al.(2020)Kaplan, McCandlish, Henighan, Brown, Chess, Child,
  Gray, Radford, Wu, and Amodei]{ScalingLawOpenAI}
Jared Kaplan, Sam McCandlish, Tom Henighan, Tom~B. Brown, Benjamin Chess, Rewon
  Child, Scott Gray, Alec Radford, Jeffrey Wu, and Dario Amodei.
\newblock Scaling laws for neural language models.
\newblock \emph{CoRR}, abs/2001.08361, 2020.
\newblock URL \url{https://arxiv.org/abs/2001.08361}.

\bibitem[Kim et~al.(2023{\natexlab{a}})Kim, Lee, Kim, Park, Yoo, Kwon, and
  Lee]{kim2023memoryefficient}
Jeonghoon Kim, Jung~Hyun Lee, Sungdong Kim, Joonsuk Park, Kang~Min Yoo, Se~Jung
  Kwon, and Dongsoo Lee.
\newblock Memory-efficient fine-tuning of compressed large language models via
  sub-4-bit integer quantization.
\newblock \emph{CoRR}, abs/2305.14152, 2023{\natexlab{a}}.
\newblock \doi{10.48550/arXiv.2305.14152}.
\newblock URL \url{https://doi.org/10.48550/arXiv.2305.14152}.

\bibitem[Kim et~al.(2023{\natexlab{b}})Kim, Hooper, Gholami, Dong, Li, Shen,
  Mahoney, and Keutzer]{kim2023squeezellm}
Sehoon Kim, Coleman Hooper, Amir Gholami, Zhen Dong, Xiuyu Li, Sheng Shen,
  Michael~W. Mahoney, and Kurt Keutzer.
\newblock Squeezellm: Dense-and-sparse quantization.
\newblock \emph{CoRR}, abs/2306.07629, 2023{\natexlab{b}}.
\newblock \doi{10.48550/arXiv.2306.07629}.

\bibitem[Kim et~al.(2023{\natexlab{c}})Kim, Hooper, Wattanawong, Kang, Yan,
  Genc, Dinh, Huang, Keutzer, Mahoney, et~al.]{kim2023full}
Sehoon Kim, Coleman Hooper, Thanakul Wattanawong, Minwoo Kang, Ruohan Yan,
  Hasan Genc, Grace Dinh, Qijing Huang, Kurt Keutzer, Michael~W Mahoney, et~al.
\newblock Full stack optimization of transformer inference.
\newblock In \emph{Architecture and System Support for Transformer Models
  (ASSYST@ ISCA 2023)}, 2023{\natexlab{c}}.

\bibitem[Kim et~al.(2023{\natexlab{d}})Kim, Mangalam, Moon, Canny, Malik,
  Mahoney, Gholami, and Keutzer]{kim2023speculative}
Sehoon Kim, Karttikeya Mangalam, Suhong Moon, John Canny, Jitendra Malik,
  Michael~W. Mahoney, Amir Gholami, and Kurt Keutzer.
\newblock Speculative decoding with big little decoder, 2023{\natexlab{d}}.

\bibitem[Klambauer et~al.(2017)Klambauer, Unterthiner, Mayr, and
  Hochreiter]{klambauer2017selfnormalizing}
G{\"{u}}nter Klambauer, Thomas Unterthiner, Andreas Mayr, and Sepp Hochreiter.
\newblock Self-normalizing neural networks.
\newblock In Isabelle Guyon, Ulrike von Luxburg, Samy Bengio, Hanna~M. Wallach,
  Rob Fergus, S.~V.~N. Vishwanathan, and Roman Garnett, editors, \emph{Advances
  in Neural Information Processing Systems 30: Annual Conference on Neural
  Information Processing Systems 2017, December 4-9, 2017, Long Beach, CA,
  {USA}}, pages 971--980, 2017.

\bibitem[Kong et~al.(2023)Kong, Li, Feng, Wang, Kong, and Liu]{kong2023serving}
Rui Kong, Yuanchun Li, Qingtian Feng, Weijun Wang, Linghe Kong, and Yunxin Liu.
\newblock Serving moe models on resource-constrained edge devices via dynamic
  expert swapping, 2023.

\bibitem[Kurtz et~al.(2020)Kurtz, Kopinsky, Gelashvili, Matveev, Carr, Goin,
  Leiserson, Moore, Shavit, and Alistarh]{kurtz2020inducing}
Mark Kurtz, Justin Kopinsky, Rati Gelashvili, Alexander Matveev, John Carr,
  Michael Goin, William Leiserson, Sage Moore, Nir Shavit, and Dan Alistarh.
\newblock Inducing and exploiting activation sparsity for fast inference on
  deep neural networks.
\newblock In \emph{International Conference on Machine Learning}, pages
  5533--5543. PMLR, 2020.

\bibitem[Lee et~al.(2023)Lee, Jin, Kim, Kim, and Park]{lee2023owq}
Changhun Lee, Jungyu Jin, Taesu Kim, Hyungjun Kim, and Eunhyeok Park.
\newblock {OWQ:} lessons learned from activation outliers for weight
  quantization in large language models.
\newblock \emph{CoRR}, abs/2306.02272, 2023.
\newblock \doi{10.48550/arXiv.2306.02272}.

\bibitem[Leviathan et~al.(2023)Leviathan, Kalman, and
  Matias]{leviathan2023fast}
Yaniv Leviathan, Matan Kalman, and Yossi Matias.
\newblock Fast inference from transformers via speculative decoding.
\newblock In Andreas Krause, Emma Brunskill, Kyunghyun Cho, Barbara Engelhardt,
  Sivan Sabato, and Jonathan Scarlett, editors, \emph{International Conference
  on Machine Learning, {ICML} 2023, 23-29 July 2023, Honolulu, Hawaii, {USA}},
  volume 202 of \emph{Proceedings of Machine Learning Research}, pages
  19274--19286. {PMLR}, 2023.

\bibitem[Li et~al.(2022)Li, You, Bhojanapalli, Li, Rawat, Reddi, Ye, Chern, Yu,
  Guo, et~al.]{li2022large}
Zonglin Li, Chong You, Srinadh Bhojanapalli, Daliang Li, Ankit~Singh Rawat,
  Sashank~J Reddi, Ke~Ye, Felix Chern, Felix Yu, Ruiqi Guo, et~al.
\newblock Large models are parsimonious learners: Activation sparsity in
  trained transformers.
\newblock \emph{arXiv preprint arXiv:2210.06313}, 2022.

\bibitem[Li et~al.(2023)Li, You, Bhojanapalli, Li, Rawat, Reddi, Ye, Chern, Yu,
  Guo, and Kumar]{li2022lazy}
Zonglin Li, Chong You, Srinadh Bhojanapalli, Daliang Li, Ankit~Singh Rawat,
  Sashank~J. Reddi, Ke~Ye, Felix Chern, Felix~X. Yu, Ruiqi Guo, and Sanjiv
  Kumar.
\newblock The lazy neuron phenomenon: On emergence of activation sparsity in
  transformers.
\newblock In \emph{The Eleventh International Conference on Learning
  Representations, {ICLR} 2023, Kigali, Rwanda, May 1-5, 2023}. OpenReview.net,
  2023.

\bibitem[Lin et~al.(2023)Lin, Tang, Tang, Yang, Dang, and Han]{lin2023awq}
Ji~Lin, Jiaming Tang, Haotian Tang, Shang Yang, Xingyu Dang, and Song Han.
\newblock {AWQ:} activation-aware weight quantization for {LLM} compression and
  acceleration.
\newblock \emph{CoRR}, abs/2306.00978, 2023.
\newblock \doi{10.48550/arXiv.2306.00978}.

\bibitem[Liu et~al.(2023{\natexlab{a}})Liu, Oguz, Zhao, Chang, Stock, Mehdad,
  Shi, Krishnamoorthi, and Chandra]{liu2023llmqat}
Zechun Liu, Barlas Oguz, Changsheng Zhao, Ernie Chang, Pierre Stock, Yashar
  Mehdad, Yangyang Shi, Raghuraman Krishnamoorthi, and Vikas Chandra.
\newblock Llm-qat: Data-free quantization aware training for large language
  models.
\newblock \emph{CoRR}, 2023{\natexlab{a}}.

\bibitem[Liu et~al.(2023{\natexlab{b}})Liu, Wang, Dao, Zhou, Yuan, Song,
  Shrivastava, Zhang, Tian, Re, et~al.]{liu2023deja}
Zichang Liu, Jue Wang, Tri Dao, Tianyi Zhou, Binhang Yuan, Zhao Song, Anshumali
  Shrivastava, Ce~Zhang, Yuandong Tian, Christopher Re, et~al.
\newblock Deja vu: Contextual sparsity for efficient llms at inference time.
\newblock In \emph{International Conference on Machine Learning}, pages
  22137--22176. PMLR, 2023{\natexlab{b}}.

\bibitem[Loshchilov and Hutter(2019)]{loshchilov2017decoupled}
Ilya Loshchilov and Frank Hutter.
\newblock Decoupled weight decay regularization.
\newblock In \emph{7th International Conference on Learning Representations,
  {ICLR} 2019, New Orleans, LA, USA, May 6-9, 2019}. OpenReview.net, 2019.
\newblock URL \url{https://openreview.net/forum?id=Bkg6RiCqY7}.

\bibitem[Ma et~al.(2023)Ma, Fang, and Wang]{ma2023llm}
Xinyin Ma, Gongfan Fang, and Xinchao Wang.
\newblock Llm-pruner: On the structural pruning of large language models.
\newblock \emph{arXiv preprint arXiv:2305.11627}, 2023.

\bibitem[Merity et~al.(2017)Merity, Xiong, Bradbury, and Socher]{Wikitext}
Stephen Merity, Caiming Xiong, James Bradbury, and Richard Socher.
\newblock Pointer sentinel mixture models.
\newblock In \emph{5th International Conference on Learning Representations,
  {ICLR} 2017, Toulon, France, April 24-26, 2017, Conference Track
  Proceedings}. OpenReview.net, 2017.

\bibitem[Mirzadeh et~al.(2020)Mirzadeh, Farajtabar, Li, Levine, Matsukawa, and
  Ghasemzadeh]{mirzadeh2020improved}
Seyed~Iman Mirzadeh, Mehrdad Farajtabar, Ang Li, Nir Levine, Akihiro Matsukawa,
  and Hassan Ghasemzadeh.
\newblock Improved knowledge distillation via teacher assistant.
\newblock In \emph{Proceedings of the AAAI conference on artificial
  intelligence}, volume~34, pages 5191--5198, 2020.

\bibitem[MosaicML(2023)]{MosaicMPT}
NLP~Team MosaicML.
\newblock Introducing mpt-7b: A new standard for open-source, commercially
  usable llms, 2023.
\newblock URL \url{www.mosaicml.com/blog/mpt-7b}.

\bibitem[Narang et~al.(2021)Narang, Chung, Tay, Fedus, F{\'{e}}vry, Matena,
  Malkan, Fiedel, Shazeer, Lan, Zhou, Li, Ding, Marcus, Roberts, and
  Raffel]{narang2021transformer}
Sharan Narang, Hyung~Won Chung, Yi~Tay, Liam Fedus, Thibault F{\'{e}}vry,
  Michael Matena, Karishma Malkan, Noah Fiedel, Noam Shazeer, Zhenzhong Lan,
  Yanqi Zhou, Wei Li, Nan Ding, Jake Marcus, Adam Roberts, and Colin Raffel.
\newblock Do transformer modifications transfer across implementations and
  applications?
\newblock In Marie{-}Francine Moens, Xuanjing Huang, Lucia Specia, and
  Scott~Wen{-}tau Yih, editors, \emph{Proceedings of the 2021 Conference on
  Empirical Methods in Natural Language Processing, {EMNLP} 2021, Virtual Event
  / Punta Cana, Dominican Republic, 7-11 November, 2021}, pages 5758--5773.
  Association for Computational Linguistics, 2021.

\bibitem[Park et~al.(2023)Park, Park, Kim, Lee, Kim, Kwon, Kwon, Kim, Lee, and
  Lee]{park2023lutgemm}
Gunho Park, Baeseong Park, Minsub Kim, Sungjae Lee, Jeonghoon Kim, Beomseok
  Kwon, Se~Jung Kwon, Byeongwook Kim, Youngjoo Lee, and Dongsoo Lee.
\newblock Lut-gemm: Quantized matrix multiplication based on luts for efficient
  inference in large-scale generative language models.
\newblock \emph{CoRR}, 2023.

\bibitem[Penedo et~al.(2023)Penedo, Malartic, Hesslow, Cojocaru, Cappelli,
  Alobeidli, Pannier, Almazrouei, and Launay]{RefinedWebDataset}
Guilherme Penedo, Quentin Malartic, Daniel Hesslow, Ruxandra Cojocaru,
  Alessandro Cappelli, Hamza Alobeidli, Baptiste Pannier, Ebtesam Almazrouei,
  and Julien Launay.
\newblock The refinedweb dataset for falcon {LLM:} outperforming curated
  corpora with web data, and web data only.
\newblock \emph{CoRR}, abs/2306.01116, 2023.
\newblock \doi{10.48550/arXiv.2306.01116}.

\bibitem[Pope et~al.(2023)Pope, Douglas, Chowdhery, Devlin, Bradbury, Heek,
  Xiao, Agrawal, and Dean]{pope2023efficiently}
Reiner Pope, Sholto Douglas, Aakanksha Chowdhery, Jacob Devlin, James Bradbury,
  Jonathan Heek, Kefan Xiao, Shivani Agrawal, and Jeff Dean.
\newblock Efficiently scaling transformer inference.
\newblock \emph{Proceedings of Machine Learning and Systems}, 5, 2023.

\bibitem[Puigcerver et~al.(2023)Puigcerver, Riquelme, Mustafa, and
  Houlsby]{puigcerver2023sparse}
Joan Puigcerver, Carlos Riquelme, Basil Mustafa, and Neil Houlsby.
\newblock From sparse to soft mixtures of experts.
\newblock \emph{arXiv preprint arXiv:2308.00951}, 2023.

\bibitem[Rajbhandari et~al.(2020)Rajbhandari, Rasley, Ruwase, and He]{ZeRO}
Samyam Rajbhandari, Jeff Rasley, Olatunji Ruwase, and Yuxiong He.
\newblock Zero: memory optimizations toward training trillion parameter models.
\newblock In Christine Cuicchi, Irene Qualters, and William~T. Kramer, editors,
  \emph{Proceedings of the International Conference for High Performance
  Computing, Networking, Storage and Analysis, {SC} 2020, Virtual Event /
  Atlanta, Georgia, USA, November 9-19, 2020}, page~20. {IEEE/ACM}, 2020.

\bibitem[Rajbhandari et~al.(2022)Rajbhandari, Li, Yao, Zhang, Aminabadi, Awan,
  Rasley, and He]{rajbhandari2022deepspeedmoe}
Samyam Rajbhandari, Conglong Li, Zhewei Yao, Minjia Zhang, Reza~Yazdani
  Aminabadi, Ammar~Ahmad Awan, Jeff Rasley, and Yuxiong He.
\newblock Deepspeed-moe: Advancing mixture-of-experts inference and training to
  power next-generation {AI} scale.
\newblock In Kamalika Chaudhuri, Stefanie Jegelka, Le~Song, Csaba
  Szepesv{\'{a}}ri, Gang Niu, and Sivan Sabato, editors, \emph{International
  Conference on Machine Learning, {ICML} 2022, 17-23 July 2022, Baltimore,
  Maryland, {USA}}, volume 162 of \emph{Proceedings of Machine Learning
  Research}, pages 18332--18346. {PMLR}, 2022.

\bibitem[Ramachandran et~al.(2017)Ramachandran, Zoph, and
  Le]{ramachandran2017searching}
Prajit Ramachandran, Barret Zoph, and Quoc~V Le.
\newblock Searching for activation functions.
\newblock \emph{arXiv preprint arXiv:1710.05941}, 2017.

\bibitem[Santacroce et~al.(2023)Santacroce, Wen, Shen, and
  Li]{santacroce2023matters}
Michael Santacroce, Zixin Wen, Yelong Shen, and Yuanzhi Li.
\newblock What matters in the structured pruning of generative language models?
\newblock \emph{CoRR}, 2023.

\bibitem[Shazeer(2020)]{shazeer2020glu}
Noam Shazeer.
\newblock Glu variants improve transformer, 2020.

\bibitem[Shazeer et~al.(2017)Shazeer, Mirhoseini, Maziarz, Davis, Le, Hinton,
  and Dean]{shazeer2017outrageously}
Noam Shazeer, Azalia Mirhoseini, Krzysztof Maziarz, Andy Davis, Quoc~V. Le,
  Geoffrey~E. Hinton, and Jeff Dean.
\newblock Outrageously large neural networks: The sparsely-gated
  mixture-of-experts layer.
\newblock In \emph{5th International Conference on Learning Representations,
  {ICLR} 2017, Toulon, France, April 24-26, 2017, Conference Track
  Proceedings}. OpenReview.net, 2017.

\bibitem[Shen et~al.(2023)Shen, Guo, Tan, Tang, Wang, and Bian]{shen2023study}
Kai Shen, Junliang Guo, Xu~Tan, Siliang Tang, Rui Wang, and Jiang Bian.
\newblock A study on relu and softmax in transformer.
\newblock \emph{CoRR}, 2023.

\bibitem[Sheng et~al.(2023)Sheng, Zheng, Yuan, Li, Ryabinin, Chen, Liang,
  R{\'{e}}, Stoica, and Zhang]{sheng2023flexgen}
Ying Sheng, Lianmin Zheng, Binhang Yuan, Zhuohan Li, Max Ryabinin, Beidi Chen,
  Percy Liang, Christopher R{\'{e}}, Ion Stoica, and Ce~Zhang.
\newblock Flexgen: High-throughput generative inference of large language
  models with a single {GPU}.
\newblock In Andreas Krause, Emma Brunskill, Kyunghyun Cho, Barbara Engelhardt,
  Sivan Sabato, and Jonathan Scarlett, editors, \emph{International Conference
  on Machine Learning, {ICML} 2023, 23-29 July 2023, Honolulu, Hawaii, {USA}},
  volume 202 of \emph{Proceedings of Machine Learning Research}, pages
  31094--31116. {PMLR}, 2023.

\bibitem[Song et~al.(2021)Song, Zhang, and Zhang]{song2021training}
Zhao Song, Lichen Zhang, and Ruizhe Zhang.
\newblock Training multi-layer over-parametrized neural network in subquadratic
  time, 2021.

\bibitem[Sun et~al.(2023)Sun, Liu, Bair, and Kolter]{sun2023simple}
Mingjie Sun, Zhuang Liu, Anna Bair, and J.~Zico Kolter.
\newblock A simple and effective pruning approach for large language models.
\newblock \emph{CoRR}, 2023.

\bibitem[Team et~al.(2022)Team, Costa-jussà, Cross, Çelebi, Elbayad,
  Heafield, Heffernan, Kalbassi, Lam, Licht, Maillard, Sun, Wang, Wenzek,
  Youngblood, Akula, Barrault, Gonzalez, Hansanti, Hoffman, Jarrett, Sadagopan,
  Rowe, Spruit, Tran, Andrews, Ayan, Bhosale, Edunov, Fan, Gao, Goswami,
  Guzmán, Koehn, Mourachko, Ropers, Saleem, Schwenk, and
  Wang]{nllbteam2022language}
NLLB Team, Marta~R. Costa-jussà, James Cross, Onur Çelebi, Maha Elbayad,
  Kenneth Heafield, Kevin Heffernan, Elahe Kalbassi, Janice Lam, Daniel Licht,
  Jean Maillard, Anna Sun, Skyler Wang, Guillaume Wenzek, Al~Youngblood, Bapi
  Akula, Loic Barrault, Gabriel~Mejia Gonzalez, Prangthip Hansanti, John
  Hoffman, Semarley Jarrett, Kaushik~Ram Sadagopan, Dirk Rowe, Shannon Spruit,
  Chau Tran, Pierre Andrews, Necip~Fazil Ayan, Shruti Bhosale, Sergey Edunov,
  Angela Fan, Cynthia Gao, Vedanuj Goswami, Francisco Guzmán, Philipp Koehn,
  Alexandre Mourachko, Christophe Ropers, Safiyyah Saleem, Holger Schwenk, and
  Jeff Wang.
\newblock No language left behind: Scaling human-centered machine translation,
  2022.

\bibitem[Touvron et~al.(2023)Touvron, Lavril, Izacard, Martinet, Lachaux,
  Lacroix, Rozi{\`{e}}re, Goyal, Hambro, Azhar, Rodriguez, Joulin, Grave, and
  Lample]{Llamav1paper}
Hugo Touvron, Thibaut Lavril, Gautier Izacard, Xavier Martinet, Marie{-}Anne
  Lachaux, Timoth{\'{e}}e Lacroix, Baptiste Rozi{\`{e}}re, Naman Goyal, Eric
  Hambro, Faisal Azhar, Aur{\'{e}}lien Rodriguez, Armand Joulin, Edouard Grave,
  and Guillaume Lample.
\newblock Llama: Open and efficient foundation language models.
\newblock \emph{CoRR}, abs/2302.13971, 2023.

\bibitem[Vaswani et~al.(2017)Vaswani, Shazeer, Parmar, Uszkoreit, Jones, Gomez,
  Kaiser, and Polosukhin]{vaswani2023attention}
Ashish Vaswani, Noam Shazeer, Niki Parmar, Jakob Uszkoreit, Llion Jones,
  Aidan~N. Gomez, Lukasz Kaiser, and Illia Polosukhin.
\newblock Attention is all you need.
\newblock In Isabelle Guyon, Ulrike von Luxburg, Samy Bengio, Hanna~M. Wallach,
  Rob Fergus, S.~V.~N. Vishwanathan, and Roman Garnett, editors, \emph{Advances
  in Neural Information Processing Systems 30: Annual Conference on Neural
  Information Processing Systems 2017, December 4-9, 2017, Long Beach, CA,
  {USA}}, pages 5998--6008, 2017.

\bibitem[Wortsman et~al.(2023)Wortsman, Lee, Gilmer, and
  Kornblith]{wortsman2023replacing}
Mitchell Wortsman, Jaehoon Lee, Justin Gilmer, and Simon Kornblith.
\newblock Replacing softmax with relu in vision transformers, 2023.

\bibitem[Xiao et~al.(2023)Xiao, Lin, Seznec, Wu, Demouth, and
  Han]{xiao2023smoothquant}
Guangxuan Xiao, Ji~Lin, Micka{\"{e}}l Seznec, Hao Wu, Julien Demouth, and Song
  Han.
\newblock Smoothquant: Accurate and efficient post-training quantization for
  large language models.
\newblock In Andreas Krause, Emma Brunskill, Kyunghyun Cho, Barbara Engelhardt,
  Sivan Sabato, and Jonathan Scarlett, editors, \emph{International Conference
  on Machine Learning, {ICML} 2023, 23-29 July 2023, Honolulu, Hawaii, {USA}},
  volume 202 of \emph{Proceedings of Machine Learning Research}, pages
  38087--38099. {PMLR}, 2023.

\bibitem[Yi et~al.(2023)Yi, Guo, Wei, Zhou, Wang, and Xu]{yi2023edgemoe}
Rongjie Yi, Liwei Guo, Shiyun Wei, Ao~Zhou, Shangguang Wang, and Mengwei Xu.
\newblock Edgemoe: Fast on-device inference of moe-based large language models.
\newblock \emph{arXiv preprint arXiv:2308.14352}, 2023.

\bibitem[Zhang and Sennrich(2019)]{zhang2019root}
Biao Zhang and Rico Sennrich.
\newblock Root mean square layer normalization.
\newblock In Hanna~M. Wallach, Hugo Larochelle, Alina Beygelzimer, Florence
  d'Alch{\'{e}}{-}Buc, Emily~B. Fox, and Roman Garnett, editors, \emph{Advances
  in Neural Information Processing Systems 32: Annual Conference on Neural
  Information Processing Systems 2019, NeurIPS 2019, December 8-14, 2019,
  Vancouver, BC, Canada}, pages 12360--12371, 2019.

\bibitem[Zhang et~al.(2023)Zhang, Chen, Shen, Yang, Ou, Yu, and
  Zhuang]{zhang2023pruning}
Mingyang Zhang, Hao Chen, Chunhua Shen, Zhen Yang, Linlin Ou, Xinyi Yu, and
  Bohan Zhuang.
\newblock Pruning meets low-rank parameter-efficient fine-tuning.
\newblock \emph{CoRR}, 2023.

\bibitem[Zhang et~al.(2022{\natexlab{a}})Zhang, Roller, Goyal, Artetxe, Chen,
  Chen, Dewan, Diab, Li, Lin, Mihaylov, Ott, Shleifer, Shuster, Simig, Koura,
  Sridhar, Wang, and Zettlemoyer]{OPTpaper}
Susan Zhang, Stephen Roller, Naman Goyal, Mikel Artetxe, Moya Chen, Shuohui
  Chen, Christopher Dewan, Mona~T. Diab, Xian Li, Xi~Victoria Lin, Todor
  Mihaylov, Myle Ott, Sam Shleifer, Kurt Shuster, Daniel Simig, Punit~Singh
  Koura, Anjali Sridhar, Tianlu Wang, and Luke Zettlemoyer.
\newblock {OPT:} open pre-trained transformer language models.
\newblock \emph{CoRR}, abs/2205.01068, 2022{\natexlab{a}}.

\bibitem[Zhang et~al.(2022{\natexlab{b}})Zhang, Lin, Liu, Li, Sun, and
  Zhou]{zhang2022moefication}
Zhengyan Zhang, Yankai Lin, Zhiyuan Liu, Peng Li, Maosong Sun, and Jie Zhou.
\newblock Moefication: Transformer feed-forward layers are mixtures of experts.
\newblock In Smaranda Muresan, Preslav Nakov, and Aline Villavicencio, editors,
  \emph{Findings of the Association for Computational Linguistics: {ACL} 2022,
  Dublin, Ireland, May 22-27, 2022}, pages 877--890. Association for
  Computational Linguistics, 2022{\natexlab{b}}.

\bibitem[Zhu et~al.(2023)Zhu, Li, Liu, Ma, and Wang]{zhu2023survey}
Xunyu Zhu, Jian Li, Yong Liu, Can Ma, and Weiping Wang.
\newblock A survey on model compression for large language models.
\newblock \emph{CoRR}, 2023.

\bibitem[Zoph et~al.(2022)Zoph, Bello, Kumar, Du, Huang, Dean, Shazeer, and
  Fedus]{zoph2022st}
Barret Zoph, Irwan Bello, Sameer Kumar, Nan Du, Yanping Huang, Jeff Dean, Noam
  Shazeer, and William Fedus.
\newblock St-moe: Designing stable and transferable sparse expert models.
\newblock \emph{arXiv preprint arXiv:2202.08906}, 2022.

\end{thebibliography}
